\documentclass[lettersize,journal]{IEEEtran}

\newif\ifanonymous
\anonymousfalse
\ifdefined\setanonymous \if0\setanonymous \anonymousfalse \else \anonymoustrue \fi \fi
\usepackage{amsmath,amsfonts}
\usepackage{algorithmic}
\usepackage{array}
\usepackage[caption=false,font=normalsize,labelfont=sf,textfont=sf]{subfig}
\usepackage{textcomp}
\usepackage{stfloats}
\usepackage{url}
\usepackage{verbatim}
\usepackage{graphicx}
\def\BibTeX{{\rm B\kern-.05em{\sc i\kern-.025em b}\kern-.08em
    T\kern-.1667em\lower.7ex\hbox{E}\kern-.125emX}}
\usepackage{balance}
\usepackage{capt-of}   

\usepackage{moreverb,url}
\usepackage{xspace}
\usepackage{xcolor}
\usepackage{colortbl}
\usepackage{siunitx}
\usepackage{wrapfig}

\usepackage{amssymb}
\usepackage{amsmath}
\usepackage{amsbsy}
\usepackage{bbm}
\usepackage{dsfont}

\usepackage{amsthm}
\newtheorem{assumption}{Assumption}

\usepackage{algorithm}
\usepackage{algorithmic}

\usepackage{graphicx}
\usepackage{wrapfig}
\usepackage{booktabs} 

\usepackage{subfiles}
\usepackage{csquotes}

\usepackage{array}

\usepackage{graphicx}
\usepackage{subfig} 
\usepackage{multirow}
\definecolor{wine}{RGB}{204, 0, 102}
\definecolor{magenta_wine}{RGB}{158, 44, 143}
\definecolor{dusty_wine}{RGB}{143, 59, 101}
\definecolor{ocean}{RGB}{13, 121, 202}
\definecolor{light_ocean}{RGB}{18, 178, 235}
\definecolor{dark_ocean}{RGB}{10, 89, 148}
\definecolor{grey}{RGB}{170, 170, 170}
\definecolor{light-grey}{RGB}{220, 220, 220}
\definecolor{dark_gray}{rgb}{0.2, 0.2, 0.2} 
\definecolor{med-grey}{rgb}{0.3, 0.3, 0.3} 
\definecolor{grape}{RGB}{112,48,160}
\definecolor{aqua}{RGB}{52,172,139}
\definecolor{dark_aqua}{RGB}{35,115,93}
\definecolor{dark_orange}{RGB}{216,92,0}
\definecolor{vibrant_orange}{RGB}{250, 160, 26}
\definecolor{vibrant_blue}{RGB}{14, 120, 255}
\definecolor{vibrant_pink}{RGB}{255, 0, 104}
\definecolor{dark_red}{RGB}{122, 0, 0}
\definecolor{dark_green}{RGB}{0, 92, 34}
\definecolor{dusty_blue}{RGB}{77, 91, 128}
\definecolor{dark_brown}{RGB}{125, 54, 36}
\definecolor{violet}{RGB}{116, 12, 173}
\definecolor{teal}{HTML}{0D948F}
\definecolor{disturbance_orange}{HTML}{FF9500}
\definecolor{nominal_blue}{HTML}{0048A6}

\newtheorem{theorem}{Theorem}

\newtheorem{corollary}{Corollary}
\newtheorem{remark}{Remark}

\newcommand{\para}[1]{\smallskip\noindent\textbf{#1. }} 
 
\newcounter{qnum}
\usepackage{pifont}
\newcommand{\cmark}{\ding{51}} 
\newcommand{\xmark}{\ding{55}}%

\DeclareMathOperator{\E}{\mathbb{E}}
\DeclareMathOperator*{\argmax}{argmax}
\DeclareMathOperator*{\argmin}{argmin}

\newcommand{\state}{s}
\newcommand{\stateSpace}{\mathcal{S}}

\newcommand{\latent}{z}
\newcommand{\latentSpace}{\mathcal{Z}}
\newcommand{\dynz}{f_{\latent}}
\newcommand{\dynamics}{f}
\newcommand{\encoder}{\mathcal{E}}
\newcommand{\decoder}{\mathcal{G}}
\newcommand{\disturbance}{d}
\newcommand{\disturbanceSpace}{D}
\newcommand{\disturbancePolicy}{\policy_{\disturbance}}

\newcommand{\objective}{J}

\newcommand{\obs}{o}
\newcommand{\obsSpace}{\mathcal{O}}
\newcommand{\observation}{\mathcal{H}} 

\newcommand{\action}{a}
\newcommand{\actionSpace}{\mathcal{A}}

\usepackage{fontawesome5}
\newcommand{\shield}{\text{\tiny{\faShield*}}}

\newcommand{\failure}{F}

\newcommand{\failuremargin}{\ell}

\newcommand{\ellz}{\ell_\latent}

\newcommand{\valfunc}{V}
\newcommand{\valfuncGt}{\valfunc_\textrm{gt}}

\newcommand{\unsafeSet}{U} 

\newcommand{\policy}{\pi}
\newcommand{\policyTask}{\pi^{\text{task}}}
\newcommand{\monitor}{V^{\shield}} 
\newcommand{\fallback}{\policy^{\shield}} 
\newcommand{\buffer}{\mathcal{B}}

\newcommand{\dataset}{\mathcal{D}}

\newcommand{\calib}{\text{calib}}

\newcommand{\divergence}{D}
\newcommand{\KL}{D_{\mathrm{KL}}}

\newcommand{\Qfunction}{Q}
\newcommand{\QfunctionRobust}{\Qfunction_{\mathrm{rob}}}
\newcommand{\targetQfunctionRobust}{\hat{Q}_{\mathrm{rob}}}

\newcommand{\fallbackRobust}{\policy_\mathrm{rob}} 
\newcommand{\monitorRobust}{V_\mathrm{rob}} 
\newcommand{\ambiguity}{\mathcal{F}}
\newcommand{\latentDisturbanceDynamics}{\dynz^d}
\newcommand{\optimalDisturbance}{\dynz^{\latentDisturbance}}
\newcommand{\latentDisturbanceDynamicsParameterized}{\dynz^{\latentDisturbance}}

\newcommand{\parameter}{\phi}
\newcommand{\latentDisturbance}{\policy_d}
\newcommand{\latentDisturbanceParam}{\psi}
\newcommand{\densityEstimator}{s_{\mathrm{OOD}}}
\newcommand{\thresholdKL}{\epsilon_{\mathrm{KL}}}
\newcommand{\thresholdOOD}{\epsilon_{\mathrm{OOD}}}
\newcommand{\calibAmbiguity}{\alpha_{\mathrm{KL}}}
\newcommand{\calibOOD}{\alpha_{\mathrm{ood}}}
\newcommand{\penaltyKL}{\lambda_{\mathrm{KL}}}
\newcommand{\penaltyOOD}{\lambda_{\mathrm{OOD}}}

\newcommand{\discount}{\gamma}
\newcommand{\sensitivity}{\kappa}
\newcommand{\bellmanOperator}{\mathcal{T}}
\newcommand{\robustOperator}{\bellmanOperator_{\mathrm{rob}}}

\newcommand{\nominalValue}{\valfunc_{\dynz}}
\newcommand{\robustValue}{\valfunc_{\mathrm{rob}}}
\newcommand{\nominalPolicy}{\policy^{\mathrm{nom}}}
\newcommand{\robustPolicy}{\policy^{\mathrm{rob}}}
\newcommand{\worstCaseValue}{\mathcal{J}_{\mathrm{wc}}}
\newcommand{\worstCaseSuboptimality}{\operatorname{SubOpt}_{\mathrm{wc}}}

\newcommand{\CVaR}{\mathrm{CVaR}}
\newcommand{\monitorNominal}{\monitor_{\mathrm{nom}}}
\newcommand{\monitorWoN}{\monitor_{\mathrm{WoN}}}
\newcommand{\monitorCVaR}{\monitor_{\CVaR}}

\newcommand{\ours}{\textit{LUCID}\xspace}
\newcommand{\nominal}{\textit{Nominal}\xspace}

\usepackage{annotate_equation}
\usepackage{soul}
\newcommand{\methodhighlight}[2]{\begingroup\sethlcolor{#1!\eqnhighlightshade}\hl{#2}\endgroup}

\usepackage[colorlinks=true,
            linkcolor=dark_ocean,
            citecolor=dark_ocean,
            urlcolor=dark_ocean]{hyperref}

\begin{document}

\title{\LARGE \bf
   Modeling Latent Disturbances for \\ Robust Decision-Making in World Models
}
\ifanonymous
\author{Anonymous Authors
\thanks{}}
\else
\author{Junwon Seo and Andrea Bajcsy%
\thanks{The authors are with the Robotics Institute, Carnegie Mellon University, Pittsburgh, PA, USA. Email: \{junwonse, abajcsy\}@andrew.cmu.edu}%
\\[0.3em] { \hypersetup{urlcolor=disturbance_orange}\url{https://junwon.me/LatentDisturbance/}}}
\fi

\makeatletter
\let\@oldmaketitle\@maketitle

\renewcommand{\@maketitle}{%
    \@oldmaketitle
    \setcounter{figure}{0}
    \centering
    \includegraphics[width=\textwidth]{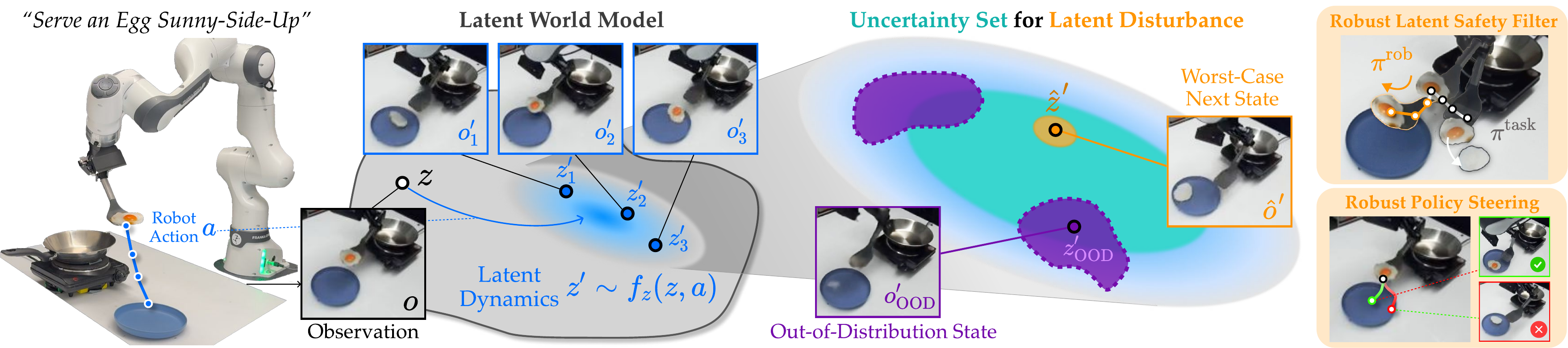}
    \vspace{-0.2in}
    \captionof{figure}{
        \textbf{Overview of Robust Decision Making in World Models.} \textit{Left}: A latent world model learns a latent dynamics model $\latent'\sim\dynz(\latent,\action)$ from observation--action transitions $(\obs,\action,\obs')$, predicting a \textit{distribution} over next latent states $\latent'$ given the current latent state $\latent$ and the robot's action $\action$. \textit{Middle}: Leveraging the learned predictive distribution, we construct an \textit{uncertainty set} of latent dynamics, calibrated to include plausible transitions while excluding implausible out-of-distribution states $\latent'_{\mathrm{OOD}}$. Within this set, we model a \textit{latent disturbance} as a perturbation to the learned dynamics that yields worst-case plausible next states $\hat{\latent}'$. \textit{Right}: This enables robust decision-making in the latent space of world models, including robust safety filtering that proactively safeguards the system without overestimating safety and robust policy steering that selects actions that remain effective under their worst-case plausible outcomes.
    }
    \label{fig:front-fig}
    \vspace{-0.4in}
    \bigskip
}
\makeatother
\maketitle

\begin{abstract}
In this paper, we study robust decision-making in the latent space of world models (WMs). Robust optimization is a mathematical framework where, given explicitly specified dynamics and physically meaningful disturbances, a robot can select actions that remain effective even under worst-case disturbances. However, applying this principle to the learned latent space of WMs introduces a fundamental challenge: because WMs have fully learned state spaces and dynamics inferred from high-dimensional observations, it is unclear how to define latent-space disturbances that faithfully represent uncertainty in the underlying physical system. Our key idea is to model a latent-space disturbance as a perturbation to the learned latent dynamics that induces pessimistic but plausible transitions. Specifically, we construct a set of plausible latent dynamics by combining a dynamics-aware similarity metric that captures physically plausible transitions with out-of-distribution detection that excludes implausible latent states. We calibrate this uncertainty set over latent dynamics using conformal prediction, ensuring that WM imaginations induced by the latent disturbance remain plausible without becoming overly pessimistic. We then jointly optimize robust robot actions and the worst-case latent disturbances within the calibrated uncertainty set through efficient game-theoretic optimization. We leverage this latent-space robust optimization framework to robustify policy steering under uncertainty, considering two paradigms: latent safety filtering and sample-and-verify style steering of a generative control policy. Our controlled simulation experiments show that our latent disturbance enables robust decision-making directly in WM latent spaces, and hardware experiments with a Franka manipulator show that modeling latent disturbances enables robust policy steering, reducing failures by $70\%$ in safety filtering and $54\%$ in sampling-based policy steering.

\end{abstract}


\section{Introduction}

Consider a robot manipulator serving a sunny-side-up egg by sliding it from a spatula onto a plate (Fig.~\ref{fig:front-fig}). This task is deceptively challenging: it not only requires precise motor control from visual observations, such as carefully tilting the spatula, but the outcome of each robot action is uncertain: the same serving motion may successfully slide the egg onto the plate or accidentally flip it sunny-side-down depending on unobserved factors such as the egg's mass or the amount of oil on the spatula. How can the robot choose actions that are \textit{robust} to these uncertainties?

A formal way to model such robust decision-making under uncertainty is via \textit{robust optimization}, which selects a robot action that remains effective under \textit{any} possible realization of a disturbance, which represents system uncertainty~\cite{ben2009robust}:
\begin{equation}
\policy_\text{rob}(\state) = \arg\min_{\action \in \actionSpace} \max_{\disturbance \in \disturbanceSpace} \,\objective\left(\dynamics(\state,\action,\disturbance)\right).
\label{eq:intro_robust_optimization}
\end{equation}
Here, $f$ denotes discrete-time system dynamics conditioned on state $s$, action $a$, and disturbance $d$ representing uncertain factors beyond the robot's control; $\disturbanceSpace$ denotes the set of such disturbances, and the objective function $\objective$ evaluates the resulting future evolution. This optimization can be viewed as a dynamic game~\cite{fisac2018general, nguyen2024gameplay}: the outer minimization selects a robust action that remains effective despite the worst-case disturbance selected by the inner maximization~\cite{mitchell2005time,pinto2017robust,moos2022robust}.

Robust optimization typically requires a structured dynamical system model, where the state and disturbance are represented, often by an expert, to have clear physical meaning (e.g., position and velocity as states and external forces as disturbances~\cite{bansal2017hamilton, fisac2018general}). However, latent world models~(WMs)~\cite{hafner2023dreamerv3, maes2026leworldmodel} have recently shown promise for representing hard-to-model systems directly from high-dimensional sensor observations. By jointly learning compact latent representations and their associated dynamics from datasets of robot observations and actions, WMs enable complex systems to be modeled directly in a \textit{learned} latent space, such as interactions with deformable objects~\cite{nakamura2025generalizing} and complex rigid bodies~\cite{seo2025uncertainty}.

However, extending robust optimization in \eqref{eq:intro_robust_optimization} to the learned latent space of WMs raises a fundamental challenge: how to define a \textit{latent-space disturbance} that faithfully represents uncertainty in the underlying physical system? Unlike disturbances in physically interpretable state spaces~\cite{hsu2023isaacs}, arbitrary perturbations in the learned latent space need not correspond to physically plausible system evolutions. For example, in Fig.~\ref{fig:front-fig}, factors such as friction or the amount of oil on the spatula are not directly observed or explicitly represented by the world model. This motivates the central question of our work:
\begin{quote}
\centering
\textit{How can we represent disturbances in the latent space of world models for robust optimization?}
\end{quote}

Our key idea is to model a latent-space disturbance as a perturbation to the learned latent dynamics that induces pessimistic world-model imaginations, based on the observation that predictive uncertainty in the latent dynamics reflects disturbances in the underlying system. The main challenge is to construct an \textit{uncertainty set}~\cite{moos2022robust} of latent disturbances, a set of plausible latent dynamics that is broad enough to cover possible transitions while excluding implausible ones that would lead to overly conservative decisions~\cite{johnstone2021conformal,hu2023deception,seo2026stressdream}. This is particularly important in high-dimensional latent spaces, where nearby or high-likelihood states may still correspond to physically infeasible or hallucinated out-of-distribution (OOD) transitions~\cite{seo2025uncertainty,hansen2026hallucination}. If such OOD transitions are not properly excluded from the uncertainty set, the disturbance can exploit these regions, producing implausibly pessimistic imaginations and leading to overly conservative decisions.

To address this, we propose \textit{Latent-space Uncertainty-Calibrated In-distribution Disturbance (\ours)}, a model of latent-space disturbance, which combines a dynamics-aware similarity metric with out-of-distribution detection to construct an uncertainty set of plausible latent dynamics. We calibrate this uncertainty set using conformal prediction~\cite{shafer2008tutorial,angelopoulos2023conformal}, ensuring that world-model imaginations induced by the latent disturbance remain plausible without becoming overly pessimistic. We then jointly optimize the worst-case disturbance and robust action through game-theoretic optimization, using an efficient parameterization of the latent disturbance for tractable optimization. We instantiate \ours for robust policy steering, including latent safety filtering~\cite{nakamura2025generalizing, seo2025uncertainty, agrawal2026anysafe} and sampling-based policy steering~\cite{wu2025foresight, yuan2026ups}, which can anticipate potential failures without overestimating safety and steer the policy with actions that remain safe under adverse uncertainty.

Through experiments, we show that: (i) in controlled settings with access to the true system dynamics, \ours approximates the optimal solution to robust optimization in the world-model latent space; (ii) in high-fidelity simulated vision-based manipulation, where we can control system disturbances such as friction and mass, \ours robustly prevents task failures despite unobservable disturbances; and (iii) in hardware experiments with a Franka robot, \ours synthesizes a robust safety filter that proactively detects potential failures without becoming overly pessimistic, reducing the failure rate ($44\%\rightarrow13\%$), and enables robust policy steering through pessimistic yet plausible WM imaginations, improving the success rate of visuomotor policies ($35\%\rightarrow70\%$).

\section{Related Work}

\para{Robust Decision Making Under Uncertainty} Robust optimization formalizes robust decisions that remain effective across a set of plausible system evolutions, including worst-case realizations~\cite{mayne2005robust, moos2022robust}, by optimizing worst-case performance over an uncertainty set. The uncertainty set characterizes plausible variations in the system~\cite{ben2009robust}, often in physically interpretable terms such as external force~\cite{nguyen2024gameplay}, model errors~\cite{fisac2019bridging, cosner2023robustsafetystochasticuncertainty}, or uncertainty estimated from data~\cite{AkametaluGP, johnstone2021conformal}. Similarly, distributionally robust optimization defines the uncertainty set over probability distributions to account for distribution shifts in environments or model dynamics~\cite{yang2018dynamic, ren2022distributionally}. The problem is commonly formulated as a game between the controller and an adversary that selects worst-case disturbances~\cite{yang2018dynamic, choi2021robust, ren2022distributionally} with game-theoretic methods~\cite{hsu2023isaacs, nguyen2024gameplay} or Lagrangian approximations~\cite{sinha2018certifiable, ren2022distributionally}. A related perspective is risk-sensitive control~\cite{chow2018risk,ma2021conservative,chapman2021risk,nishimura2023rap}, which accounts for adverse outcomes using risk measures such as conditional value-at-risk (CVaR)~\cite{Majumdar2017HowSA}. Certain risk measures can also be expressed as worst-case expectations over corresponding sets of probability distributions, providing a connection between risk-sensitive and robust optimization~\cite{chow2015risk}. In robust reinforcement learning, policies are optimized against adverse perturbations to observations~\cite{liang2022efficient, zhang2020robust, ren2022distributionally}, dynamics variations~\cite{pinto2017robust, mandlekar2017adversarially}, distributional shifts encountered at deployment, and out-of-distribution state-action regions~\cite{curi2021combining, rigter2022rambo, blanchet2023double, chen2026policydriven}.

A key design choice in robust optimization is the disturbance and its uncertainty set, often defined using physically meaningful bounds specified by experts~\cite{hsu2023isaacs, nguyen2024gameplay, oh2026synthesis} or estimated from data~\cite{johnstone2021conformal, rigter2023one, sun2023conformal, marques2024quantifying}, where an overly broad set may admit implausible system evolutions, leading to overly conservative decisions~\cite{zhou2026calibrating}. This issue is more pronounced in the learned latent space of world models, where latent-space similarity does not necessarily reflect the underlying physical dynamics~\cite{ wang2026temporal} and may include implausible out-of-distribution states~\cite{seo2025uncertainty, lutkus2025latent, hansen2026hallucination}. Our work addresses this challenge by constructing a conformalized uncertainty set over plausible latent dynamics, optimizing for worst-case within this set to produce pessimistic but plausible world-model imaginations.

\para{World Models in Robotics} World models predict the outcomes of robot actions directly from high-dimensional observations such as RGB images, enabling vision-based policy learning~\cite{hafner2023dreamerv3}, planning~\cite{hansen2024tdmpc2}, policy evaluation~\cite{guo2026ctrl, quevedo2026worldgym}, and policy improvement~\cite{guo2026vlaw, sharma2026world} through imaginations~\cite{agarwal2025cosmos, gao2026dreamdojo, ye2026world}. Compared to video world models that predict future observations directly in pixel space~\cite{mei2026video, seo2026stressdream}, latent world models jointly learn compact latent representations and their associated dynamics~\cite{zhou2024dino, maes2026leworldmodel}, enabling efficient decision-making directly in the learned latent space~\cite{nakamura2025generalizing, seo2025uncertainty, jain2026texttt}. Because world models operate under partial observability, latent dynamics models are often probabilistic models such as recurrent state-space models~\cite{hafner2019learning} and diffusion models~\cite{Kerssies_2026_CVPR}. However, defining disturbances that meaningfully represent uncertainties for robust optimization is challenging~\cite{ren2022distributionally, nath2026pixels}, while model errors and unseen inputs can further produce hallucinated states during imagination~\cite{seo2025uncertainty, hansen2026hallucination}. Consequently, world models typically use finite samples from learned dynamics~\cite{zhang2023storm, hafner2025training, huang2026nano}, which can miss rare but high-impact outcomes in high-dimensional spaces~\cite{seo2026stressdream}. Our work introduces a latent-space disturbance that directly searches for worst-case latent dynamics with gradient-based optimization.

\para{Latent Safety Filters} Safety filtering is a control-theoretic approach that safeguards dynamical systems from failures given known safety constraints and a system dynamics model~\cite{hsu2023safety}. Safety filters can be implemented through minimally restrictive switching with a safety fallback policy synthesized via Hamilton--Jacobi (HJ) reachability analysis~\cite{bansal2017hamilton}, control barrier functions (CBFs)~\cite{choi2021robust}, or predictive rollout-based methods such as model predictive safety filtering~\cite{li2020robust,bastani2021safe}. While such filters can be synthesized using dynamic programming~\cite{mitchell2007toolbox}, self-supervised learning~\cite{bansal2021deepreach} and reinforcement learning (RL)~\cite{fisac2019bridging} have enabled scalable approximations for high-dimensional systems~\cite{wabersich2023data}. Latent safety filters extend this idea to the latent space of learned world models, enabling safety filtering for hard-to-model tasks and safety constraints from high-dimensional observations~\cite{nakamura2025generalizing, seo2025uncertainty, agrawal2026anysafe}. By anticipating future failures within world-model imaginations, these methods can proactively steer robots away from failures such as spilling the contents of deformable bags~\cite{nakamura2026howto} or toppling complex rigid-body structures~\cite{seo2025uncertainty}. However, existing latent safety filters evaluate safety under nominal latent dynamics and may therefore overestimate safety under optimistic imaginations. Our method enables the synthesis of robust safety filters using game-theoretic adversarial RL~\cite{hsu2023isaacs, nguyen2024gameplay, oh2026synthesis} in the learned latent space of world models.

\section{Preliminaries: Latent World Models}
\label{sec:prelim}

\para{System} The robot operates from high-dimensional observations $\obs_t=\observation(\state_t)\in\obsSpace$ (e.g., RGB images), where $\observation:\stateSpace\rightarrow\obsSpace$ denotes the sensor mapping from the privileged state to the observation. We assume that the underlying system \(\state_{t+1}=\dynamics(\state_t, \action_t, \disturbance_t)\) is a discrete-time system with a privileged state and continuous control input $\action_t \in \actionSpace \subset \mathbb{R}^{|\actionSpace|}$, and the disturbance $\disturbance_t \in \disturbanceSpace$ is not observable to the robot, inducing uncertainty in the system. In classical control, a disturbance set $\disturbanceSpace$ is typically specified by human experts or identified from data and often carries a physical meaning (e.g., bounds on wind or actuation error)~\cite{fisac2019bridging, hsu2023safety, wabersich2023data}.

\para{Latent World Model} To represent hard-to-model systems, we learn a latent world model from a dataset of observation-action trajectories,
\(
\mathcal{D}_{\text{train}}
:=\left\{
(\obs_t,\action_t,\obs_{t+1})_{t=1}^{T-1}
\right\}_{i=1}^{N}
\)~\cite{hafner2023dreamerv3, maes2026leworldmodel}. It comprises: (i) an encoder $\encoder$ that maps an observation history to latent states $\latent \in \latentSpace$, (ii) a decoder $\decoder$ that reconstructs observations from latent states, and (iii) a latent dynamics model $\dynz\!:\!\latentSpace\times\actionSpace\!\rightarrow\!\Delta(\latentSpace)$ that predicts the next latent states, where $\Delta$ denotes the set of distributions:
\begin{equation}
\begin{aligned}
\text{Encoder: } &\latent_t \sim \encoder(\latent_t \mid \obs_{\le t}) \\
\text{Decoder: } &\obs_t \sim \decoder(\obs_t \mid \latent_t) \\ 
\text{Latent Dynamics: } &\latent_t \sim \dynz(\latent_t \mid \latent_{t-1}, \action_{t-1}).
\label{eq:world_model} 
\end{aligned}
\end{equation}

\para{World Model Training} In this work, we focus on recurrent state-space models (RSSM)~\cite{hafner2019learning, hafner2023dreamerv3}, which learn probabilistic latent dynamics with tractable conditional distributions over future latent states, using Gaussian or categorical distributions:
\begin{equation}\label{eq:latent-dynamics-parameterization}
    \dynz(\latent'\mid \latent, \action)
    =
    \mathcal{N}\big(\mu (\latent, \action), \Sigma(\latent, \action)\big)
    \,\, \text{or} \,\,\,
    \mathrm{Cat}\big(\parameter(\latent, \action)\big).
\end{equation}
The latent world model can be trained by learning both the latent representation through a reconstruction objective and the latent dynamics through a KL-divergence term:
\begingroup
\renewcommand{\mintagsep}{2pt}
\setlength{\thinmuskip}{0mu}
\setlength{\medmuskip}{1mu}
\setlength{\thickmuskip}{1mu}
\begin{equation}\label{eq:rssm-training-objective}
\begin{aligned}
\sum_{t=1}^T
\!\big[\!\underbrace{\log\decoder(\obs_t\!\mid\!\latent_t)}_{\text{\makebox[0pt]{reconstruction}}}\!-\!\underbrace{\KL\!\left(
\encoder(\latent_t\!\mid\!\obs_{\le t})
\|
\dynz(\latent_t\!\mid\!\latent_{t-1},\!\action_{t-1})
\right)}_{\text{latent dynamics}}\big],
\end{aligned}
\end{equation}
\endgroup
where the $\encoder(\latent_t\mid\obs_{\le t})$ provides a learned posterior target for the latent dynamics predictive distribution $\dynz(\latent_t\mid\latent_{t-1},\action_{t-1})$.

\para{Uncertainty in Latent Dynamics Represents Disturbances} Recall the example from Fig.~\ref{fig:front-fig}, where unobservable disturbances such as the oiliness of the spatula can lead to different next states under the same action. The probabilistic latent dynamics model therefore \textit{implicitly} captures these disturbances through its predictive distribution over next latent states~\cite{AkametaluGP, wabersich2023data} (Fig.~\ref{fig:front-fig}, left). During WM training with \eqref{eq:rssm-training-objective}, the latent state encoded from the next observation serves as the realized transition target for the dynamics model through the KL-divergence objective. In this way, uncertainty induced by unobserved system disturbances is captured in terms of the predictive uncertainty in the learned latent dynamics.

\section{Setup: Robust Optimization in Latent Space}\label{sec:setup-latent-ro}

Using the latent world model to predict the outcomes of robot actions, our goal is to solve the robust optimization problem in \eqref{eq:intro_robust_optimization} over the world model's latent space. We specify the robot's task through a cost function $\objective:\latentSpace\rightarrow\mathbb{R}$ that evaluates a latent state. Depending on the task, $\objective$ may represent an immediate cost, a cumulative cost over a finite horizon~\cite{ren2022distributionally}, or an infinite-horizon value function~\cite{choi2021robust, hsu2023safety}. Throughout this section, we present the optimization in single-step form for notational simplicity; the formulation also extends to multi-step (e.g., $T$-step) or infinite-horizon settings.

\para{Latent Disturbances as Pessimistic Imaginations} Traditionally, solving \eqref{eq:intro_robust_optimization} requires identifying disturbances $\disturbance\in\disturbanceSpace$ that induce worst-case next states, but the learned latent dynamics in \eqref{eq:world_model} do not explicitly model such disturbances. Nevertheless, as described in Sec.~\ref{sec:prelim}, the latent dynamics implicitly capture the effects of system disturbances through probabilistic predictions. Based on the intuition that predictive uncertainty in the learned dynamics reflects such disturbances, we model the worst-case \methodhighlight{disturbance_orange}{\textit{latent-space disturbance}} as selecting perturbed latent dynamics $\latentDisturbanceDynamics$ around the nominal learned dynamics $\dynz$, inducing plausible but pessimistic imaginations:
\begin{equation}
\begingroup
\setlength{\fboxsep}{0pt}
\renewcommand{\eqnannotationfont}{\small}
\policy_\text{rob}(\latent)
=
\argmin_{\action\in\actionSpace}
\eqnmarkbox[disturbance_orange]{robust-inner-opt}{\max_{\latentDisturbanceDynamics(\latent,\action)\in\eqnhighlight{teal}{\ambiguity(\latent,\action)}}
\mathbb{E}_{\latent'\sim\latentDisturbanceDynamics(\cdot\mid\latent,\action)}}
\left[\objective(\latent')\right],
\label{eq:problem-setup-latent-robust-optimization}
\endgroup
\end{equation} where \methodhighlight{teal}{the uncertainty set \mbox{$\ambiguity(\latent,\action)$}} defines a set of plausible latent transitions of the underlying system for a given $(\latent,\action)$. 

\begin{figure}[h]
\centering
\includegraphics[width=1.0\linewidth]{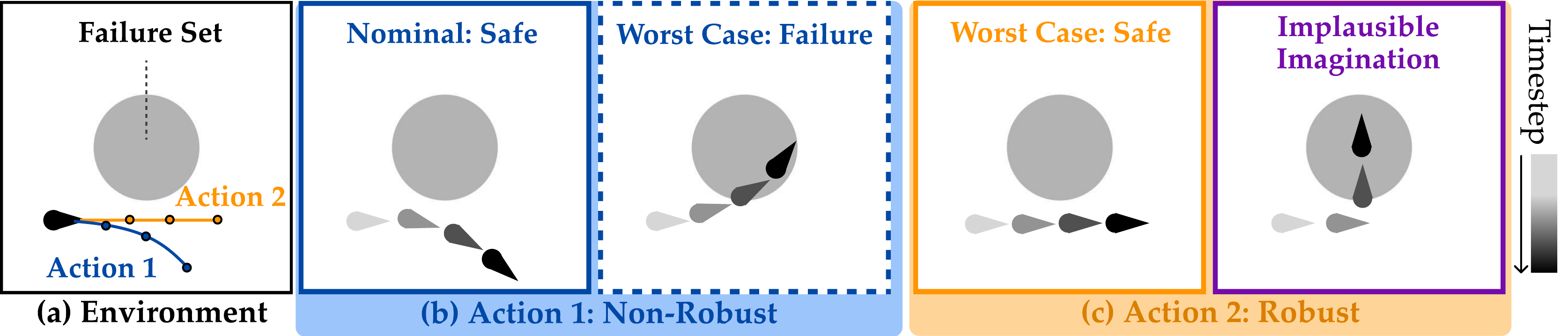}\vspace{-0.1in}
\caption{\textbf{Naughty 3D Dubins' Car.} (a) Environment with the vehicle and a failure set at the center. (b) For an action sequence that turns right, an adverse disturbance can instead drive the vehicle into failure, making the decision non-robust. (c) Driving straight remains safe even in the worst case. However, an implausible imagination can make this robust action appear unsafe.}\label{fig:naughty-dubins}
\vspace{-0.25in}
\end{figure}

\para{Running Example: \textit{Naughty} 3D Dubins' Car}\label{sec:running-example} In Fig.~\ref{fig:naughty-dubins}, we consider a discrete-time 3D Dubins' car with state $\state=[p_x,p_y,\theta]$, speed $v\!=\!1\mathrm{m/s}$, and timestep $\Delta t\!=\!0.05\,\mathrm{s}$. A circular failure set of radius $0.25\,\mathrm{m}$ centered at the origin defines the cost function \( \objective(\state)\!:=\!0.25^2-p_x^2-p_y^2.\) Each state is rendered as a \(3\!\times\!128\!\times\!128\) RGB image, and the robot's action is the angular velocity $\action_t\in[-1.25,1.25]\,\mathrm{rad/s}$, while its \textit{naughty} behavior is a disturbance introducing uncertainty by randomly flipping the sign of the control input:
\begin{equation}\label{eq:naughty-dynamics}
    \state_{t+1}
    \!=\!
    \state_t
    +
    \Delta t[v\cos(\theta_t),v\sin(\theta_t),\delta_t \action_t \,],
    \,\,
    \delta_t \in \{-1,1\}.
\end{equation} This disturbance is unobservable to the robot, making the underlying dynamics uncertain: for the same state and action, two distinct plausible next states may be realized. A robust action remains safe under the worst-case plausible outcome, whereas a non-robust action may appear safe \textit{nominally} but fail under an adverse disturbance.

\para{Challenges: Overly Pessimistic Disturbances} While \eqref{eq:problem-setup-latent-robust-optimization} formalizes robust optimization in the latent space of the world model, constructing an uncertainty set of latent dynamics is challenging because a latent-space distance metric can be misleading for two reasons. First, distance in a high-dimensional latent space does not necessarily reflect physical or dynamical plausibility, so nearby latent states may still correspond to infeasible outcomes. Second, learned latent dynamics may assign high likelihood to implausible out-of-distribution states, allowing the latent disturbance to exploit them and produce implausible imaginations. Consequently, a poorly constructed uncertainty set can admit overly pessimistic imaginations and make robust optimization overly conservative. 

For example, an action sequence of \methodhighlight{nominal_blue}{turning right} in Fig.~\ref{fig:naughty-dubins}(b) is a non-robust action because a control-sign flip can steer the vehicle into the central failure set, whereas \methodhighlight{disturbance_orange}{driving straight} in Fig.~\ref{fig:naughty-dubins}(c) is a robust action sequence because it remains safe under all possible control-sign flips. However, an \methodhighlight{violet}{implausible latent disturbance} can imagine the vehicle turning into the failure set, making the robust action appear unsafe.

\section{Conformalized Latent-Space Disturbance}

Our key idea is to model a latent-space disturbance as an alternative latent dynamics model within a \textit{conformalized} uncertainty set over plausible latent dynamics. To prevent the latent-space disturbance from becoming overly pessimistic, we adopt a dynamics-aware similarity metric together with out-of-distribution detection (Sec.~\ref{sec:disturbance-set}), and then calibrate this set using conformal prediction (Sec.~\ref{sec:conformal-calib}) (Fig.~\ref{fig:front-fig}). Finally, we jointly compute the robust action and the worst-case latent disturbance through game-theoretic optimization (Sec.~\ref{sec:latent-robust-optim}).

\subsection{Uncertainty Set over Plausible Latent Dynamics}\label{sec:disturbance-set}

\para{Dynamics-Aware Similarity} 
Recall that the latent dynamics $\dynz(\latent_{t+1}\,|\,\latent_t,\action_t)$ is trained to match the distribution over next latent states induced by next observations $\encoder(\latent_{t+1}\,|\,\obs_{\le t+1})$ using the KL divergence in \eqref{eq:rssm-training-objective}. This motivates using the same KL divergence to define the uncertainty set over plausible latent transitions, yielding a dynamics-aware notion of similarity rather than relying on an arbitrary distance in latent space. Specifically, for each transition from a latent-action pair $(\latent,\action)$, we define the dynamics-aware uncertainty set over plausible latent dynamics $\latentDisturbanceDynamics$ as \methodhighlight{teal}{a left-KL ball around the learned predictive distribution}\footnote{This KL direction matches the training objective in \eqref{eq:rssm-training-objective}. While the formulation is distributional, the mode-seeking character of left KL favors concentration on a single adverse mode of the learned latent dynamics model.}:
$\divergence_{\mathrm{KL}}\!\left(
\latentDisturbanceDynamics(\latent,\action)
\,\|\,
\dynz(\latent,\action)
\right)
\!\le \thresholdKL$, where $\thresholdKL$ controls the size of the uncertainty set to ensure coverage. This dynamics-aware KL ball contains alternative latent transitions consistent with the learned distribution, whereas a Euclidean ball may include nearby but dynamically infeasible transitions.

\para{In-Distribution Constraint} While the KL ball can provide \emph{coverage} by setting its radius to include plausible transitions, in high-dimensional latent spaces, coverage alone does not ensure \emph{precision}: latent states with high likelihood under the predictive distribution may still be implausible, or out-of-distribution (OOD)~\cite{seo2025uncertainty, hansen2026hallucination}. The disturbance may exploit such states, leading to overly pessimistic imaginations. We therefore augment the uncertainty set with an \methodhighlight{violet}{in-distribution constraint using an OOD score function \mbox{$\densityEstimator(\latent)$}}:
\begingroup
\setlength{\fboxsep}{0pt}
\setlength{\thickmuskip}{4mu}
\setlength{\medmuskip}{3mu}
\renewcommand{\eqnannotationfont}{\small}
\renewcommand{\mintagsep}{2pt}
\begin{align}
\rule{0pt}{2.4em}
\ambiguity(\dynz; \latent,\action)
:=&
\Big\{
\latentDisturbanceDynamics \in \Delta(\latentSpace)
: \eqnmarkbox[teal]{ambiguity-kl}{
\divergence_{\mathrm{KL}}\!\left(
\latentDisturbanceDynamics(\latent,\action)
\,\middle\|\,
\dynz(\latent,\action)
\right)
\le \thresholdKL
},
\notag\\[0.1em]
&\eqnmarkbox[violet]{ambiguity-ood}{
\mathbb{P}_{\latent' \sim \latentDisturbanceDynamics(\cdot\mid\latent,\action)}
\!\left(
\densityEstimator(\latent') > \thresholdOOD
\right)
\le \calibOOD
}
\Big\},
\rule[-1.7em]{0pt}{1.7em} 
\begin{tikzpicture}[overlay,remember picture]
\coordinate (ambiguity-kl-arrow) at ([xshift=1.8em]ambiguity-kl.north west);
\coordinate (ambiguity-ood-arrow) at ([xshift=5.8em]ambiguity-ood.south west);
\end{tikzpicture}
\addvalue{ambiguity-kl-arrow}{teal}
\addvalue{ambiguity-ood-arrow}{violet}
\annotate[yshift=0.4em,xshift=0.3em]{above,right,label above}{ambiguity-kl-arrow}{Dynamics-aware Similarity}
\annotate[yshift=-0.1em,xshift=0.3em]{below,right,label above}{ambiguity-ood-arrow}{In-distribution Constraint}
\label{eq:ambiguity-set}
\end{align}
\endgroup where $\thresholdOOD$ is the OOD threshold and $\calibOOD$ is the error rate. In practice, we instantiate $\densityEstimator$ using a flow-matching-based density estimator (logpZO)~\cite{xu2025can} trained on training latent states, where smaller scores indicate in-distribution.

\subsection{Conformal Calibration of the Uncertainty Set}\label{sec:conformal-calib}

Because the uncertainty set depends on learned latent dynamics and OOD score functions, its thresholds must be calibrated to ensure that plausible transitions are captured without making the uncertainty set overly conservative. We therefore use conformal prediction (CP)~\cite{shafer2008tutorial, angelopoulos2023conformal, marques2024quantifying}, a distribution-free statistical method, to calibrate two quantities: the KL-ball radius $\thresholdKL$ and the OOD threshold $\thresholdOOD$. Since transitions within a trajectory are temporally dependent, we perform both calibrations at the trajectory level using a held-out in-distribution calibration dataset of observation-action trajectories: \(
\mathcal{D}_{\text{calib}}
:=\left\{
(\obs_t^i,\action_t^i,\obs_{t+1}^i)_{t=1}^{T-1}
\right\}_{i=1}^{N_\text{calib}}.
\) Proofs and additional details for this section are provided in Appendix~\ref{sec:appendix-conformal}.

\para{Calibrating the Dynamics-Aware Similarity}
Our goal is to calibrate the KL-ball radius $\thresholdKL$ such that the KL ball around the learned latent dynamics prediction contains the ground-truth latent dynamics realized by the systems with high probability. Mirroring the training objective of the latent dynamics model \eqref{eq:rssm-training-objective}, we adopt the nonconformity score $\KL\left(
\encoder(\latent_t \mid \obs_{\le t})
\,\|\,\dynz(\latent_t \mid \latent_{t-1}, \action_{t-1})
\right),$ and apply conformal prediction with a user-specified miscoverage level $\calibAmbiguity \in [0,1]$. Under the exchangeability assumption, the calibrated uncertainty set guarantees inclusion of the test-time latent representation encoded from observations within the uncertainty set with probability at least $1-\calibAmbiguity$:
\begin{equation}\label{eq:kl-ball-cp}
\begin{aligned}
\mathbb{P}\left(
D_\mathrm{KL}\!\left(
\encoder(\obs^{\mathrm{test}}_{\le t})
\,\|\,
\dynz(\latent^{\mathrm{test}}_{t-1}, \action^{\mathrm{test}}_{t-1})
\right)
\leq \thresholdKL
\right)
\geq 1-\calibAmbiguity.
\end{aligned}
\end{equation}

\para{Calibrating the In-Distribution Constraint} While we aim to classify a latent state as OOD when $\densityEstimator(\latent) > \thresholdOOD$, OOD latent states are, by definition, not directly available~\cite{seo2025uncertainty}; we only have access to in-distribution states. Similar to \cite{seo2025uncertainty}, we therefore adopt class-conditional conformal prediction~\cite{ding2023class} to calibrate $\thresholdOOD$ using in-distribution states encoded from the calibration dataset, providing recall guarantees for detecting in-distribution states with a miscoverage level $\calibOOD \in [0,1]$:
\begin{equation}\label{eq:ood-cp}
\mathbb{P}\left(
\densityEstimator(\latent_\text{ID}^{\text{test}})
\leq \thresholdOOD
\right)
\geq 1-\calibOOD.
\end{equation} Intuitively, an in-distribution latent state is accepted as in-distribution with probability at least $1-\calibOOD$, while latent states whose OOD scores exceed $\thresholdOOD$ are less likely to be in-distribution and are therefore classified as OOD.

\para{Trajectory-Level Calibration} \eqref{eq:kl-ball-cp} and \eqref{eq:ood-cp} assume exchangeability of the calibration and test samples. However, individual transitions within a trajectory are temporally dependent and therefore cannot be treated as exchangeable. To address this, similar to \cite{ren2023robots}, we perform calibration at the \textit{trajectory} level, assuming that the calibration trajectories and the test trajectory are exchangeable, and causally reconstruct the state-level thresholds from the trajectory-level result.

Given a trajectory $\tau_i:=\{(\obs_{t}^i,\action_{t}^i,\obs_{t+1}^i)\}_{t=1}^{T-1} \in \mathcal{D}_{\text{calib}}$ with state-level nonconformity scores $s_t^i$, we define the trajectory-level nonconformity score as $S(\tau_i)\!:=\!\max_t s_t^i$ and compute the $(1\!-\!\alpha)$-quantile of $\{S(\tau_i)\}_{i=1}^{N_{\mathrm{calib}}}$ to obtain the calibrated threshold $\epsilon^\text{traj}$. This trajectory-level threshold can then be causally reconstructed as the state-level threshold:
\begin{equation}
\mathbb{P}\left(
\max_t s_t^{\mathrm{test}} \leq \epsilon^\text{traj}
\right)
=
\mathbb{P}\left(
s_t^{\mathrm{test}} \leq \epsilon^\text{traj},\;\forall t
\right)
\geq 1-\alpha,
\end{equation}
where $(\epsilon,\alpha)\in\{(\thresholdKL,\calibAmbiguity),(\thresholdOOD,\calibOOD)\}$. Thus, the trajectory-level threshold can be causally applied as the corresponding state-level threshold at each test-time transition.

\subsection{Solving Latent-Space Robust Optimization}\label{sec:latent-robust-optim}

With a calibrated model of latent disturbances, we now solve the latent-space robust optimization problem in \eqref{eq:problem-setup-latent-robust-optimization} by jointly optimizing the latent disturbance within the conformalized uncertainty set and the action that minimizes the cost under the imagination induced by the worst-case latent disturbance.

\para{Efficient Disturbance Parameterization}
Directly searching over all admissible dynamics in the uncertainty set is intractable in high-dimensional latent spaces (e.g., $1$K dims). We therefore parameterize the latent-space disturbance as a \textit{residual} perturbation to the learned latent dynamics. Specifically, a latent-space disturbance network $\latentDisturbance$ outputs an additive residual to the next latent state distribution parameters: \begin{equation}\label{eq:disturbance-residual} \latentDisturbance\,(\latent,\action) = \left(\Delta\mu\,(\latent,\action),\Delta\Sigma\,(\latent,\action)\right) \quad \text{or} \quad \Delta\parameter\,(\latent,\action), \end{equation} yielding the latent dynamics conditioned on $\latentDisturbance$ and $\dynz$: \begin{equation}\label{eq:rssm-pessimistic-dynamics} \latentDisturbanceDynamicsParameterized(\latent'\mid\latent,\action) = \mathcal{N}(\mu+\Delta\mu,\Sigma+\Delta\Sigma) \,\,\text{or} \,\, \mathrm{Cat}(\parameter+\Delta\parameter). \end{equation}

\para{Game-Theoretic Optimization}
We compute the worst-case latent disturbance $\eqnhighlight{disturbance_orange}{\optimalDisturbance(\cdot\mid\latent,\action)}$ for each latent state and robot action pair $(\latent,\action)$ by softly relaxing the uncertainty set into hinge penalties. Since both the latent dynamics and OOD score function are differentiable, we reformulate the inner maximization of \eqref{eq:problem-setup-latent-robust-optimization} with coefficients $\penaltyOOD,\penaltyKL>0$:
\begingroup
\setlength{\fboxsep}{0pt}
\setlength{\thinmuskip}{1mu}
\setlength{\medmuskip}{0.5mu}
\setlength{\thickmuskip}{1mu}
\setlength{\jot}{1pt}
\begin{align}
\max_{\latentDisturbanceDynamicsParameterized}\,
\,\mathbb{E}_{\latent'\sim\latentDisturbanceDynamicsParameterized(\latent, \action)}
&\Big[\objective(\latent')\label{eq:disturbance-objective}-\eqnhighlight{violet}{\penaltyOOD
\big[\densityEstimator(\latent')-\thresholdOOD\big]_+}\\[0pt]
&-\eqnhighlight{teal}{\penaltyKL
\Big[
\KL(
\latentDisturbanceDynamicsParameterized (\latent, \action)
\|
\dynz(\latent, \action)
)
-\thresholdKL
\Big]_+}
\Big],\notag
\end{align}
\endgroup where $[x]_+:=\max(x,0)$ applies a penalty only when a constraint is violated. The $\latentDisturbance$ is optimized through differentiable latent dynamics with gradient ascent, and the robust action is then optimized under the resulting latent disturbance:
\begin{equation}\label{eq:robust-policy-objective}
\begingroup
\setlength{\fboxsep}{0pt}
\policy_\text{rob}(\latent)=\arg\min_{\action}
\eqnmarkbox[disturbance_orange]{policy-optimal-disturbance}{\mathbb{E}_{
\latent'\sim\optimalDisturbance(\cdot\mid\latent,\action)
}}
\left[
\objective(\latent')
\right],
\endgroup
\end{equation} where the robust action and latent-space disturbance form a two-player minimax game that can be solved iteratively~\cite{ren2022distributionally}. Following \cite{wang2024magics}, we update the disturbance faster than the action to promote convergence, providing a practical approximation to the latent-space robust optimization problem in \eqref{eq:problem-setup-latent-robust-optimization}.

\section{Application: Robust Policy Steering}

\subsection{Setup: Policy Steering with World Models}\label{sec:application-prelim}

Modeling latent disturbances allows us to instantiate a suite of runtime steering techniques that guide a visuomotor policy robustly away from hard-to-model failures. We consider two steering paradigms: (i) \textit{latent safety filters}, which safeguard a task policy using a safety monitor and safe controller computed from world-model imaginations~\cite{nakamura2025generalizing, seo2025uncertainty, agrawal2026anysafe, kim2025well}; and (ii) \textit{sample-and-verify}, which samples multiple action candidates from a task policy, evaluates their imagined outcomes using the world model, and selects the best action~\cite{wu2025foresight, yuan2026ups}. Existing WM-based policy steering methods predominantly rely on \methodhighlight{nominal_blue}{\textit{nominal} world-model imaginations} both when synthesizing the safety controller and during runtime action selection. As a result, they can overestimate safety under optimistic predictions and produce non-robust safety actions that remain effective only under nominal system realizations.

\para{Failure Specification} Hard-to-model constraints for policies are specified via a latent failure set $\failure := \{\latent : \ellz(z) \leq 0 \} \subset \latentSpace$ defined as the zero-sublevel set of a latent failure margin function $\ellz:\!\latentSpace\!\rightarrow\!\mathbb{R}$. The margin function $\ellz\in[-1,1]$ is, in practice, instantiated as a binary classifier. 

\para{Steering Method \#1: Latent Safety Filters}
A latent safety filter synthesizes two components with WM: a safety value function \(\monitor:\!\latentSpace\!\rightarrow\!\mathbb{R}\), which measures how close the current latent state of the robot is to inevitable failures, and a safety-preserving policy \(\fallback: \latentSpace \rightarrow \actionSpace\), which steers the robot away from failure. The latent safety filter computes these functions using the learned latent dynamics $\dynz$ through Hamilton--Jacobi (HJ) reachability analysis~\cite{mitchell2005time, hsu2023safety}, satisfying the latent-space safety Bellman equation~\cite{nakamura2025generalizing}:
\begingroup
\setlength{\fboxsep}{0pt}
\renewcommand{\eqnannotationfont}{\small}
\setlength{\abovedisplayskip}{4pt}
\setlength{\belowdisplayskip}{4pt}
\setlength{\abovedisplayshortskip}{4pt}
\setlength{\belowdisplayshortskip}{4pt}
\setlength{\jot}{1pt}
\begin{subequations}\label{eq:trad-fixed-pt-bellman}
\begin{align}
        \rule{0pt}{2.2em}
        \monitor(\latent)
        &= \min \Big\{ \ellz(\latent),
        \max_{\action \in \actionSpace}
        \eqnmarkbox[nominal_blue]{nominal-value-expectation}{\mathbb{E}_{\latent'\sim \dynz(\cdot\mid\latent, \action)}}
        \left[\monitor(\latent')\right] \Big\},
        \begin{tikzpicture}[overlay,remember picture]
        \coordinate (nominal-expectation-arrow) at (nominal-value-expectation.north);
        \end{tikzpicture}%
        \addvalue{nominal-expectation-arrow}{nominal_blue}%
        \annotate[yshift=0.5em,xshift=-0.3em]{above,left,label above}{nominal-expectation-arrow}{Nominal Latent Dynamics}%
        \label{eq:trad-fixed-pt-bellman-value}\\
        \fallback(\latent)
        &= \arg\max_{\action \in \actionSpace}
        \eqnmarkbox[nominal_blue]{nominal-policy-expectation}{\mathbb{E}_{\latent'\sim \dynz(\cdot\mid\latent, \action)}}
        \left[\monitor(\latent')\right], \label{eq:trad-fixed-pt-bellman-policy}
\end{align}
\end{subequations}
\endgroup
\noindent where a discount factor $\gamma\in[0,1)$ can be incorporated to induce a contraction mapping~\cite{fisac2019bridging}. The zero-sublevel set of the safety value,
$\unsafeSet:=\{\latent:\monitor(\latent)\leq0\}$,
defines the unsafe set of latent states from which failure cannot be avoided.

At runtime, the safety filter safeguards an arbitrary task policy $\policyTask$ by evaluating its proposed action in a least-restrictive fashion~\cite{bansal2017hamilton}. Given the current latent state $\latent$, it estimates the expected safety value of next states {\setlength{\fboxsep}{0pt}$\eqnhighlight{nominal_blue}{\latent'\sim\dynz(\cdot\mid\latent,\policyTask(\latent))}$}, and the task action is executed if the predicted safety value is positive; otherwise, it is overridden by the fallback policy: \begin{equation}\label{eq:trad-filter-rollout}
\begingroup
\setlength{\fboxsep}{0pt}
\policy^{\text{steer}}(\latent) = \mathds{1}_{\{\eqnhighlight{nominal_blue}{\mathbb{E}_{\latent'}[\monitor(\latent')]} > 0\}} \cdot \policyTask(\latent) + \mathds{1}_{\{\eqnhighlight{nominal_blue}{\mathbb{E}_{\latent'}[\monitor(\latent')]} \le 0\}} \cdot \fallback(\latent).
\endgroup
\end{equation}

\para{Steering Method \#2: Sample-and-Verify}
Given $K$ candidate action sequences $\{\mathbf{a}^{k}\}_{k=1}^{k=K}\!\sim\!\policyTask(\latent)$ sampled from a task policy (e.g., diffusion policy~\cite{chi2023diffusionpolicy}), we evaluate each candidate using WM imaginations and select the best action sequence:
\begin{equation}
\label{eq:sampling-based-steering}
\begingroup
\setlength{\fboxsep}{0pt}
\policy^{\text{steer}}(\latent)
= \arg\min_{\mathbf{a}\in\{\mathbf{a}^{(k)}\}_{k=1}^{K}}
\eqnhighlight{nominal_blue}{
\mathbb{E}_{\mathbf{z}'\sim\dynz(\cdot\mid\latent,\mathbf{a})}
}
\left[\objective(\mathbf{z}')\right],
\endgroup
\end{equation}
where $\mathbf{z}'\!:=\!z_{i:i+H}$ denotes the latent states autoregressively generated over the action chunk of length $H$ using nominal imagination
{\setlength{\fboxsep}{0pt}$\eqnhighlight{nominal_blue}{\latent_{i+1}^{k}\sim\dynz(\cdot\mid\latent_i^{k},a_i^{k})}$} starting from $\latent_i$,
and $\objective$ is a cost function that scores its outcome, such as a latent failure margin~\cite{jain2026texttt}, safety value function~\cite{kim2025well}, or a VLM-based verifier applied to decoded images~\cite{yuan2026ups}.

\para{Limitation: Nominal Outcome under Uncertainty} While policy steering methods with latent world models show promise in steering robots away from hard-to-model failures~\cite{nakamura2025generalizing, seo2025uncertainty}, they can become non-robust in tasks with uncertain dynamics because both the synthesis in \eqref{eq:trad-fixed-pt-bellman} and the runtime evaluation in \eqref{eq:trad-filter-rollout} and \eqref{eq:sampling-based-steering} rely only on the \methodhighlight{nominal_blue}{nominal WM imaginations}. When disturbances in the underlying system can induce qualitatively different outcomes for the same action, the robot may choose non-robust actions based on overly optimistic evaluations.

\subsection{Robustifying Policy Steering}\label{sec:robust-safety-filter}
Both policy-steering paradigms in Sec.~\ref{sec:application-prelim} require a latent disturbance that drives the system toward failure through a sequence of plausible future transitions, rather than a single worst-case transition. To learn such a worst-case latent disturbance, we instantiate \ours using latent-space Hamilton--Jacobi--Isaacs~(HJI) reachability analysis~\cite{mitchell2005time}, which recursively evaluates future safety over successive latent transitions.

\para{Latent-Space HJI Reachability Analysis}
With the conformalized latent-space disturbance $\latentDisturbanceDynamics\in\ambiguity$, we formulate infinite-horizon robust decision making by replacing \eqref{eq:trad-fixed-pt-bellman} with a latent-space Bellman--Isaacs fixed-point equation~\cite{mitchell2005time, hsu2023isaacs}, the discrete-time dynamic-programming counterpart of Hamilton--Jacobi--Isaacs (HJI) reachability analysis:
\vspace{-0.05in}
\begin{equation}\label{eq:robust-fixed-pt-bellman}
\begingroup
\setlength{\fboxsep}{0pt}
\setlength{\thinmuskip}{1.5mu}
\setlength{\medmuskip}{1.5mu}
\setlength{\thickmuskip}{1.5mu}
\setlength{\abovedisplayshortskip}{2pt}
\setlength{\belowdisplayshortskip}{4pt}
\begin{aligned}
\monitorRobust(\latent)
\!=\!
\min\Bigg\{
\!\ellz(\latent)\!,
\max_{\action\in\actionSpace}\!
\eqnhighlight{disturbance_orange}{\min_{\latentDisturbanceDynamics\in\eqnhighlight{teal}{\ambiguity}}
\!\mathbb{E}_{\latent'\sim\latentDisturbanceDynamics(\cdot\mid\latent,\action)}}
\big[\monitorRobust(\latent')\big]
\Bigg\}.
\end{aligned}
\endgroup
\end{equation} Similarly, the zero-sublevel set $\unsafeSet_\text{rob}:=\{\latent:\monitorRobust(\latent)\leq0\}$ defines a robust latent unsafe set, and the corresponding robust safety policy and worst-case latent-space disturbance $\optimalDisturbance$ are then formulated as a two-player game optimizing against $\monitorRobust$:
\begingroup
\setlength{\fboxsep}{0pt}
\setlength{\thinmuskip}{2mu}
\setlength{\medmuskip}{2mu}
\setlength{\thickmuskip}{2mu}
\setlength{\abovedisplayskip}{0pt}
\setlength{\belowdisplayskip}{3pt}
\setlength{\abovedisplayshortskip}{0pt}
\setlength{\belowdisplayshortskip}{3pt}
\vspace{-0.05in}
\begin{subequations}\label{eq:robust-fixed-pt-bellman-optimizer}
\begin{align}
\eqnhighlight{disturbance_orange}{\optimalDisturbance(\cdot\mid\latent,\action)}
&=
\argmin_{\latentDisturbanceDynamics(\latent,\action)\in\eqnhighlight{teal}{\ambiguity(\latent,\action)}}
\mathbb{E}_{\latent'\sim\latentDisturbanceDynamics(\cdot\mid\latent,\action)}
\big[\monitorRobust(\latent')\big],
\label{eq:robust-fixed-pt-bellman-disturbance}
\\
\fallbackRobust(\latent)
&=
\argmax_{\action\in\actionSpace}
\eqnhighlight{disturbance_orange}{\mathbb{E}_{\latent'\sim\optimalDisturbance(\cdot\mid\latent,\action)}}
\big[\monitorRobust(\latent')\big].
\label{eq:robust-fixed-pt-bellman-policy}
\end{align}
\end{subequations}
\endgroup

\para{Computation via Adversarial RL}
To compute the solution of the latent-space HJI reachability analysis, we approximate the dynamic game in \eqref{eq:robust-fixed-pt-bellman}--\eqref{eq:robust-fixed-pt-bellman-optimizer} using off-policy actor--critic learning with Soft Actor-Critic (SAC)~\cite{haarnoja2018soft}, using the WM as a simulator. We learn three networks: (i) a robust safety critic $\QfunctionRobust(\latent,\action)$, with $\monitorRobust(\latent)=\QfunctionRobust(\latent,\fallbackRobust(\latent))$;
(ii) a latent-space disturbance network $\latentDisturbance(\latent,\action)$ that predicts residual distribution parameters for the nominal latent dynamics $\dynz(\latent,\action)$, yielding the pessimistic dynamics $\latentDisturbanceDynamicsParameterized(\latent,\action)$, as in \eqref{eq:disturbance-residual}; and
(iii) a robust safety policy $\fallbackRobust(\latent)$. 

Starting from latent states encoded from $\dataset_\text{train}$, we generate imaginations with $\fallbackRobust$ and $\latentDisturbanceDynamicsParameterized$, and store the latent-action pairs $\{(\latent, \action)\}$ in a replay buffer $\buffer$. At each iteration, we sample transitions from $\buffer$ and update the robust safety critic using the discounted safety Bellman equation, with $\gamma\in[0,1)$~\cite{fisac2019bridging}:
\begin{equation}
    \begin{aligned}
    \mathcal{L}_{\QfunctionRobust}
    =
    \mathbb{E}_\buffer\Big[
    &\Big(
    \QfunctionRobust(\latent ,\action)
    -
     \Big(
    (1-\gamma)\,\ellz(\latent)
    \\
    & +\gamma\,\min\!\left\{
    \ellz(\latent),
    \targetQfunctionRobust(\latent',\action')
    \right\}
    \Big)
    \Big)^2
    \Big],
    \end{aligned}\label{eq:discounted-safety-value}
\end{equation}
where the next latent state {\setlength{\fboxsep}{0pt}$\eqnhighlight{disturbance_orange}{\latent'\sim\latentDisturbanceDynamicsParameterized(\cdot\mid\latent,\action)}$} is sampled from the latent dynamics induced by the current disturbance policy, and $\action'\sim\fallbackRobust(\cdot\mid\latent')$ is sampled from the current robust safety policy. The latent-space disturbance $\latentDisturbanceDynamicsParameterized$ is then updated to minimize the robust safety value while remaining within the calibrated uncertainty set, following \eqref{eq:disturbance-objective}:
\begingroup
\setlength{\fboxsep}{0pt}
\setlength{\thinmuskip}{2mu}
\setlength{\medmuskip}{2mu}
\setlength{\thickmuskip}{2mu}
\setlength{\abovedisplayskip}{5pt}
\setlength{\belowdisplayskip}{5pt}
\begin{subequations}\label{eq:robust-latent-safety-filters-disturbance-optim}
\renewcommand{\mintagsep}{2pt}
\begin{align}
&\mathcal{L}_{\latentDisturbance}:=
\mathbb{E}_\buffer
\Big[
\underbrace{\QfunctionRobust(\latent', \fallbackRobust(\latent'))}_{\text{Minimize Value}}
+
\underbrace{\eqnhighlight{violet}{\penaltyOOD\,\big[\densityEstimator(\latent')-\thresholdOOD\big]_+}}_{\text{In-Distribution Constraint}}
\notag\\[0.2em]
&\quad+
\underbrace{\eqnhighlight{teal}{\penaltyKL\,
\big[\KL (\latentDisturbanceDynamicsParameterized(\cdot \mid \latent, \action) \,\|\, \dynz(\cdot \mid \latent, \action)) - \thresholdKL \big]_+}}_{\text{Dynamics-aware Uncertainty Set Constraint}}
\Big],\label{eq:disturbance-loss}\\[0.5em]
&\mathcal{L}_{\fallbackRobust}:=
\mathbb{E}_\buffer
\Big[\underbrace{-\QfunctionRobust(\latent, \action')}_{\text{Maximize Value}}
+
\underbrace{\alpha_{\text{entropy}}\, \log \fallbackRobust(\action' \mid \latent)}_{\text{SAC Entropy Regularization}}\Big].
\label{eq:safety-policy-loss}
\end{align}
\end{subequations}
\endgroup
\noindent where $\alpha_{\text{entropy}}$ is the SAC entropy-regularization coefficient.

\para{Runtime \#1: Robust Latent Safety Filters} Since the robust safety value is computed under worst-case latent dynamics, its zero-sublevel set $\unsafeSet_\text{rob}:=\{\latent:\monitorRobust(\latent)\leq0\}$ defines the robust unsafe set, with $\monitorRobust(\latent)=\QfunctionRobust(\latent,\fallbackRobust(\latent))$: the latent states from which failure is unavoidable under adverse uncertainty. At deployment, we therefore evaluate the action proposed by the task policy using the learned robust safety value, and the safety policy overrides it when the value is non-positive:
\begin{equation}\label{eq:robust-filter-rollout}
    \policy^{\text{steer}}_{\text{rob}}(\latent)
    =
    \begin{cases}
    \policyTask(\latent),
    & \text{if } \QfunctionRobust\big(\latent, \policyTask(\latent)\big) > 0, \\[0.3em]
    \fallbackRobust(\latent),
    & \text{otherwise}.
    \end{cases}
\end{equation}

\para{Runtime \#2: Robust Sample-and-Verify}
Replacing the nominal imaginations in \eqref{eq:sampling-based-steering} with pessimistic imaginations induced by the learned worst-case latent-space disturbance $\optimalDisturbance$, we evaluate $K$ candidate action sequences $\mathbf{a}^{(k)}\!\sim\!\policyTask(\latent)$ with $k=1,\cdots,K$, and select the best sequence:
\begin{equation}
\label{eq:robust-sampling-based-steering}
\begingroup
\setlength{\fboxsep}{0pt}
\policy^{\text{steer}}_{\text{rob}}(\latent)
= \arg\min_{\mathbf{a}\in\{\mathbf{a}^{(k)}\}_{k=1}^{K}}
\eqnhighlight{disturbance_orange}{
\mathbb{E}_{\mathbf{z}'\sim\latentDisturbanceDynamicsParameterized(\cdot\mid\latent,\mathbf{a})}
}
\left[\objective(\mathbf{z}')\right],
\endgroup
\end{equation}
where $\mathbf{z}'$ denotes the imagined latent trajectory autoregressively generated over the action chunk $\latent_{i+1}^{k}\sim\latentDisturbanceDynamicsParameterized(\cdot\mid\latent_i^{k},a_i^{k})$, so that evaluation of sampled actions accounts for adverse outcomes under the calibrated system uncertainty.

\section{Case Study: Dubins' Car World Model with Known Disturbances}\label{sec:dubins}

We first study our latent-space robust optimization in controlled settings where the underlying system dynamics and disturbances are known, allowing us to evaluate whether our latent-space robust optimization can recover the optimal solution directly in the latent space of world models. Throughout this section, we focus on the latent safety filter paradigm.

\subsection{Experimental Setup}

\para{Setup: 3D Dubins' Car with Disturbances} Let a discrete-time 3D Dubins' car evolve under two different types of disturbances. First, we use the \textit{naughty} dynamics introduced in Sec.~\ref{sec:setup-latent-ro}, where the system randomly flips the sign of the action, inducing a discrete, multimodal transition distribution with two possible outcomes for each action. Second, we consider continuous \textit{positional disturbances}, $d_{x,t},d_{y,t}\in[-0.3,0.3]$: \begin{equation}\label{eq:dubins-bounded} 
\state_{t+1}
    \!=\!
    \state_t
    +
    \Delta t\,[v\cos(\theta_t)+d_{x,t},\,\,v\sin(\theta_t)+d_{y,t},\,\,\action_t \,]. 
\end{equation} For both settings, the ground-truth worst-case disturbance and corresponding robust action can be computed analytically, enabling evaluation of our latent-space robust optimization.

\para{World Model}
We adopt Dreamer~\cite{Hafner2020Dream} with a Recurrent State-Space Model (RSSM)~\cite{hafner2019learning} where latent dynamics are modeled as a Gaussian distribution~\cite{hafner2019learning}. The world model is trained offline on $N_{\text{train}}=4{,}000$ observation--action trajectories, where actions and disturbances are sampled randomly for $T=100$ timesteps. We annotate each observation with the binary ground-truth failure margin $\failuremargin(\state)=-1$ if $\state\in\failure$ and $\failuremargin(\state)=1$ otherwise, and train a two-layer MLP on top of the latent state to predict the latent failure margin $\ellz(\latent)$.

\para{Latent Safety Filter Setup} We follow Sec.~\ref{sec:robust-safety-filter} to learn the latent safety filter using the pretrained world model. The latent disturbance takes the mean and covariance predicted by the nominal latent dynamics and outputs additive perturbations to both. The thresholds $\thresholdKL$ and $\thresholdOOD$ are calibrated on a held-out set of $1{,}000$ trajectories with $\calibAmbiguity=\calibOOD=0.01$.

\para{Evaluation \& Metrics} With access to system dynamics, we compute the ground-truth safety value function and optimal robust action using grid-based methods~\cite{mitchell2007toolbox}. We first evaluate the safety value $\monitor$ by its accuracy in classifying unsafe states across three state dimensions, discretized into $50$ grids per dimension, reporting balanced accuracy (B.Acc.), false positive rate (FPR), and false negative rate (FNR). We then evaluate the learned safety policy with $1{,}000$ random trajectories, comprising $500$ safe and $500$ that fail under the worst-case disturbance. We apply least-restrictive filtering \eqref{eq:robust-filter-rollout} on those action sequences starting from the initial states in a closed loop, reporting the \textit{failure rate} (Fail).

\begin{table}[t]
\centering
\caption{\textbf{3D Dubins' Car World Model: Quantitative Results.}}
\label{tab:dubins-value-classification}
\vspace{-0.1in}
\renewcommand{\arraystretch}{1.1}
\resizebox{1.0\linewidth}{!}{
\begin{tabular}{rcccccccc}
\toprule
& \multicolumn{4}{c}{\textbf{Naughty}}
& \multicolumn{4}{c}{\textbf{Positional Disturbances}} \\
\cmidrule(lr){2-5}\cmidrule(lr){6-9}
\textbf{Method}
& \textbf{B.Acc.} $\uparrow$
& \textbf{FPR} $\downarrow$
& \textbf{FNR} $\downarrow$
& \textbf{Fail} $\downarrow$
& \textbf{B.Acc.} $\uparrow$
& \textbf{FPR} $\downarrow$
& \textbf{FNR} $\downarrow$
& \textbf{Fail} $\downarrow$ \\
\midrule
Nominal
& 0.932 & 0.006 & 0.130 & 0.114
& 0.831 & 0.000 & 0.337 & 0.218 \\
WoN
& 0.959 & 0.057 & 0.026 & 0.061
& 0.911 & 0.162 & 0.017 & 0.034 \\
$\CVaR_{0.05}$
& 0.944 & 0.017 & 0.095 & 0.062
& 0.943 & 0.028 & 0.087 & 0.084 \\
$\CVaR_{0.1}$
& 0.935 & 0.012 & 0.118 & 0.061
& 0.936 & 0.018 & 0.109 & 0.101 \\
$\CVaR_{0.2}$
& 0.927 & 0.010 & 0.135 & 0.062
& 0.927 & 0.011 & 0.135 & 0.134 \\
$\CVaR_{0.5}$
& 0.915 & 0.007 & 0.163 & 0.057
& 0.910 & 0.005 & 0.175 & 0.197 \\
\rowcolor{gray!25}
\textbf{\ours}
& 0.958 & 0.037 & 0.046 & 0.038
& 0.956 & 0.023 & 0.066 & 0.021 \\
\bottomrule
\end{tabular}
} \vspace{-0.2in}
\end{table}

\begin{figure}[t]
\centering
\includegraphics[width=1.0\linewidth]{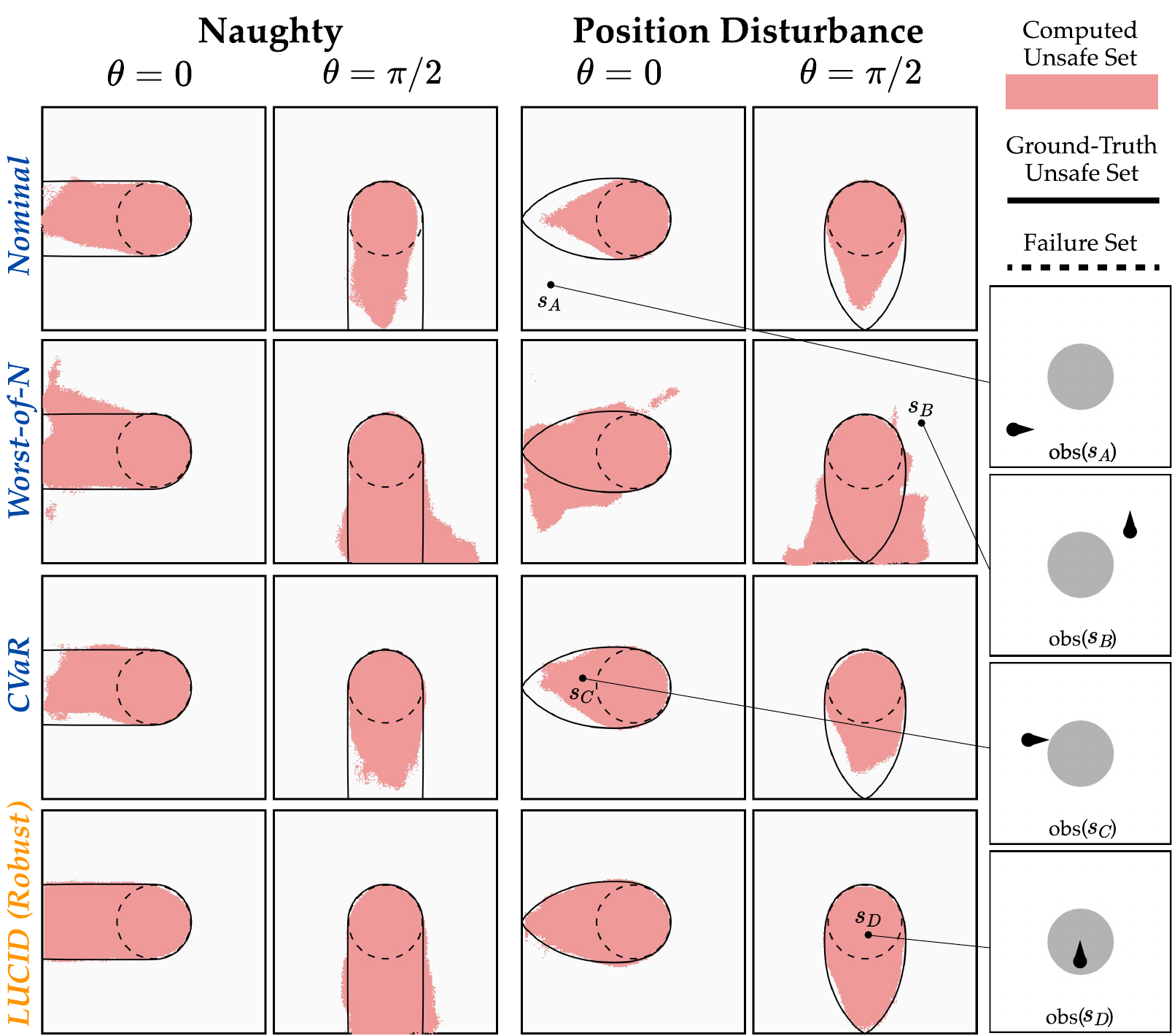}\vspace{-0.1in}
\caption{\textbf{Qualitative Results: 3D Dubins' Car.} Solid lines represent the ground-truth unsafe-set boundaries, and red regions denote the unsafe sets computed by each method. \ours closely approximates the ground-truth unsafe sets of robust optimization under two different types of disturbances. In contrast, \nominal is overly optimistic, while the sampling-based baselines such as \textit{WoN} and \textit{CVaR} inaccurately approximate unsafe sets.}
\label{fig:dubins-result}
\vspace{-0.2in}
\end{figure}

\subsection{Can \ours Make Robust Decisions in World Models?}

\para{Baselines} We compare \ours against the following latent safety-filter baselines, which differ in how they account for uncertainty in the learned latent dynamics, using the same WM. We first consider \nominal, which trains the safety filter following \eqref{eq:trad-fixed-pt-bellman} that evaluates the expected safety value:
\begin{equation}
\begingroup
\setlength{\fboxsep}{0pt}
\begin{aligned}
\monitorNominal(\latent)
\!=\!
\min\left\{
\ellz(\latent),
\max_{\action\in\actionSpace}
\eqnhighlight{nominal_blue}{\E_{\latent'\sim\dynz(\cdot\mid\latent,\action)}}
\left[\monitorNominal(\latent')\right]
\right\}
\end{aligned}
\endgroup
\end{equation}
\textit{Worst-of-$N$ (WoN)} samples $N$=$10$ states from the learned dynamics {\setlength{\fboxsep}{0pt}$\eqnhighlight{nominal_blue}{\latent'_i\!\sim\!\dynz(\cdot\!\mid\!\latent,\action)}$} and uses the least-safe sample:
\begin{equation}
\monitorWoN(\latent)
=
\min\left\{
\ellz(\latent),
\max_{\action\in\actionSpace}
\min_{i=1,\ldots,N}
\monitorWoN(\latent'_i)
\right\}.
\end{equation}
Lastly, \textit{Risk-Sensitive} (\textit{CVaR}) learns a distributional value function~\cite{Marc2017Distributional,Dabney2018Quantile} from sampled next states of learned latent dynamics and evaluates each action using the conditional value-at-risk,
$\CVaR_{\beta}$, with $\beta\in\{0.05,0.1,0.2,0.5\}$:
\begin{equation}
\monitorCVaR(\latent)
=
\min\left\{
\ellz(\latent),
\max_{\action\in\actionSpace}
\CVaR_{\beta}
\left(
\monitorCVaR(\latent')
\right)
\right\}.
\end{equation}
They are all trained using the discounted safety-value formulation in \eqref{eq:discounted-safety-value}, with the safety policy optimized to maximize the learned safety value. However, unlike \ours, these baselines account only for uncertainty captured by finite samples from the nominal learned dynamics and do not explicitly optimize a latent disturbance over the calibrated uncertainty set.

\para{Result: \ours Enables Robust Decision Making in Latent Space}
Table~\ref{tab:dubins-value-classification} shows that \ours learns a robust latent safety filter under both discrete and continuous system uncertainty, improving both the open-loop accuracy of the safety value function and closed-loop safety performance. As shown in Fig.~\ref{fig:dubins-result}, the learned safety value accurately identifies states from which the worst-case plausible disturbance can induce failure despite the best control action. In contrast, \nominal evaluates safety by averaging over nominal world-model futures and can therefore be overly optimistic: its high FNR indicates that many ground-truth unsafe states are predicted as safe, resulting in higher failure rates in closed-loop performance.

\para{Result: Finite Sampling Limits Robust Decision Making} A simple way to account for uncertain latent dynamics is to use finite samples from the nominal next-state distribution, as in \textit{WoN} or \textit{Risk-Sensitive} with distributional critics. However, such zeroth-order optimization can miss adverse but plausible outcomes, especially as the space of possible transitions grows. Accordingly, \textit{WoN} degrades more relative to \ours under continuous disturbances than under the binary \textit{naughty} disturbance. \textit{Risk-Sensitive} also depends on the CVaR risk level: increasing it toward $0.5$ emphasizes typical rather than adverse returns, degrading worst-case value estimation. These results motivate explicitly computing a worst-case latent disturbance that leads to plausibly pessimistic WM imaginations, using gradient ascent to search for adverse transitions within the calibrated uncertainty set rather than relying only on finite nominal sampling, consistent with the findings of~\cite{nishimura2023rap}.

\subsection{Does the Uncertainty Set Represent Plausible Dynamics?}

To evaluate the conformalized uncertainty set in Sec.~\ref{sec:disturbance-set} and Sec.~\ref{sec:conformal-calib}, we ablate whether it effectively includes plausible transitions without admitting implausible ones.

\para{Setup: Uncertainty Set Ablation} To test the conformalized uncertainty set, we ablate latent safety-filter training with the following setups: (i) \textit{Euclidean}, which replaces the dynamics-aware similarity metric with the isotropic distance to the mean next-state prediction, with its threshold similarly calibrated; (ii) \textit{w/o OOD}, which removes the in-distribution constraint; and (iii) \textit{Uncalibrated} KL thresholds, which use either $0.5\thresholdKL$ or $2\thresholdKL$ instead of the conformally calibrated threshold.

\begin{table}[t]
\caption{\textbf{3D Dubins' Car: Uncertainty Set Ablation.} Dyn., Conf., and OOD represent dynamics-aware similarity, conformal calibration, and in-distribution constraints, respectively.} \vspace{-0.1in}
\label{tab:dubins-ambiguity-ablation}
\centering
\renewcommand{\arraystretch}{1.2}
\resizebox{1.0\linewidth}{!}{
\begin{tabular}{r|ccc|cccccc}
\toprule
& \multicolumn{3}{c|}{\textbf{Uncertainty Set}}
& \multicolumn{3}{c}{\textbf{Naughty}}
& \multicolumn{3}{c}{\textbf{Positional Disturbances}} \\
\cmidrule(lr){2-4}\cmidrule(lr){5-7}\cmidrule(lr){8-10}
\textbf{Method}
& \textbf{Dyn.}
& \textbf{Conf.}
& \textbf{OOD}
& \textbf{B.Acc.} $\uparrow$
& \textbf{FPR} $\downarrow$
& \textbf{FNR} $\downarrow$
& \textbf{B.Acc.} $\uparrow$
& \textbf{FPR} $\downarrow$
& \textbf{FNR} $\downarrow$ \\
\midrule

\rowcolor{gray!25}
\multicolumn{1}{l|}{\textbf{\ours}}
& \cmark & \cmark & \cmark
& 0.958 & 0.037 & 0.046
& 0.956 & 0.023 & 0.066 \\

\textit{w/o OOD}
& \cmark & \cmark & \xmark
& 0.809 & 0.381 & 0.001
& 0.621 & 0.759 & 0.000 \\

$0.5\thresholdKL$
& \cmark & \xmark & \cmark
& 0.927 & 0.011 & 0.136
& 0.864 & 0.001 & 0.271 \\

$2.0\thresholdKL$
& \cmark & \xmark & \cmark
& 0.926 & 0.126 & 0.022
& 0.630 & 0.733 & 0.007 \\

\textit{Euclidean}
& \xmark & \cmark & \cmark
& 0.605 & 0.790 & 0.000
& 0.875 & 0.001 & 0.250 \\

\bottomrule
\end{tabular}
}\vspace{-0.2in}
\end{table}

\para{Result: The Conformalized Uncertainty Set Prevents Over-Pessimism}
Table~\ref{tab:dubins-ambiguity-ablation} shows that robust optimization with uncertainty-\textit{unaware} uncertainty sets can become overly pessimistic by admitting implausible latent transitions. Using a dynamics-unaware similarity metric such as Euclidean distance can include geometrically close but dynamically infeasible latent states or exclude plausible but distant ones, yielding high FPRs or FNRs in the safe set computation. Likewise, removing the in-distribution constraint (\textit{w/o OOD}) allows the latent disturbance to exploit implausible OOD states and produce unrealistic imaginations, again resulting in high FPRs. In contrast, \ours combines a dynamics-aware similarity metric with the in-distribution constraint using OOD detection to restrict the disturbance to plausible latent transitions.

\para{Result: Conformal Calibration Balances Coverage and Pessimism}
We further test whether conformal calibration selects an appropriate uncertainty-set size by ablating the calibrated KL threshold $\thresholdKL$, which controls the radius around the nominal latent dynamics. Changing the threshold to $0.5\times\thresholdKL$ makes the uncertainty set too narrow and yields a high FNR, since worst-case transitions can fall outside the set and the disturbance fails to capture adverse outcomes that may occur under the true system. Conversely, increasing the threshold to $2\times\thresholdKL$ makes the uncertainty set too broad and substantially increases the FPR, as the disturbance can exploit overly pessimistic transitions that are not plausible. These results show that conformal calibration is important for balancing coverage of plausible latent dynamics against excessive conservatism.

\subsection{Ablation: Optimistic Imagination with Latent Disturbances}

To further evaluate whether the latent disturbance can effectively perturb latent dynamics while preserving plausible world-model imaginations, we consider an \emph{optimistic} imagination ablation in the \textit{naughty} Dubins' car setting.

\para{Setup: Optimistic Latent Disturbance} By optimizing both the robot action and latent disturbance to \textit{maximize} safety, we replace the minimax problem with a max--max objective, further assessing \ours's ability to identify alternative latent transitions within the uncertainty set while preserving plausibility. Specifically, we modify the disturbance objective in \eqref{eq:robust-latent-safety-filters-disturbance-optim} to maximize the safety value while remaining within the calibrated uncertainty set. In the \textit{naughty} Dubins' car setting, this optimistic formulation corresponds to the disturbance-free system. We therefore use its ground-truth safety value without any disturbance as the evaluation reference.

\para{Result: Latent Disturbances Induce Effective and Plausible Imaginations}
Table~\ref{tab:dubins-optimistic-steering} shows that \ours learns a safety value function close to that of the disturbance-free system through optimistic imagination, despite using a world model trained on trajectories collected under disturbances. The conformalized uncertainty set is essential for preserving plausibility: Fig.~\ref{fig:dubins_ablation} shows that, without the in-distribution constraint, the disturbance can trivially generate implausible imaginations that avoid failure even from ground-truth unsafe states, leading to a high FNR. Dynamics-unaware uncertainty sets similarly yield high FNRs, while miscalibrated KL thresholds degrade the accuracy of value function computation.

\begin{table}[t]
\centering
\caption{\textbf{Ablation: Optimistic Latent Disturbance.}}
\label{tab:dubins-optimistic-steering}
\vspace{-0.1in}
\scriptsize
\renewcommand{\arraystretch}{1.0}
\begin{tabular}{>{\raggedleft\arraybackslash}p{\dimexpr0.2\linewidth-2\tabcolsep\relax}
*{4}{>{\centering\arraybackslash}p{\dimexpr0.2\linewidth-2\tabcolsep\relax}}}
\toprule
\multicolumn{1}{l}{\textbf{Method}}
& \textbf{B.Acc.} $\uparrow$
& \textbf{F1} $\uparrow$
& \textbf{FPR} $\downarrow$
& \textbf{FNR} $\downarrow$ \\
\midrule
\multicolumn{1}{l}{\textit{Nominal}}
& 0.924 & 0.748 & 0.122 & 0.031 \\
\multicolumn{1}{l}{\textit{WoN}}
& 0.966 & 0.898 & 0.037 & 0.032 \\
\multicolumn{1}{l}{$\CVaR_{0.2}$}
& 0.942 & 0.860 & 0.046 & 0.069 \\
\rowcolor{gray!25}
\multicolumn{1}{l}{\textbf{\ours}}
& 0.966 & 0.947 & 0.009 & 0.059 \\
\textit{w/o OOD}
& 0.841 & 0.810 & 0.000 & 0.318 \\
$0.5\thresholdKL$
& 0.956 & 0.861 & 0.055 & 0.034 \\
$2.0\thresholdKL$
& 0.945 & 0.939 & 0.001 & 0.109 \\
\textit{Euclidean}
& 0.940 & 0.921 & 0.008 & 0.113 \\
\bottomrule
\end{tabular}
\vspace{-0.15in}
\end{table}

\begin{figure}[t]
\centering
\includegraphics[width=1.0\linewidth]{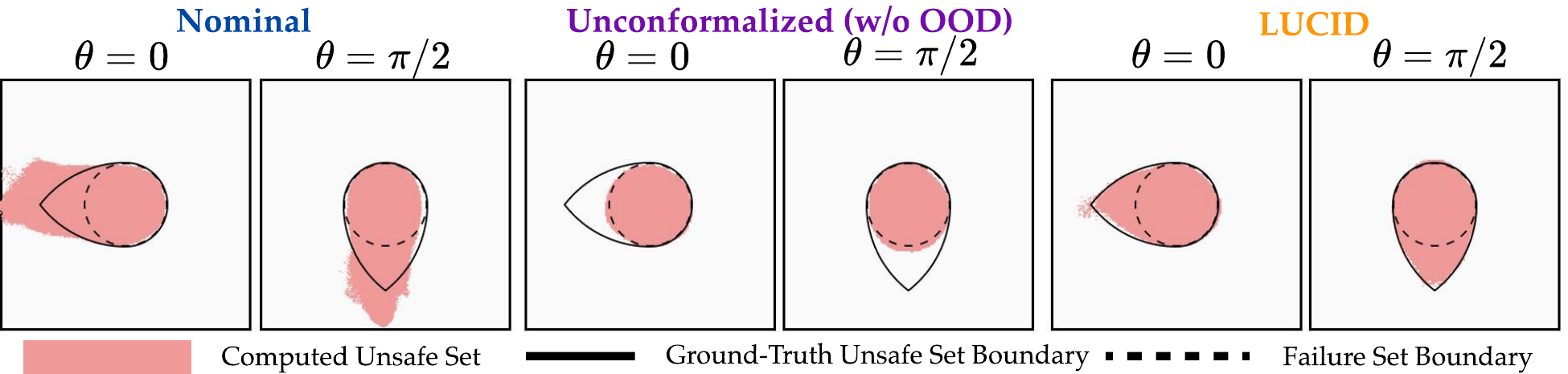}
\vspace{-0.2in}
\caption{\textbf{Qualitative: Naughty Dubins' Car with Optimistic Disturbances.} Solid lines show the ground-truth unsafe-set boundaries of the disturbance-free system, and dashed lines show the failure set. Without the in-distribution constraint, the disturbance exploits implausible latent states, trivially imagining trajectories that avoid failure from any non-failure state.}
\label{fig:dubins_ablation}
\vspace{-0.25in}
\end{figure}
\section{Simulation: Scaling Latent-Space Robust Optimization to Vision-Based Manipulation} \label{sec:manipulation-sim}

We scale our latent-space robust optimization to a vision-based manipulation task, once again focusing on the latent safety filtering paradigm. We use IsaacLab~\cite{mittal2025isaac}, which allows us to control system parameters, such as friction coefficients, that induce uncertainty in the system dynamics.

\subsection{Experimental Setup: Vision-Based Block Pouring}

\para{Setup: Block Pouring} A Franka manipulator must pick up and tilt an orange block to slide a green block placed on top onto a blue block without falling off, as shown in Fig.~\ref{fig:block-pouring-value}. Observations consist of two tabletop $3\times128\times128$ RGB camera images and $7$-D joint-position proprioception. Actions are $6$-DoF end-effector delta poses with a discrete gripper command.

\para{Disturbance} To introduce uncertainty in the system dynamics, we randomize physics parameters including (i) the mass, (ii) friction coefficients, and (iii) restitution of the green block, without providing these parameter values to the robot.

\para{World Model} We adopt DreamerV3~\cite{hafner2023dreamerv3} with discrete latent states modeled by categorical latent dynamics, and train it on approximately $3{,}000$ trajectories containing both successes and failures. The failure-margin function is trained on top of the latent states using state-level ground-truth failure labels.

\para{Latent Safety Filter Setup} In this section, we focus on latent safety filtering for policy steering using the least-restrictive filter in \eqref{eq:robust-filter-rollout}, which safeguards the task policy by intervening with the learned safety policy whenever the safety value of the task-policy action becomes non-positive. We follow \cite{seo2025uncertainty} to learn the latent safety filter, and the latent disturbance takes the categorical logits predicted by the latent dynamics model and outputs additive perturbations. We calibrate the uncertainty set using $200$ held-out calibration trajectories.

\para{Task Policies} We evaluate the learned latent safety filter with three task policies $\policyTask$: (i) \textit{DreamerV3}~\cite{hafner2023dreamerv3}, a visuomotor policy trained with dense task rewards independently from the world model used for the safety filter, (ii) \textit{Diffusion Policy}~\cite{chi2023diffusionpolicy}, an imitation-learning policy trained on $500$ successful trajectories, and (iii) replaying successful human teleoperation trajectories. For each policy, we roll out $2{,}000$ trajectories with randomized initial states and physics parameters.

\begin{table}[t]
\centering
\caption{\textbf{Quantitative Results: Vision-Based Block Pouring.}}
\label{tab:block-pouring-filtering}
\vspace{-0.1in}
\renewcommand{\arraystretch}{1.0}
\resizebox{1.0\linewidth}{!}{
\begin{tabular}{crcccc}
\toprule
\textbf{$\policyTask$}
& \textbf{Safety Filter}
& \textbf{Success} $\uparrow$
& \textbf{Failure} $\downarrow$
& \textbf{Safety Gain} $\uparrow$
& \textbf{Cond. Failure} $\downarrow$ \\
\midrule

\multirow{6}{*}{DreamerV3}
& No Filter
& 0.44 & 0.54
& -- & -- \\

& \nominal
& 0.45 & 0.54
& -0.12 $\pm$ 0.21
& 0.82 $\pm$ 0.38 \\

& \textit{WoN}
& 0.46 & 0.50
& 0.01 $\pm$ 0.15
& 0.50 $\pm$ 0.50 \\

& \textit{CVaR}$_{\text{0.1}}$
& 0.39 & 0.48
& -0.08 $\pm$ 0.16
& 0.48 $\pm$ 0.50 \\

& \textit{UnConf.}
& 0.00 & 0.43
& 0.08 $\pm$ 0.04
& 0.43 $\pm$ 0.50 \\

\rowcolor{gray!25}
& \textbf{\ours}
& 0.73 & 0.19
& 0.22 $\pm$ 0.21
& 0.20 $\pm$ 0.40 \\

\midrule

\multirow{6}{*}{\shortstack{Diffusion\\Policy}}
& No Filter
& 0.55 & 0.42
& -- & -- \\

& \nominal
& 0.56 & 0.40
& 0.12 $\pm$ 0.33
& 0.71 $\pm$ 0.46 \\

& \textit{WoN}\
& 0.46 & 0.48
& 0.03 $\pm$ 0.16
& 0.48 $\pm$ 0.50 \\

& \textit{CVaR}$_{\text{0.1}}$
& 0.13 & 0.19
& -0.02 $\pm$ 0.10
& 0.19 $\pm$ 0.39 \\

& \textit{UnConf.}
& 0.00 & 0.11
& 0.09 $\pm$ 0.03
& 0.11 $\pm$ 0.32 \\

\rowcolor{gray!25}
& \textbf{\ours}
& 0.63 & 0.24
& 0.18 $\pm$ 0.19
& 0.27 $\pm$ 0.44 \\

\bottomrule
\end{tabular}
}
\vspace{-0.1in}
\end{table}

\begin{figure}[t]
\centering
\includegraphics[width=1.0\linewidth]{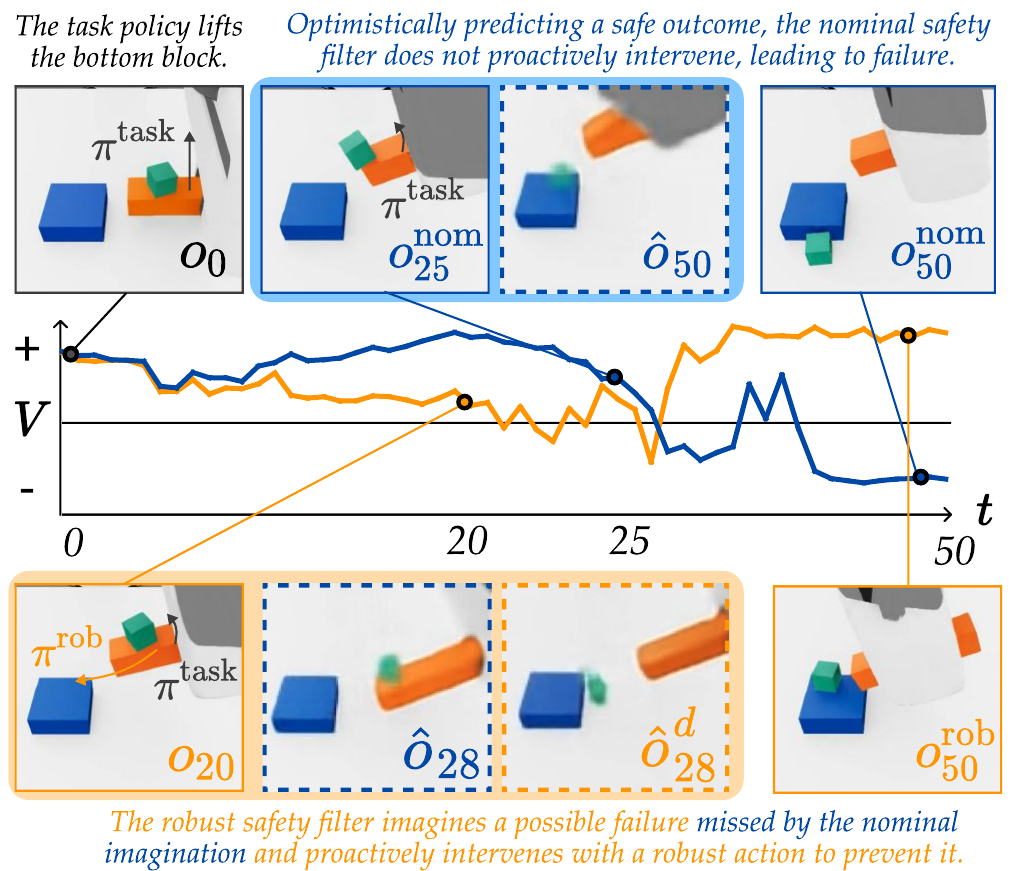}
\vspace{-0.25in}
\caption{\textbf{Safety Value Function: \textcolor{nominal_blue}{\nominal} vs. \textcolor{disturbance_orange}{\ours}.} Safeguarding the same $\policyTask$ from the same initial states, the robust safety filter proactively intervenes based on a plausibly pessimistic imagination (\textcolor{disturbance_orange}{$\hat{o}^{d}_{28}$}) in which $\policyTask$ can lead to failure, while the nominal imagination (\textcolor{nominal_blue}{$\hat{o}_{28}$}) predicts a non-failure outcome. The resulting robust action prevents failure. In contrast, the nominal latent safety filter overestimates safety and intervenes too late based on an optimistic imagination (\textcolor{nominal_blue}{$\hat{o}_{50}$}), leading to failure under an adverse realization.}
\label{fig:block-pouring-value}
\vspace{-0.2in}
\end{figure}

\para{Baselines} Following Sec.~\ref{sec:dubins}, we compare against latent safety filters trained using the same world model and dataset but accounting for dynamics uncertainty in different ways: \nominal, \textit{Worst-of-$N$}, and the risk-sensitive baseline (\textit{CVaR}). We also include \textit{Unconformalized latent disturbance (UnConf.)}, an ablation of \ours that uses an uncertainty set without the in-distribution constraint.

\para{Evaluation Metrics}
Since the ground-truth solutions are unavailable, we focus on the closed-loop performance of the policy. Success is completing a task without a timeout or failure. To further assess whether the learned filter is robust, we additionally measure: (i) realized safety gain (\textit{Safety Gain}): the increase in the learned safety value after each filter intervention, which measures whether the executed safety action robustly leads to a safer realized next state; and (ii) conditional failure rate (\textit{Cond. Failure}): the fraction of trajectories that fail among those in which the filter intervenes at least once.

\subsection{Can \ours Robustly Prevent Failures Under Uncertainty?}

\begin{table}[t]
\centering
\caption{\textbf{Results: Trajectory Replay Across $20$ Different Physics.}}
\label{tab:controlled-uncertainty}
\vspace{-0.1in}
\scriptsize
\renewcommand{\arraystretch}{1.0}
\resizebox{1.0\linewidth}{!}{
\begin{tabular}{rccccc}
\toprule
\textbf{Safety Filter}
& \textbf{Success} $\uparrow$
& \textbf{Failure} $\downarrow$
& \textbf{Safety Gain} $\uparrow$
& \textbf{Cond. Failure} $\downarrow$
& \textbf{Robust Rate} $\uparrow$ \\
\midrule
No Filter
& 0.33
& 0.47
& --
& --
& 0.10 \\

\nominal
& 0.31
& 0.40
& 0.17 $\pm$ 0.2
& 0.46 $\pm$ 0.5
& 0.11 \\

\textit{WoN}
& 0.35
& 0.31
& -0.04 $\pm$ 0.2
& 0.32 $\pm$ 0.4
& 0.07 \\

\textit{CVaR}$_{0.1}$
& 0.30
& 0.31
& -0.11 $\pm$ 0.2
& 0.32 $\pm$ 0.4
& 0.10 \\

\textit{UnConf.}
& 0.05
& 0.06
& 0.10 $\pm$ 0.03
& 0.06 $\pm$ 0.23
& 0.90 \\

\rowcolor{gray!25}
\ours
& 0.34
& 0.11
& 0.21 $\pm$ 0.15
& 0.11 $\pm$ 0.31
& 0.77 \\
\bottomrule
\end{tabular}
} \vspace{-0.05in}
\end{table}

\begin{figure}[t]
    \centering
    \includegraphics[width=1.0\linewidth]{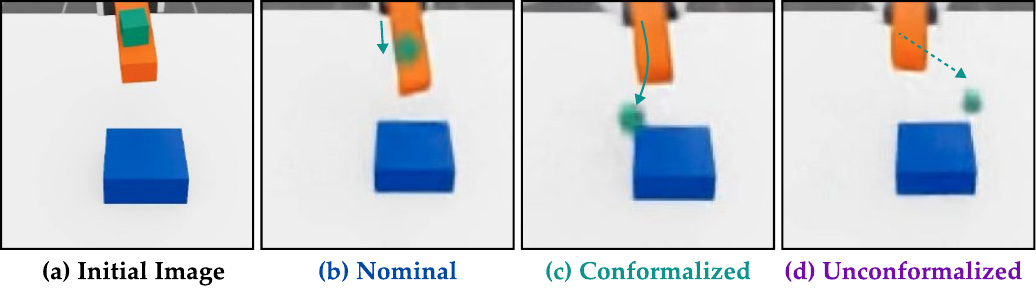}
    \vspace{-0.25in}
    \caption{\textbf{Block Pouring: World Model Imaginations.} (a) Initial state: the orange block is tilted while $\policyTask$ remains still. (b) The nominal imagination predicts a non-failure outcome. (c) With the conformalized uncertainty set, the latent disturbance imagines a plausible adverse transition in which the green block slides and falls. (d) Without the conformalized uncertainty set, the latent disturbance induces an implausible transition in which the green block falls in an infeasible manner, resulting in an overly pessimistic imagination.}\label{fig:sim-ambiguity}
    \vspace{-0.1in}
\end{figure}

\para{Result: \ours Robustly Safeguards Task Policies}
Table~\ref{tab:block-pouring-filtering} shows the closed-loop safety-filtering results under randomized physical parameters. \ours achieves lower failure rates while maintaining higher success rates, demonstrating robustness to dynamics uncertainty without becoming overly conservative. In contrast, latent disturbances without the in-distribution constraint (\textit{UnConf.}) become overly pessimistic, resulting in low success rates as its imagination trivially drifts toward implausible failure states, as shown in Fig.~\ref{fig:sim-ambiguity}.

Baselines that rely on sampling from the nominal latent dynamics (\nominal, \textit{WoN}, and \textit{CVaR}) are less effective, with a larger gap to \ours than in the Dubins' car experiments in Sec.~\ref{sec:dubins}. Compared to the Dubins' car experiments, the visual manipulation task involves higher-dimensional system uncertainty and more diverse plausible next states, making zeroth-order sampling less likely to capture meaningful worst-case outcomes within a limited sampling budget. Risk-sensitive distributional critics (\textit{CVaR}) are likewise either ineffective or overly conservative depending on the CVaR level, since extreme tails in high-dimensional latent spaces can include implausibly pessimistic states even under random sampling. In contrast, \ours explicitly optimizes against worst-case \emph{plausible} latent dynamics within the calibrated uncertainty set.

\para{Result: \ours Does Not Overestimate Safety} 
Fig.~\ref{fig:block-pouring-value} compares the nominal and robust safety values along an identical action sequence as a task policy. While the nominal safety value evaluates future safety under the average outcomes of actions and can therefore overestimate safety when adverse outcomes occur, the robust safety monitor evaluates safety under calibrated worst-case plausible outcomes and intervenes more preemptively when failure is possible, using robust actions whose plausible outcomes avoid failure. As shown in Table~\ref{tab:block-pouring-filtering}, \ours achieves a lower conditional failure rate than the nominal filter when filtering is activated, while intervening more frequently but only when necessary, remaining less pessimistic than \textit{UnConf.}, which doesn't use the in-distribution constraint for the uncertainty set construction. Moreover, \ours yields larger positive realized safety gains, indicating that its safety actions move the system toward safer realized next states, whereas the nominal filter exhibits smaller or even negative gains under dynamics uncertainty.


\subsection{Ablation: Robustness Under Controlled Uncertainty}

To further evaluate the robustness of safety filtering specifically to system uncertainty, independent of stochasticity in the task policy, we vary only the physics parameters of the system while replaying exactly the same task-policy action sequences.

\para{Setup: Replaying Safe Trajectories} We evaluate the robustness to system disturbances by replaying $100$ successful teleoperation trajectories from the same initial states and actions, each under $20$ different physics parameter settings. Here, we also report robust rate, defined as the fraction of replay trajectories that remain non-failure across all $20$ conditions.

\para{Result: \ours Robustly Prevents Failures Across Dynamics Variations}
Table~\ref{tab:controlled-uncertainty} shows that \ours keeps replaying $77$ of $100$ teleoperation trajectories failure-free across all $20$ physics settings, showing robustness to system disturbances. It also achieves the lowest overall and conditional failure rates among methods that retain nonzero task success; \textit{UnConf.} attains lower failure rates only by being overly conservative, resulting in a near-zero success rate. This indicates that when the filter intervenes, \ours proposes safety actions that remain effective across variations in system dynamics. Its positive realized safety gain further shows that these interventions consistently move the system toward safer realized states.

\subsection{Ablation: Robustness Under Partial Observability}

While we have so far considered uncertainty in future transitions arising from explicit system disturbances, partial observability can also induce uncertainty even when the underlying dynamics are deterministic and disturbance-free. For example, small differences in the underlying angle of the orange block or position of the green block may be difficult to distinguish from pixel observations, yet can lead to different future outcomes under the same action. We therefore ask how our latent-space robust optimization performs when uncertainty arises from the learned world model itself, due to partial observability or model approximation.

\begin{table}[t]
\centering
\caption{\textbf{Ablation: Deterministic Block-Pouring with Fixed Physics}}
\label{tab:deterministic-block-pouring}
\vspace{-0.05in}
\small
\renewcommand{\arraystretch}{1.0}
\resizebox{1.0\linewidth}{!}{
\begin{tabular}{crcccc}
\toprule
$\policyTask$
& \textbf{Safety Filter}
& \textbf{Success} $\uparrow$
& \textbf{Failure} $\downarrow$
& \textbf{Safety Gain} $\uparrow$
& \textbf{Cond. Failure} $\downarrow$ \\
\midrule

\multirow{5}{*}{DreamerV3}
& No Filter
& 0.43 & 0.56
& -- & -- \\

& \nominal
& 0.60 & 0.39
& 0.21 $\pm$ 0.24
& 0.40 $\pm$ 0.48 \\

& \textit{WoN}
& 0.34 & 0.54
& 0.10 $\pm$ 0.20
& 0.54 $\pm$ 0.49 \\

& \textit{CVaR}$_{0.1}$
& 0.00 & 0.34
& -0.05 $\pm$ 0.05
& 0.34 $\pm$ 0.47 \\

& \textit{UnConf.}
& 0.00 & 0.00
& 0.29 $\pm$ 0.08
& 0.00 $\pm$ 0.06 \\

\rowcolor{gray!25}
& \ours
& 0.81 & 0.18
& 0.27 $\pm$ 0.22
& 0.19 $\pm$ 0.39 \\

\midrule

\multirow{5}{*}{\shortstack{Diffusion\\Policy}}
& No Filter
& 0.47 & 0.51
& -- & -- \\

& \nominal
& 0.57 & 0.35
& 0.15 $\pm$ 0.16
& 0.25 $\pm$ 0.48 \\

& \textit{WoN}
& 0.28 & 0.66
& -0.00 $\pm$ 0.20
& 0.66 $\pm$ 0.47 \\

& \textit{CVaR}$_{0.1}$
& 0.00 & 0.19
& -0.05 $\pm$ 0.05
& 0.19 $\pm$ 0.39 \\

& \textit{UnConf.}
& 0.00 & 0.00
& 0.21 $\pm$ 0.06
& 0.00 $\pm$ 0.06 \\

\rowcolor{gray!25}
& \ours
& 0.66 & 0.25
& 0.16 $\pm$ 0.16
& 0.27 $\pm$ 0.44 \\

\bottomrule
\end{tabular}
} \vspace{-0.2in}
\end{table}

\para{Setup: Deterministic Block Pouring} We use the same block-pouring task, but fix the physical parameters to remove explicit disturbances from the underlying system dynamics. Thus, uncertainty arises from partial observability and approximation when mapping high-dimensional observations into the learned latent space. We collect approximately $3{,}000$ trajectories and train the world model, latent safety filter, and task policies following the same procedure. This setting isolates uncertainty in future transitions arising from partial observability and model approximation, allowing us to evaluate whether robust latent optimization remains beneficial even when the underlying system is deterministic.

\para{Result: Robust Optimization Improves Safety under Partial Observability}
Table~\ref{tab:deterministic-block-pouring} shows that robust latent-space optimization improves safety even when the underlying system dynamics are deterministic. In this setting, uncertainty in future transitions arises from partial observability and approximation error in the learned world model, rather than from explicit system disturbances. The nominal safety filter is more effective than in the stochastic environment, but \ours further reduces failures and increases success by mitigating uncertainty. These results suggest that robust latent-space optimization can improve decision-making under different sources of transition uncertainty, provided that the realized test-time transition is contained within the calibrated uncertainty set.

\section{Hardware: Robust Steering During Contact-Rich Visuomotor Policy Execution}\label{sec:egg}

Finally, we scale \ours to a real-world visual manipulation task using a Franka Research 3. The robot must serve a sunny-side-up fried egg from a spatula onto a plate without dropping it on the table or flipping it sunny-side down. The egg's dynamics are uncertain due to unobserved physical properties such as friction between the egg and spatula.

\subsection{Experimental Setup: Serving Sunny-Side-Up Fried Eggs}

\para{Setup} Observations consist of two tabletop $3\times192\times256$ RGB camera images and 7-D joint-position proprioception. Actions are 6-D end-effector delta poses with a discrete gripper command, executed at $15$ Hz. 

\para{Disturbance} We introduce an explicit source of unobserved disturbance by varying the spatula surface to mimic changes in friction caused by oil and surface conditions in household kitchens. We use three surface conditions: (i) no tape, (ii) scotch tape, and (iii) heavy-duty tape (see Fig.~\ref{fig:egg-replay}).

\para{World Model} We employ an RSSM with Gaussian latent dynamics~\cite{hafner2019learning} learned from frozen DINOv3~\cite{simeoni2025dinov3} embeddings, and construct the latent dynamics model using a Transformer with an $8$-step history to predict the next latent-state distribution. We additionally train the model with a DINO token reconstruction objective alongside pixel reconstruction (see Appendix~\ref{sec:appendix-experiments} for details). We collect $900$ teleoperated trajectories with frame-level failure annotations, which are used to learn a latent failure margin over the learned latent space.

\para{Latent Safety Filter} We follow \cite{seo2025uncertainty} to learn the latent safety filter. The latent disturbance predicts additive perturbations to the mean and covariance of the Gaussian latent dynamics, and the uncertainty set is calibrated using $50$ held-out trajectories.

\para{Task Policies} We consider three task policies that we seek to steer at runtime: (i) a Diffusion Policy~\cite{chi2023diffusionpolicy} trained from scratch using $300$ successful trajectories; (ii) a vision-language-action (VLA) policy fine-tuned from $\pi_{0.5}$~\cite{intelligence2025pi} using the same trajectories; and (iii) a teleoperation from a human.

\begin{figure}[t]
\centering
\includegraphics[width=1.0\linewidth]{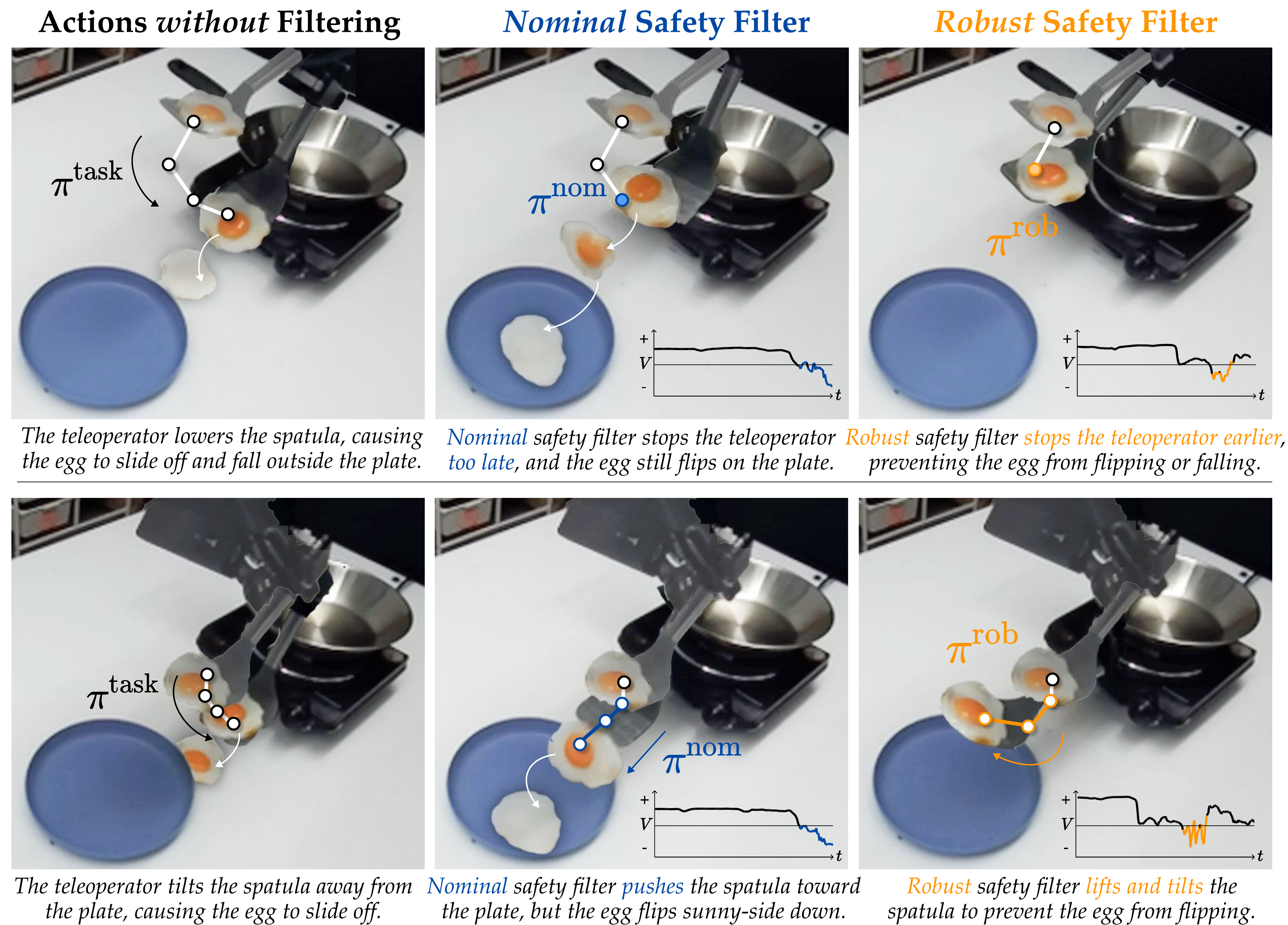}\vspace{-0.1in}
\caption{\textbf{Qualitative Results: Preventing Sunny-Side-Down Egg.} \textcolor{disturbance_orange}{\ours} preemptively identifies teleoperator actions that may cause the egg to fall or flip and intervenes with robust actions that remain safe under worst-case dynamics. In contrast, the \textcolor{nominal_blue}{\nominal} safety filter relies on overly optimistic safety values and non-robust actions, leading to sunny-side-down outcomes.}
\label{fig:egg_qualitative}
\vspace{-0.05in}
\end{figure}

\begin{figure}[t]
\centering
\includegraphics[width=1.0\linewidth]{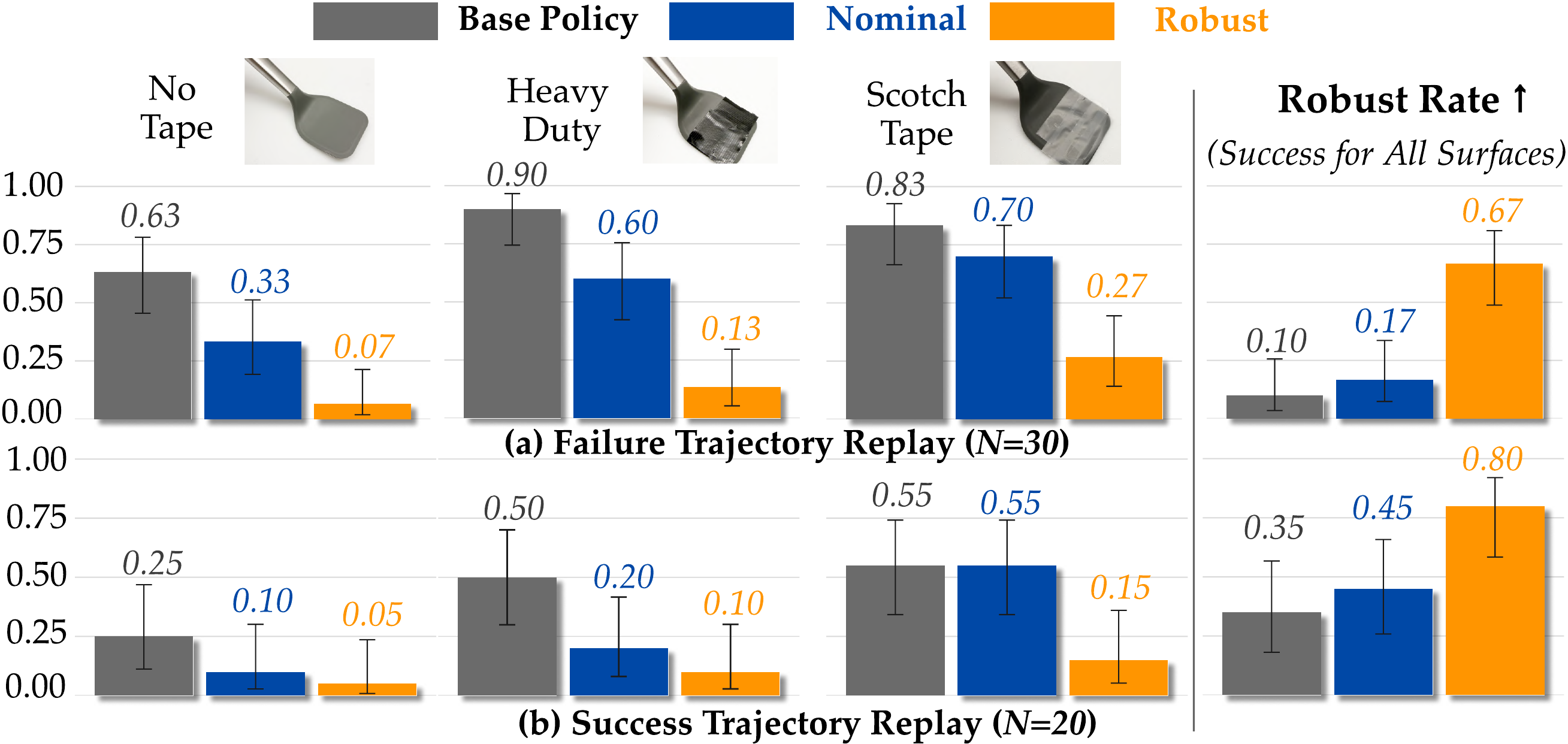}
\vspace{-0.3in}
\caption{\textbf{Filtering the Same $\policyTask$ Across Different Surfaces.} (\textit{Left}) Failure rates under each spatula surface condition. (\textit{Right}) Robust rate, measuring the fraction of trajectories that remain safe across all surface conditions. \textcolor{disturbance_orange}{\ours} consistently minimizes failures and achieves a higher robust rate, whereas the \textcolor{nominal_blue}{\nominal} safety filter prevents failures only under certain surface conditions.}
\label{fig:egg-replay}
\vspace{-0.1in}
\end{figure}

\subsection{Robust Latent Safety Filtering}

We compare a \nominal latent safety filter, trained using the expected safety value \eqref{eq:trad-fixed-pt-bellman}, against \ours trained with \eqref{eq:robust-fixed-pt-bellman}.

\para{Result: Robustness Across Surface Disturbances} We first evaluate robustness to disturbances induced by the spatula surface condition. We replay $20$ success and $30$ failure trajectories under all three surface conditions while safeguarding them. Fig.~\ref{fig:egg-replay} shows that \ours consistently achieves lower failure rates and a higher robust rate, defined as the fraction of replay trajectories that remain safe across all surface conditions. In contrast, the \nominal safety filter is effective only under certain surface conditions and achieves a lower robust rate.

\para{Result: Robustly Safeguarding $\policyTask$} We roll out Diffusion Policy~\cite{chi2023diffusionpolicy} and $\pi_{0.5}$~\cite{intelligence2025pi} for $20$ trials each, safeguarded by either the \nominal safety filter or \ours, with Scotch tape applied to the spatula. Fig.~\ref{fig:egg_quantitative} shows that \ours achieves a larger improvement in success rate than the \textit{Nominal} filter, demonstrating more robustness against uncertain dynamics.

Fig.~\ref{fig:egg_qualitative} shows qualitative results safeguarding the teleoperator. \ours preemptively overrides unsafe actions, whereas the \nominal safety filter may fail to intervene in time due to overly optimistic imaginations. Moreover, \ours intervenes with robust fallback actions whose plausible outcomes remain safe, while the \nominal filter may select actions that appear safe under nominal dynamics but fail under adverse realizations.

\begin{figure}[t]
\centering
\includegraphics[width=1.0\linewidth]{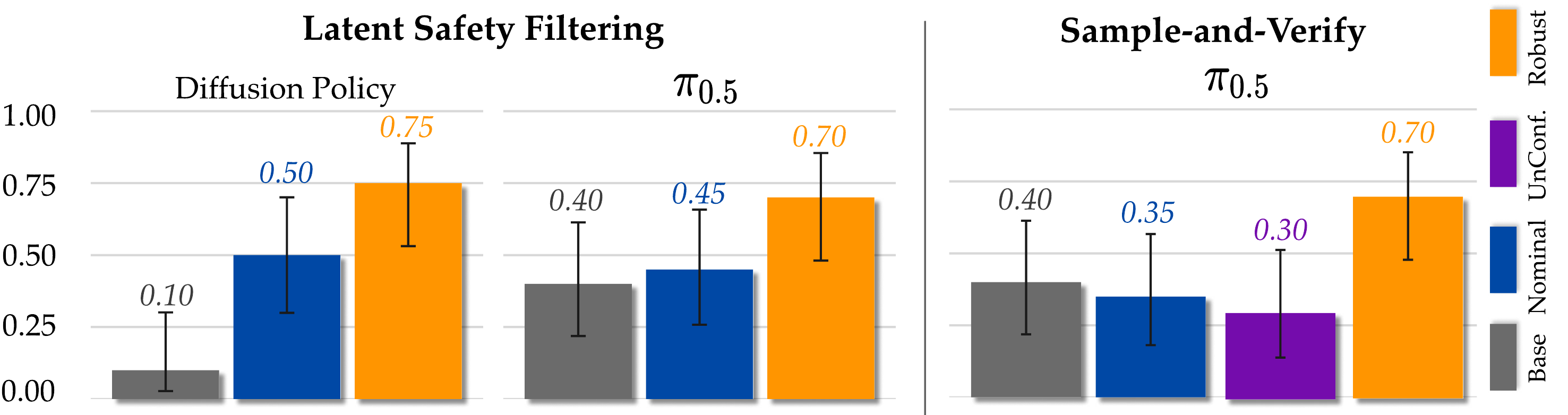} \vspace{-0.3in}
\caption{\textbf{Quantitative Results: Success Rates.} (\textit{Left}) Safety filtering over $20$ rollouts for Diffusion Policy and $\pi_{0.5}$. (\textit{Right}) Steering $\pi_{0.5}$ with sample-and-verify action selection, using world-model imaginations to evaluate candidate action chunks. In both settings, \ours robustly improves success rates.}
\label{fig:egg_quantitative}
\vspace{-0.2in}
\end{figure}

\subsection{Robust Sample-and-Verify with Latent Disturbances}

We test the sampling-and-verify policy steering using WM imagination \eqref{eq:robust-sampling-based-steering} on fine-tuned $\pi_{0.5}$~\cite{intelligence2025pi}, evaluating whether the learned latent disturbance enables robust action selection.

\para{Setup} We sample $8$ candidate action chunks from the policy, each containing $16$ future actions. We evaluate the outcome of each candidate using WM imaginations and execute the action chunk with the best predicted outcome. We compare three imagination schemes: (i) \nominal, using the learned latent dynamics; (ii) \ours, using the learned latent disturbance; and (iii) \textit{Unconformalized (UnConf.)}, using a latent disturbance without the in-distribution constraint, as illustrated in Fig.~\ref{fig:egg_steering}.

\para{Objective Function ($\objective$)} To evaluate action outcomes, we use a cost function that combines the learned failure margin $\ellz(\latent)$, evaluated on predicted latent states, with VLM-based task-success and progress scores from \textsc{Robometer}~\cite{liang2026robometer}, evaluated on decoded images. The VLM reward is averaged over each action chunk to provide a dense signal.

\para{Result: \ours Enables Robust Policy Steering}
Fig.~\ref{fig:egg_quantitative} reports sample-and-verify style policy steering results over $20$ trials for each method, showing that \ours improves policy success. Fig.~\ref{fig:egg_steering} illustrates that \ours selects actions that remain favorable even under plausible worst-case futures, enabling robust policy steering. In contrast, \nominal does not improve over the base policy without steering, imagining optimistic futures for risky actions. Latent disturbance using an unconformalized uncertainty set instead trivially imagines implausible failures even for safe actions, which weakens discrimination among candidate actions for policy steering.

\begin{figure}[t]
\centering
\includegraphics[width=1.0\linewidth]{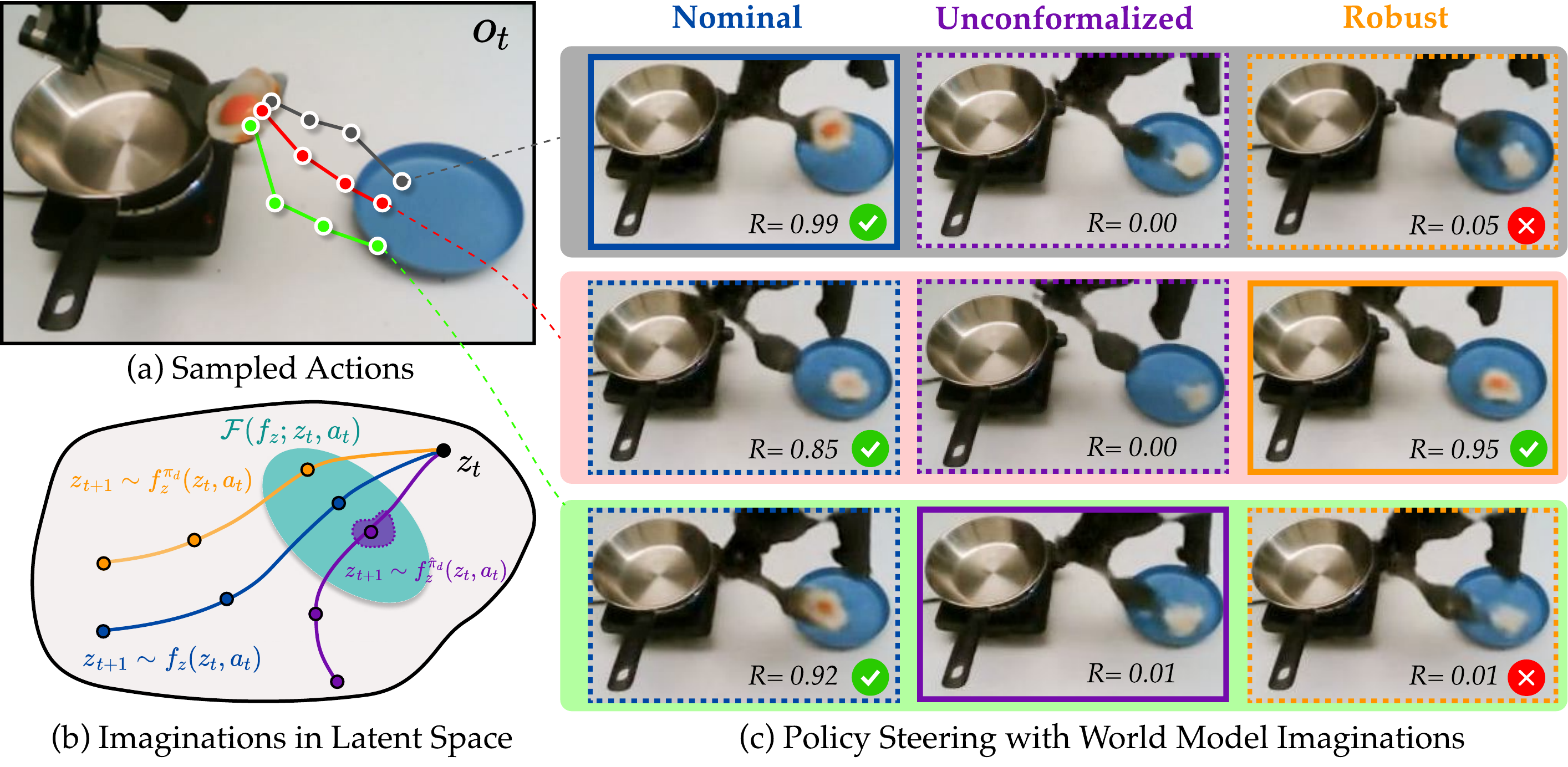} \vspace{-0.2in}
\caption{\textbf{Qualitative: Policy Steering with World Model.} (a) Initial observation and sampled candidate actions. (b) For each candidate, \textcolor{nominal_blue}{\nominal} imagines outcomes using \textcolor{nominal_blue}{$\dynz$}, \textcolor{disturbance_orange}{\ours} computes worst-case \textcolor{disturbance_orange}{$\dynz^{\latentDisturbance}$} within the \textcolor{teal}{calibrated uncertainty set $\ambiguity$}, while \textit{\textcolor{violet}{Unconformalized}} \textcolor{violet}{$\dynz^{\hat{\latentDisturbance}}$} can exploit out-of-distribution latent states. (c) Imagined outcomes for candidate actions. \ours selects an action that remains successful under the plausible worst-case, whereas the other schemes lose discriminability among candidate actions.}
\label{fig:egg_steering}
\vspace{-0.1in}
\end{figure}

\section{Limitations}

\para{Balancing Plausibility}
We optimize our latent disturbance to generate adverse outcomes while regularizing it to stay close to the learned dynamics and remain in distribution. However, these constraints provide only \textit{proxies} for plausibility, so the resulting uncertainty set can still become overly conservative or fail to capture plausible but critical outcomes. Moreover, because the constraints are enforced through soft penalties, the optimized disturbance may still produce implausible imaginations, particularly when the underlying world model is inaccurate~\cite{seo2025uncertainty}. Conformal calibration provides a marginal coverage guarantee over the calibration distribution rather than an input-conditional guarantee~\cite{angelopoulos2023conformal}. As disturbance optimization steers WM imaginations away from this distribution, the coverage guarantee may further weaken under distribution shift. More structured world models~\cite{wang2026temporal}, adaptive or input-dependent calibration schemes, and hallucination detection for world-model imaginations~\cite{hansen2026hallucination} could help ensure that pessimistic imaginations remain physically meaningful.

\para{World Model Fidelity}
Our current implementation trains task-specific latent WMs from offline interaction data, requiring substantial data collection, and latent disturbance optimization cannot recover outcomes that the WM has not learned to represent. This underscores the need for diverse robot interaction data, including failures, rare outcomes, and variations in physical conditions~\cite{balaji2026oopsieverse}. In addition, our current formulation assumes a probabilistic latent WM with an explicit transition distribution, which is used to characterize predictive uncertainty and define the latent disturbance. It is therefore not directly applicable to deterministic WMs or diffusion-based video WMs, where uncertainty is represented through a different generative process; complementary approaches can instead steer diffusion-based WM imaginations toward adverse yet plausible outcomes~\cite{seo2026stressdream}. Our framework also assumes a meaningful failure margin or objective function; inaccurate specifications can misdirect disturbance optimization and policy steering. More general task and safety specifications, for example those derived from vision-language models, could reduce this dependence on task-specific objectives~\cite{liang2026robometer}.

\para{Game-Theoretic Optimization}
Solving the resulting dynamic game remains challenging in high-dimensional action (e.g., $7$-D) and latent disturbance (e.g., $1$K-D) spaces. Although our disturbance parameterization and adversarial optimization make the problem more tractable, the learned disturbance explores only a restricted class of latent dynamics, and the optimization does not guarantee convergence to a global saddle point~\cite{wang2024magics}. An insufficiently optimized disturbance may miss adverse outcomes, while an insufficiently optimized robot policy may produce suboptimal robust actions. Developing more reliable game-theoretic optimization methods and characterizing the effects of function approximation and optimization error remain important future work.

\section{Conclusion}
In this work, we propose a framework for robust decision-making in the learned latent space of world models. We model latent-space disturbances as adverse yet plausible perturbations to learned latent dynamics, enabling robust optimization directly in latent spaces learned from high-dimensional observations without requiring disturbances to be explicitly represented in a physically interpretable state space. To characterize which latent dynamics constitute plausible disturbances, we construct an uncertainty set using a dynamics-aware similarity metric that captures predictive uncertainty in the learned dynamics, together with out-of-distribution detection to exclude implausible latent states. We calibrate this uncertainty set using conformal prediction, allowing the robot to reason about worst-case outcomes while preventing the resulting WM imaginations from becoming overly pessimistic. We instantiate this framework for two forms of robust policy steering: latent safety filtering and sample-and-verify steering of a generative control policy. Controlled experiments show that the proposed formulation closely approximates ground-truth robust safety values and reduces failures under unobserved system disturbances, while ablations demonstrate the importance of calibration and in-distribution constraints for preserving the plausibility of imaginations induced by latent disturbances. In simulated and real-robot vision-based manipulation, our framework improves policy success rates through robust runtime steering.

\ifanonymous\else
\section*{Acknowledgment}
This work was supported in part by Toyota Research Institute and the DARPA Young Faculty Award (YFA). This article solely reflects the opinions and conclusions of its authors, and not TRI (nor any other Toyota entity), nor DARPA, nor the U.S. Government.
\fi

\addtolength{\textheight}{0cm}

\bibliographystyle{IEEEtran}
\bibliography{mybib}

@article{mitchell2005time,
  title={A time-dependent Hamilton-Jacobi formulation of reachable sets for continuous dynamic games},
  author={Mitchell, Ian M and Bayen, Alexandre M and Tomlin, Claire J},
  journal={IEEE Transactions on Automatic Control},
  year={2005}
}

@inproceedings{bansal2021deepreach,
  title={Deepreach: A deep learning approach to high-dimensional reachability},
  author={Bansal, Somil and Tomlin, Claire J},
  booktitle={IEEE International Conference on Robotics and Automation (ICRA)},
  pages={1817--1824},
  year={2021}
}

@inproceedings{bansal2017hamilton,
  title={Hamilton-jacobi reachability: A brief overview and recent advances},
  author={Bansal, Somil and Chen, Mo and Herbert, Sylvia and Tomlin, Claire J},
  booktitle={IEEE Conference on Decision and Control (CDC)},
  pages={2242--2253},
  year={2017}
}

@article{hsu2023safety,
  title={The safety filter: A unified view of safety-critical control in autonomous systems},
  author={Hsu, Kai-Chieh and Hu, Haimin and Fisac, Jaime F},
  journal={Annual Review of Control, Robotics, and Autonomous Systems},
  volume={7},
  year={2023}
}

@article{wabersich2023data,
  title={Data-driven safety filters: Hamilton-jacobi reachability, control barrier functions, and predictive methods for uncertain systems},
  author={Wabersich, Kim P and Taylor, Andrew J and Choi, Jason J and Sreenath, Koushil and Tomlin, Claire J and Ames, Aaron D and Zeilinger, Melanie N},
  journal={IEEE Control Systems Magazine},
  volume={43},
  number={5},
  pages={137--177},
  year={2023}
}

@inproceedings{bastani2021safe,
  title={Safe reinforcement learning with nonlinear dynamics via model predictive shielding},
  author={Bastani, Osbert},
  booktitle={American Control Conference (ACC)},
  pages={3488--3494},
  year={2021}
}

@inproceedings{li2020robust,
  title={Robust model predictive shielding for safe reinforcement learning with stochastic dynamics},
  author={Li, Shuo and Bastani, Osbert},
  booktitle={IEEE International Conference on Robotics and Automation (ICRA)},
  pages={7166--7172},
  year={2020}
}

@article{mitchell2007toolbox,
  title={A toolbox of level set methods},
  author={Mitchell, Ian M and others},
  journal={UBC Department of Computer Science Technical Report TR-2007-11},
  year={2007}
}

@inproceedings{fisac2019bridging,
  title={Bridging hamilton-jacobi safety analysis and reinforcement learning},
  author={Fisac, Jaime F and Lugovoy, Neil F and Rubies-Royo, Vicen{\c{c}} and Ghosh, Shromona and Tomlin, Claire J},
  booktitle={IEEE International Conference on Robotics and Automation (ICRA)},
  pages={8550--8556},
  year={2019}
}

@inproceedings{AkametaluGP,
  author={Akametalu, Anayo K. and Fisac, Jaime F. and Gillula, Jeremy H. and Kaynama, Shahab and Zeilinger, Melanie N. and Tomlin, Claire J.},
  booktitle={IEEE Conference on Decision and Control (CDC)}, 
  title={Reachability-based safe learning with Gaussian processes}, 
  year={2014}}

@inproceedings{agrawal2026anysafe,
  author    = {Agrawal, Sankalp and Seo, Junwon and Nakamura, Kensuke and Tian, Ran and Bajcsy, Andrea},
  title     = {AnySafe: Adapting Latent Safety Filters at Runtime via Safety Constraint Parameterization in the Latent Space},
  booktitle = {IEEE International Conference on Robotics and Automation (ICRA)},
  year      = {2026}
}

@inproceedings{nakamura2026howto,
  title = 	 {How to Train Your Latent Control Barrier Function: Smooth Safety Filtering Under Hard-to-Model Constraints},
  author =       {Nakamura, Kensuke and Bishop, Arun L and Man, Steven and Johnson, Aaron M. and Manchester, Zachary and Bajcsy, Andrea},
  booktitle = 	 {Learning for Dynamics and Control Conference (L4DC)},
  year = 	 {2026},
}

@article{wu2025foresight,
  title={From foresight to forethought: Vlm-in-the-loop policy steering via latent alignment},
  author={Wu, Yilin and Tian, Ran and Swamy, Gokul and Bajcsy, Andrea},
  journal={Robotics: Science and Systems (RSS)},
  year={2025}
}

@article{yuan2026ups,
    title={When to Act, Ask, or Learn: Uncertainty-Aware Policy Steering},
    author={Yuan, Jessie and Wu, Yilin and Bajcsy, Andrea},
    journal={Robotics: Science and Systems (RSS)},
    year={2026}
}

@inproceedings{Majumdar2017HowSA,
  title={How Should a Robot Assess Risk? Towards an Axiomatic Theory of Risk in Robotics},
  author={Anirudha Majumdar and Marco Pavone},
  booktitle={International Symposium of Robotics Research (ISRR)},
  year={2017}
}

@article{xu2025can,
  title={Can We Detect Failures Without Failure Data? Uncertainty-Aware Runtime Failure Detection for Imitation Learning Policies},
  author={Xu, Chen and Nguyen, Tony Khuong and Dixon, Emma and Rodriguez, Christopher and Miller, Patrick and Lee, Robert and Shah, Paarth and Ambrus, Rares and Nishimura, Haruki and Itkina, Masha},
  journal={Robotics: Science and Systems (RSS)},
  year={2025}
}

@article{shafer2008tutorial,
  title={A tutorial on conformal prediction.},
  author={Shafer, Glenn and Vovk, Vladimir},
  journal={Journal of Machine Learning Research},
  volume={9},
  number={3},
  year={2008}
}

@article{angelopoulos2023conformal,
  title={Conformal prediction: A gentle introduction},
  author={Angelopoulos, Anastasios N and Bates, Stephen and others},
  journal={Foundations and Trends{\textregistered} in Machine Learning},
  year={2023},
  publisher={Now Publishers, Inc.}
}

@article{nakamura2025generalizing,
  title={Generalizing Safety Beyond Collision-Avoidance via Latent-Space Reachability Analysis},
  author={Nakamura, Kensuke and Peters, Lasse and Bajcsy, Andrea},
  journal={Robotics: Science and Systems (RSS)},
  year={2025}
}

@article{seo2025uncertainty,
    title={Uncertainty-aware Latent Safety Filters for Avoiding Out-of-Distribution Failures},
    author={Seo, Junwon and Nakamura, Kensuke and Bajcsy, Andrea},
    journal={Conference on Robot Learning (CoRL)},
    year={2025}
  }

@article{kim2025well,
  title={How Well Do Latent World Models Understand Partially Observable Safety Constraints?},
  author={Kim, Matthew and Nakamura, Kensuke and Bajcsy, Andrea},
  journal={Conference on Robot Learning (CoRL)},
  year={2026}
}

@inproceedings{hsu2023isaacs,
  title={Isaacs: Iterative soft adversarial actor-critic for safety},
  author={Hsu, Kai-Chieh and Nguyen, Duy Phuong and Fisac, Jaime Fernandez},
  booktitle={Learning for Dynamics and Control Conference (L4DC)},
  pages={90--103},
  year={2023}
}

@inproceedings{nguyen2024gameplay,
      title={Gameplay Filters: Robust Zero-Shot Safety through Adversarial Imagination},
      author={Duy P. Nguyen and Kai-Chieh Hsu and Wenhao Yu and Jie Tan and Jaime Fernández Fisac},
      booktitle={Conference on Robot Learning (CoRL)},
      year={2024}
}

@article{wang2024magics,
  title={Magics: Adversarial rl with minimax actors guided by implicit critic stackelberg for convergent neural synthesis of robot safety},
  author={Wang, Justin and Hu, Haimin and Nguyen, Duy Phuong and Fisac, Jaime Fern{\'a}ndez},
  journal={arXiv preprint arXiv:2409.13867},
  year={2024}
}

@article{mittal2025isaac,
  title={Isaac Lab: A GPU-Accelerated Simulation Framework for Multi-Modal Robot Learning},
  author={Mittal, Mayank and Roth, Pascal and Tigue, James and Richard, Antoine and Zhang, Octi and Du, Peter and Serrano-Mu{\~n}oz, Antonio and Yao, Xinjie and Zurbr{\"u}gg, Ren{\'e} and Rudin, Nikita and others},
  journal={arXiv preprint arXiv:2511.04831},
  year={2025}
}

@article{simeoni2025dinov3,
  title={Dinov3},
  author={Sim{\'e}oni, Oriane and Vo, Huy V and Seitzer, Maximilian and Baldassarre, Federico and Oquab, Maxime and Jose, Cijo and Khalidov, Vasil and Szafraniec, Marc and Yi, Seungeun and Ramamonjisoa, Micha{\"e}l and others},
  journal={arXiv preprint arXiv:2508.10104},
  year={2025}
}

@inproceedings{hafner2019learning,
  title={Learning latent dynamics for planning from pixels},
  author={Hafner, Danijar and Lillicrap, Timothy and Fischer, Ian and Villegas, Ruben and Ha, David and Lee, Honglak and Davidson, James},
  booktitle={International Conference on machine learning (ICML)},
  year={2019}
}

@inproceedings{
    Hafner2020Dream,
    title={Dream to Control: Learning Behaviors by Latent Imagination},
    author={Danijar Hafner and Timothy Lillicrap and Jimmy Ba and Mohammad Norouzi},
    booktitle={International Conference on Learning Representations (ICLR)},
    year={2020}
}

@article{hafner2023dreamerv3,
	title = {Mastering diverse control tasks through world models},
	volume = {640},
	issn = {1476-4687},
	number = {8059},
	journal = {Nature},
	author = {Hafner, Danijar and Pasukonis, Jurgis and Ba, Jimmy and Lillicrap, Timothy},
	month = apr,
	year = {2025},
	pages = {647--653},
}

@inproceedings{hansen2024tdmpc2,
	title={TD-MPC2: Scalable, Robust World Models for Continuous Control}, 
	author={Nicklas Hansen and Hao Su and Xiaolong Wang},
	booktitle={International Conference on Learning Representations (ICLR)},
	year={2024}
}

@article{zhou2024dino,
  title={Dino-wm: World models on pre-trained visual features enable zero-shot planning},
  author={Zhou, Gaoyue and Pan, Hengkai and LeCun, Yann and Pinto, Lerrel},
  journal={International Conference on machine learning (ICML)},
  year={2025}
}

@article{ding2023class,
  title={Class-conditional conformal prediction with many classes},
  author={Ding, Tiffany and Angelopoulos, Anastasios and Bates, Stephen and Jordan, Michael and Tibshirani, Ryan J},
  journal={Advances in Neural Information Processing Systems (NeurIPS)},
  volume={36},
  pages={64555--64576},
  year={2023}
}

@inproceedings{
    ren2023robots,
    title={Robots That Ask For Help: Uncertainty Alignment for Large Language Model Planners},
    author={Allen Z. Ren and Anushri Dixit and Alexandra Bodrova and Sumeet Singh and Stephen Tu and Noah Brown and Peng Xu and Leila Takayama and Fei Xia and Jake Varley and Zhenjia Xu and Dorsa Sadigh and Andy Zeng and Anirudha Majumdar},
    booktitle={Conference on Robot Learning (CoRL)},
    year={2023}
}

@inproceedings{liang2026robometer,
  title     = {Robometer: Scaling General-Purpose Robotic Reward Models via Trajectory Comparisons},
  author={Anthony Liang and Yigit Korkmaz and Jiahui Zhang and Minyoung Hwang and Abrar Anwar and Sidhant Kaushik and Aditya Shah and Alex S. Huang and Luke Zettlemoyer and Dieter Fox and Yu Xiang and Anqi Li and Andreea Bobu and Abhishek Gupta and Stephen Tu and Erdem Biyik and Jesse Zhang},
  year={2026},
  booktitle={Robotics: Science and Systems (RSS)},
}

@inproceedings{lutkus2025latent,
  title={Latent representations for control design with provable stability and safety guarantees},
  author={Lutkus, Paul and Wang, Kaiyuan and Lindemann, Lars and Tu, Stephen},
  booktitle={IEEE Conference on Decision and Control (CDC)},
  pages={2937--2944},
  year={2025},
  organization={IEEE}
}

@inproceedings{
    wang2026temporal,
    title={Temporal Straightening for Latent Planning},
    author={Ying Wang and Oumayma Bounou and Gaoyue Zhou and Randall Balestriero and Tim G. J. Rudner and Yann LeCun and Mengye Ren},
    booktitle={International Conference on Machine Learning (ICML)},
    year={2026}
}

@book{ben2009robust,
  title={Robust optimization},
  author={Ben-Tal, Aharon and Nemirovski, Arkadi and El Ghaoui, Laurent},
  year={2009},
  publisher={Princeton university press}
}

@article{mayne2005robust,
  title={Robust model predictive control of constrained linear systems with bounded disturbances},
  author={Mayne, David Q and Seron, Mar{\'\i}a M and Rakovi{\'c}, Sa{\v{s}}a V},
  journal={Automatica},
  volume={41},
  number={2},
  pages={219--224},
  year={2005},
  publisher={Elsevier}
}

@article{fisac2018general,
  title={A general safety framework for learning-based control in uncertain robotic systems},
  author={Fisac, Jaime F and Akametalu, Anayo K and Zeilinger, Melanie N and Kaynama, Shahab and Gillula, Jeremy and Tomlin, Claire J},
  journal={IEEE Transactions on Automatic Control},
  volume={64},
  number={7},
  pages={2737--2752},
  year={2018},
  publisher={IEEE}
}

@article{chapman2021risk,
  title={Risk-sensitive safety analysis using conditional value-at-risk},
  author={Chapman, Margaret P and Bonalli, Riccardo and Smith, Kevin M and Yang, Insoon and Pavone, Marco and Tomlin, Claire J},
  journal={IEEE Transactions on Automatic Control},
  volume={67},
  number={12},
  pages={6521--6536},
  year={2021},
  publisher={IEEE}
}

@inproceedings{choi2021robust,
  title={Robust control barrier--value functions for safety-critical control},
  author={Choi, Jason J and Lee, Donggun and Sreenath, Koushil and Tomlin, Claire J and Herbert, Sylvia L},
  booktitle={IEEE Conference on Decision and Control (CDC)},
  pages={6814--6821},
  year={2021}
}

@article{oh2026synthesis,
  title={Synthesis and Deployment of Maximal Robust Control Barrier Functions through Adversarial Reinforcement Learning},
  author={Oh, Donggeon David and Nguyen, Duy P and Hu, Haimin and Fisac, Jaime Fern{\'a}ndez},
  journal={arXiv preprint arXiv:2604.13192},
  year={2026}
}

@misc{cosner2023robustsafetystochasticuncertainty,
      title={Robust Safety under Stochastic Uncertainty with Discrete-Time Control Barrier Functions}, 
      author={Ryan K. Cosner and Preston Culbertson and Andrew J. Taylor and Aaron D. Ames},
      journal={Proceedings of Robotics: Science and Systems (RSS)},
      year={2023},
}

@article{yang2018dynamic,
  title={A dynamic game approach to distributionally robust safety specifications for stochastic systems},
  author={Yang, Insoon},
  journal={Automatica},
  volume={94},
  pages={94--101},
  year={2018},
  publisher={Elsevier}
}

@article{ren2022distributionally,
  title={Distributionally robust policy learning via adversarial environment generation},
  author={Ren, Allen Z and Majumdar, Anirudha},
  journal={IEEE Robotics and Automation Letters},
  volume={7},
  number={2},
  pages={1379--1386},
  year={2022},
  publisher={IEEE}
}

@inproceedings{pinto2017robust,
  title={Robust adversarial reinforcement learning},
  author={Pinto, Lerrel and Davidson, James and Sukthankar, Rahul and Gupta, Abhinav},
  booktitle={International Conference on Machine Learning (ICML)},
  pages={2817--2826},
  year={2017}
}

@inproceedings{curi2021combining,
  title={Combining pessimism with optimism for robust and efficient model-based deep reinforcement learning},
  author={Curi, Sebastian and Bogunovic, Ilija and Krause, Andreas},
  booktitle={International Conference on Machine Learning (ICML)},
  pages={2254--2264},
  year={2021},
  organization={PMLR}
}

@article{blanchet2023double,
  title={Double pessimism is provably efficient for distributionally robust offline reinforcement learning: Generic algorithm and robust partial coverage},
  author={Blanchet, Jose and Lu, Miao and Zhang, Tong and Zhong, Han},
  journal={Advances in Neural Information Processing Systems (NeurIPS)},
  volume={36},
  pages={66845--66859},
  year={2023}
}

@article{liang2022efficient,
  title={Efficient adversarial training without attacking: Worst-case-aware robust reinforcement learning},
  author={Liang, Yongyuan and Sun, Yanchao and Zheng, Ruijie and Huang, Furong},
  journal={Advances in Neural Information Processing Systems (NeurIPS)},
  volume={35},
  pages={22547--22561},
  year={2022}
}

@article{rigter2023one,
  title={One risk to rule them all: A risk-sensitive perspective on model-based offline reinforcement learning},
  author={Rigter, Marc and Lacerda, Bruno and Hawes, Nick},
  journal={Advances in Neural Information Processing Systems (NeurIPS)},
  volume={36},
  pages={77520--77545},
  year={2023}
}

@inproceedings{
    chen2026policydriven,
    title={Policy-Driven World Model Adaptation for Robust Offline Model-based Reinforcement Learning},
    author={Jiayu Chen and Le Xu and Aravind Venugopal and Jeff Schneider},
    booktitle={International Conference on Machine Learning (ICML)},
    year={2026}
}

@article{rigter2022rambo,
  title={Rambo-rl: Robust adversarial model-based offline reinforcement learning},
  author={Rigter, Marc and Lacerda, Bruno and Hawes, Nick},
  journal={Advances in Neural Information Processing Systems (NeurIPS)},
  volume={35},
  pages={16082--16097},
  year={2022}
}

@article{moos2022robust,
  title={Robust reinforcement learning: A review of foundations and recent advances},
  author={Moos, Janosch and Hansel, Kay and Abdulsamad, Hany and Stark, Svenja and Clever, Debora and Peters, Jan},
  journal={Machine Learning and Knowledge Extraction},
  volume={4},
  number={1},
  pages={276--315},
  year={2022},
  publisher={MDPI}
}

@inproceedings{mandlekar2017adversarially,
  title={Adversarially robust policy learning: Active construction of physically-plausible perturbations},
  author={Mandlekar, Ajay and Zhu, Yuke and Garg, Animesh and Fei-Fei, Li and Savarese, Silvio},
  booktitle={IEEE/RSJ International Conference on Intelligent Robots and Systems (IROS)},
  pages={3932--3939},
  year={2017}
}

@inproceedings{
    sinha2018certifiable,
    title={Certifiable Distributional Robustness with Principled Adversarial Training},
    author={Aman Sinha and Hongseok Namkoong and John Duchi},
    booktitle={International Conference on Learning Representations (ICLR)},
    year={2018},
}

@article{zhang2020robust,
  title={Robust deep reinforcement learning against adversarial perturbations on state observations},
  author={Zhang, Huan and Chen, Hongge and Xiao, Chaowei and Li, Bo and Liu, Mingyan and Boning, Duane and Hsieh, Cho-Jui},
  journal={Advances in Neural Information Processing Systems (NeurIPS)},
  volume={33},
  pages={21024--21037},
  year={2020}
}

@article{ma2021conservative,
  title={Conservative offline distributional reinforcement learning},
  author={Ma, Yecheng and Jayaraman, Dinesh and Bastani, Osbert},
  journal={Advances in Neural Information Processing Systems (NeurIPS)},
  volume={34},
  pages={19235--19247},
  year={2021}
}

@inproceedings{Dabney2018Quantile,
    author = {Dabney, Will and Rowland, Mark and Bellemare, Marc G. and Munos, R{\'e}mi},
    booktitle={Proceedings of the AAAI Conference on Artificial Intelligence.},
    title = {Distributional reinforcement learning with quantile regression},
    year = {2018},
}

@InProceedings{Marc2017Distributional,
  title = 	 {A Distributional Perspective on Reinforcement Learning},
  author =       {Marc G. Bellemare and Will Dabney and R{\'e}mi Munos},
  booktitle = 	 {International Conference on Machine Learning (ICML)},
  year = 	 {2017}
}

@inproceedings{nishimura2023rap,
  title={Rap: Risk-aware prediction for robust planning},
  author={Nishimura, Haruki and Mercat, Jean and Wulfe, Blake and McAllister, Rowan Thomas and Gaidon, Adrien},
  booktitle={Conference on Robot Learning (CoRL)},
  pages={381--392},
  year={2023}
}

@article{chow2018risk,
  title={Risk-constrained reinforcement learning with percentile risk criteria},
  author={Chow, Yinlam and Ghavamzadeh, Mohammad and Janson, Lucas and Pavone, Marco},
  journal={Journal of Machine Learning Research},
  volume={18},
  number={167},
  pages={1--51},
  year={2018}
}

@article{chow2015risk,
  title={Risk-sensitive and robust decision-making: a cvar optimization approach},
  author={Chow, Yinlam and Tamar, Aviv and Mannor, Shie and Pavone, Marco},
  journal={Advances in Neural Information Processing Systems (NeurIPS)},
  volume={28},
  year={2015}
}

@inproceedings{
    hu2023deception,
    title={Deception Game: Closing the Safety-Learning Loop in Interactive Robot Autonomy},
    author={Haimin Hu and Zixu Zhang and Kensuke Nakamura and Andrea Bajcsy and Jaime Fern{\'a}ndez Fisac},
    booktitle={Conference on Robot Learning (CoRL)},
    year={2023}
}

@inproceedings{johnstone2021conformal,
  title={Conformal uncertainty sets for robust optimization},
  author={Johnstone, Chancellor and Cox, Bruce},
  booktitle={Conformal and Probabilistic Prediction and Applications},
  pages={72--90},
  year={2021},
  organization={PMLR}
}

@inproceedings{
    zhou2026calibrating,
    title={Calibrating Decision Robustness via Inverse Conformal Risk Control},
    author={Wenbin Zhou and Shixiang Zhu},
    booktitle={International Conference on Machine Learning (ICML)},
    year={2026}
}

@article{sun2023conformal,
  title={Conformal prediction for uncertainty-aware planning with diffusion dynamics model},
  author={Sun, Jiankai and Jiang, Yiqi and Qiu, Jianing and Nobel, Parth and Kochenderfer, Mykel J and Schwager, Mac},
  journal={Advances in Neural Information Processing Systems (NeurIPS)},
  volume={36},
  pages={80324--80337},
  year={2023}
}

@inproceedings{marques2024quantifying,
  title={Quantifying aleatoric and epistemic dynamics uncertainty via local conformal calibration},
  author={Marques, Lu{\'\i}s and Berenson, Dmitry},
  booktitle={International Workshop on the Algorithmic Foundations of Robotics},
  pages={85--103},
  year={2024},
  organization={Springer}
}

@article{hansen2026hallucination,
  title={Hallucination in World Models is Predictable and Preventable},
  author={Hansen, Nicklas and Wang, Xiaolong},
  journal={arXiv preprint arXiv:2606.27326},
  year={2026}
}

@inproceedings{seo2026stressdream,
  title     = {StressDream: Steering Video World Models for Robust Policy Evaluation and Improvement},
  author    = {Seo, Junwon and Veer, Sushant and Tian, Ran and Ding, Wenhao and Sharma, Apoorva and Leung, Karen and Schmerling, Edward and Pavone, Marco and Bajcsy, Andrea},
  booktitle = {Conference on Robot Learning (CoRL)},
  year      = {2026}
}

@inproceedings{
    guo2026vlaw,
    title={{VLAW}: Iterative Co-Improvement of Vision-Language-Action Policy and World Model},
    author={Yanjiang Guo and Tony Lee and Lucy Xiaoyang Shi and Jianyu Chen and Percy Liang and Chelsea Finn},
    booktitle={International Conference on Machine Learning (ICML)},
    year={2026}
}

@inproceedings{
    quevedo2026worldgym,
    title={WorldGym: World Model as An Environment for Policy Evaluation},
    author={Julian Hector Quevedo and Ansh Kumar Sharma and Yixiang Sun and Varad Suryavanshi and Percy Liang and Sherry Yang},
    booktitle={International Conference on Learning Representations (ICLR)},
    year={2026}
}

@article{sharma2026world,
  title={World-Gymnast: Training Robots with Reinforcement Learning in a World Model},
  author={Sharma, Ansh Kumar and Sun, Yixiang and Lu, Ninghao and Zhang, Yunzhe and Liu, Jiarao and Yang, Sherry},
  journal={arXiv preprint arXiv:2602.02454},
  year={2026}
}

@inproceedings{guo2026ctrl,
  title={Ctrl-world: A controllable generative world model for robot manipulation},
  author={Guo, Yanjiang and Shi, Lucy and Chen, Jianyu and Finn, Chelsea},
  booktitle={International Conference on Learning Representations (ICLR)},
  year={2026}
}

@article{agarwal2025cosmos,
  title={Cosmos world foundation model platform for physical ai},
  author={Agarwal, Niket and Ali, Arslan and Bala, Maciej and Balaji, Yogesh and Barker, Erik and Cai, Tiffany and Chattopadhyay, Prithvijit and Chen, Yongxin and Cui, Yin and Ding, Yifan and others},
  journal={arXiv preprint arXiv:2501.03575},
  year={2025}
}

@article{gao2026dreamdojo,
  title={Dreamdojo: A generalist robot world model from large-scale human videos},
  author={Gao, Shenyuan and Liang, William and Zheng, Kaiyuan and Malik, Ayaan and Ye, Seonghyeon and Yu, Sihyun and Tseng, Wei-Cheng and Dong, Yuzhu and Mo, Kaichun and Lin, Chen-Hsuan and others},
  journal={arXiv preprint arXiv:2602.06949},
  year={2026}
}

@article{maes2026leworldmodel,
  title={Leworldmodel: Stable end-to-end joint-embedding predictive architecture from pixels},
  author={Maes, Lucas and Lidec, Quentin Le and Scieur, Damien and LeCun, Yann and Balestriero, Randall},
  journal={arXiv preprint arXiv:2603.19312},
  year={2026}
}

@article{ye2026world,
  title={World action models are zero-shot policies},
  author={Ye, Seonghyeon and Ge, Yunhao and Zheng, Kaiyuan and Gao, Shenyuan and Yu, Sihyun and Kurian, George and Indupuru, Suneel and Tan, You Liang and Zhu, Chuning and Xiang, Jiannan and others},
  journal={arXiv preprint arXiv:2602.15922},
  year={2026}
}

@article{hafner2025training,
  title={Training agents inside of scalable world models},
  author={Hafner, Danijar and Yan, Wilson and Lillicrap, Timothy},
  journal={arXiv preprint arXiv:2509.24527},
  year={2025}
}

@article{huang2026nano,
  title={Nano World Models: A Minimalist Implementation of Future Video Prediction},
  author={Huang, Siqiao and Kaushik, Partha and Chen, Michael and Pan, Hengkai and Geng, Kaiwen and Chehab, Omar and Moreno-Pino, Fernando and Simchowitz, Max},
  journal={arXiv preprint arXiv:2605.23993},
  year={2026}
}

@article{jain2026texttt,
  title={WEAVER, Better, Faster, Longer: An Effective World Model for Robotic Manipulation},
  author={Jain, Arnav Kumar and Wu, Yilin and Farebrother, Jesse and Swamy, Gokul and Bajcsy, Andrea},
  journal={arXiv preprint arXiv:2606.13672},
  year={2026}
}

@InProceedings{Kerssies_2026_CVPR,
    author    = {Kerssies, Tommie and Berton, Gabriele and He, Ju and Yu, Qihang and Ma, Wufei and de Geus, Daan and Dubbelman, Gijs and Chen, Liang-Chieh},
    title     = {A Frame is Worth One Token: Efficient Generative World Modeling with Delta Tokens},
    booktitle = {Proceedings of the IEEE/CVF Conference on Computer Vision and Pattern Recognition (CVPR)},
    month     = {June},
    year      = {2026},
    pages     = {27978-27988}
}

@article{zhang2023storm,
  title={Storm: Efficient stochastic transformer based world models for reinforcement learning},
  author={Zhang, Weipu and Wang, Gang and Sun, Jian and Yuan, Yetian and Huang, Gao},
  journal={Advances in Neural Information Processing Systems (NeurIPS)},
  volume={36},
  pages={27147--27166},
  year={2023}
}

@article{nath2026pixels,
  title={Pixels to Proofs: Probabilistically-Safe Latent World Model Control via Parallel Conformal Robust MPC},
  author={Nath, Devesh and Srinivasan, Anutam and Yin, Haoran and Jiang, Ruitong and Fang, Jeffrey and Chou, Glen},
  journal={arXiv preprint arXiv:2606.15594},
  year={2026}
}

@inproceedings{chi2023diffusionpolicy,
	title={Diffusion Policy: Visuomotor Policy Learning via Action Diffusion},
	author={Chi, Cheng and Feng, Siyuan and Du, Yilun and Xu, Zhenjia and Cousineau, Eric and Burchfiel, Benjamin and Song, Shuran},
	booktitle={Robotics: Science and Systems (RSS)},
	year={2023}
}

@article{intelligence2025pi,
  title={{$\pi_{0.5}$}: A Vision-Language-Action Model with Open-World Generalization},
  author={Physical Intelligence and Black, Kevin and Brown, Noah and Darpinian, James and Dhabalia, Karan and Driess, Danny and Esmail, Adnan and Equi, Michael and Finn, Chelsea and Fusai, Niccolo and others},
  journal={arXiv preprint arXiv:2504.16054},
  year={2025}
}

@article{mei2026video,
  title={Video Generation Models in Robotics-Applications, Research Challenges, Future Directions},
  author={Mei, Zhiting and Yin, Tenny and Shorinwa, Ola and Badithela, Apurva and Zheng, Zhonghe and Bruno, Joseph and Bland, Madison and Zha, Lihan and Hancock, Asher and Fisac, Jaime Fern{\'a}ndez and others},
  journal={arXiv preprint arXiv:2601.07823},
  year={2026}
}

@inproceedings{haarnoja2018soft,
  title={Soft actor-critic: Off-policy maximum entropy deep reinforcement learning with a stochastic actor},
  author={Haarnoja, Tuomas and Zhou, Aurick and Abbeel, Pieter and Levine, Sergey},
  booktitle={International Conference on Machine Learning (ICML)},
  pages={1861--1870},
  year={2018}
}

@inproceedings{balaji2026oopsieverse,
  title={OopsieVerse: A Safety Benchmark with Damage-Aware Simulation for Robot Manipulation},
  author={Balaji, Arnav and Bahety, Arpit and Ambatipudi, Sriniket and Lam, Daniel and Xu, Junhong and Mart{\'i}n-Mart{\'i}n, Roberto},
  booktitle={Robotics: Science and Systems (RSS)},
  year={2026}
}

\appendices
\ifanonymous
  \refstepcounter{section}\label{sec:theory}
\else
  \section{Theoretical Analysis: Robust Safety Filter}\label{sec:theory}

We provide a theoretical analysis of \ours for safety filtering with learned latent dynamics, focusing on uncertainty in the latent dynamics and assuming no additional uncertainty from the learned representation. Assuming that the conformalized uncertainty set contains the latent dynamics realized at test time, we show that (i) the robust safety value does not overestimate the safety value under any plausible dynamics, whereas the nominal safety value can, and (ii) the perfect robust safety policy achieves zero worst-case suboptimality.

\begin{assumption}[Markov Latent State]\label{ass:markov-latent}
The latent state $\latent_t$ is sufficient for predicting the system's evolution. Consequently, the induced latent process is Markov and admits a conditional transition distribution of the form $\dynz(\latent_{t+1} \mid \latent_t,\action_t)$.
\end{assumption}

\para{Zero Safety Overestimation}
Safety values computed from nominal imaginations can be overly optimistic when the realized dynamics are adverse. In contrast, the robust safety value evaluates each action under the most adverse plausible transition in the uncertainty set, thereby avoiding overestimation. For a latent dynamics model \(\dynz\), given a policy $\policy$, define the discounted Bellman operator of a value function $\valfunc$ based on its imaginations:
\begin{equation}\label{eq:theory-model-bellman-eval}
\begin{aligned}
(\bellmanOperator_{\dynz}^\policy \valfunc)(\latent)
&:=(1-\discount)\ellz(\latent)\\
&\quad+\discount\min\!\left\{\ellz(\latent),
\E_{\latent'\sim \dynz(\cdot\mid\latent,\policy(\latent))}
[\valfunc(\latent')]\right\},
\end{aligned}
\end{equation}where its robust counterpart using $\latentDisturbanceDynamics\in\ambiguity(\dynz)$ is
\begin{equation}\label{eq:theory-robust-bellman-eval}
\begin{aligned}
&(\robustOperator^\policy \valfunc)(\latent)
:=(1-\discount)\ellz(\latent)\\[-0.2em]
&+\discount\min\!\left\{\ellz(\latent),
\min_{\latentDisturbanceDynamics\in\ambiguity(\latent,\policy(\latent))}
\E_{\latent'\sim \latentDisturbanceDynamics(\latent,\policy(\latent))}[\valfunc(\latent')]\right\}.
\end{aligned}
\end{equation} Let $\nominalValue^\policy$ and $\robustValue^\policy$ denote their unique fixed points, and denote the rectangular \textit{global} uncertainty set is
\begin{equation}\label{eq:theory-global-ambiguity-set}
\ambiguity(\dynz)
:=\left\{\latentDisturbanceDynamics:\latentDisturbanceDynamics(\cdot\mid\latent,\action)
\in\ambiguity(\latent,\action),\ \forall(\latent,\action)\right\}.
\end{equation}

\begin{theorem}[Robust Value Does Not Overestimate]
\label{thm:theory-zero-overestimation}
For any policy $\policy$ and any plausible latent dynamics
$\latentDisturbanceDynamics\in\ambiguity(\dynz)$,
\begin{equation}\label{eq:theory-zero-overestimation}
\robustValue^\policy(\latent)
\leq\valfunc_{\latentDisturbanceDynamics}^\policy(\latent),
\qquad\forall\latent\in\latentSpace.
\end{equation}
\end{theorem}

\begin{proof}[Proof of Theorem~\ref{thm:theory-zero-overestimation}]
Fix $\policy$ and $\latentDisturbanceDynamics\in\ambiguity(\dynz)$. By
\eqref{eq:theory-global-ambiguity-set},
$\latentDisturbanceDynamics(\cdot\mid\latent,\policy(\latent))$ belongs to the local ambiguity set
at every $\latent$. The minimum in the robust backup therefore cannot exceed
the expectation under this particular transition. Hence, for
every bounded $\valfunc$,
\begin{equation*}
\robustOperator^\policy\valfunc
\leq\bellmanOperator_{\latentDisturbanceDynamics}^\policy\valfunc
\qquad\text{pointwise}.
\end{equation*}
Applying this inequality to the fixed point gives
\begin{equation*}
\robustOperator^\policy\valfunc_{\latentDisturbanceDynamics}^\policy
\leq\bellmanOperator_{\latentDisturbanceDynamics}^\policy
\valfunc_{\latentDisturbanceDynamics}^\policy
=\valfunc_{\latentDisturbanceDynamics}^\policy.
\end{equation*}
Monotonicity gives
$(\robustOperator^\policy)^m\valfunc_{\latentDisturbanceDynamics}^\policy
\leq\valfunc_{\latentDisturbanceDynamics}^\policy$ for every $m$. By contraction, $(\robustOperator^\policy)^m\valfunc_{\latentDisturbanceDynamics}^\policy$ converges to $\robustValue^\policy$, proving
\eqref{eq:theory-zero-overestimation}.
\end{proof}

Theorem~\ref{thm:theory-zero-overestimation} states that, for any plausible dynamics in the uncertainty set, the safety value computed by robust optimization remains a lower bound on the realized safety value, enabling robust safety filtering. In contrast, the nominal value may overestimate safety when the realized dynamics lead to more pessimistic next states than the nominal model.

\begin{remark} The above guarantee applies under the perfect coverage guarantee where the true latent dynamics are contained in the uncertainty set, while conformal calibration only provides a marginal probabilistic guarantee.
\end{remark}

\begin{corollary}[Nominal Safety Value Can Overestimate]\label{cor:corollary-zoverestimation}
For a policy $\policy$, there exists an admitted latent dynamics
$\latentDisturbanceDynamics\in\ambiguity(\dynz)$ such that the nominal safety value
computed with $\dynz$ overestimates the realized safety value:
\[
\nominalValue^\policy(\latent)
\ge
\valfunc_{\latentDisturbanceDynamics}^\policy(\latent).
\]
\end{corollary}

\begin{proof}[Proof of Corollary~\ref{cor:corollary-zoverestimation}]
Let $\optimalDisturbance$ attain the worst-case admitted value:
\[
\optimalDisturbance
\in
\argmin_{\latentDisturbanceDynamics\in\ambiguity(\dynz)}
\valfunc_{\latentDisturbanceDynamics}^\policy(\latent).
\]
Then,
\(
\valfunc_{\optimalDisturbance}^\policy(\latent)
=
\robustValue^\policy(\latent).
\) Since the nominal dynamics $\dynz$ is also admitted,
Theorem~\ref{thm:theory-zero-overestimation} gives that the safety value in $\optimalDisturbance$ can be lower than the nominal value computed with $\dynz$
\[
\nominalValue^\policy(\latent)
\ge
\robustValue^\policy(\latent) = \valfunc_{\optimalDisturbance}^\policy(\latent).
\]
Thus, an adverse but plausible dynamics can make the nominal safety value
overestimate the realized value.
\end{proof}

\para{Sub-optimality of Safety Policy}
Robust optimization aims to select actions that remain effective across plausible outcomes, while a nominal policy optimized only under nominal imaginations may select actions that become suboptimal when adverse dynamics are realized. To formalize this, for a policy $\policy$, define its state-wise worst-case safety value:
\begin{equation}\label{eq:theory-worst-case-policy-value}
\worstCaseValue^\policy(\latent)
:=\min_{\latentDisturbanceDynamics\in\ambiguity(\dynz)}
\valfunc_{\latentDisturbanceDynamics}^\policy(\latent),
\end{equation} representing the safety value achieved by $\policy$ under the worst-case dynamics. We then define the worst-case suboptimality of $\policy$ as the gap between its worst-case safety value and the maximum safety value achievable for the worst-case:
\begin{equation}\label{eq:theory-worst-case-suboptimality}
\worstCaseSuboptimality(\policy,\latent):=\max_{\policy'}\worstCaseValue^{\policy'}(\latent),
-\worstCaseValue^\policy(\latent).
\end{equation}
Thus, zero suboptimality means that a policy achieves the best possible safety value despite the most adverse dynamics. In this analysis, we write $\robustPolicy$ for the robust safety policy $\fallbackRobust$. Let $\robustPolicy$ solve the robust Bellman equation \eqref{eq:robust-fixed-pt-bellman}, and let $\nominalPolicy$ solve the nominal Bellman equation \eqref{eq:trad-fixed-pt-bellman}.

\begin{theorem}[Worst-Case Suboptimality of the Robust Policy]
\label{thm:theory-robust-zero-suboptimality}
For the global uncertainty set in
\eqref{eq:theory-global-ambiguity-set} and every $\latent\in\latentSpace$,
\begin{equation}\label{eq:theory-robust-zero-suboptimality}
\worstCaseSuboptimality(\robustPolicy,\latent)=0.
\end{equation}
\end{theorem}

\begin{proof}[Proof of Theorem~\ref{thm:theory-robust-zero-suboptimality}]
Under the rectangular ambiguity set \eqref{eq:theory-global-ambiguity-set}, for any $\policy$,
\begin{equation}\label{eq:theory-robust-dp-identity}
\worstCaseValue^\policy(\latent)
=\min_{\latentDisturbanceDynamics\in\ambiguity(\dynz)}
\valfunc_{\latentDisturbanceDynamics}^\policy(\latent)
=\robustValue^\policy(\latent).
\end{equation} Since $\robustPolicy$ is optimal for the robust Bellman equation,
\begin{equation*}
\worstCaseValue^{\robustPolicy}(\latent)
=
\robustValue^{\robustPolicy}(\latent)
=
\max_{\policy}\robustValue^\policy(\latent)
=
\max_{\policy}\worstCaseValue^{\policy}(\latent),
\end{equation*}
which proves $\worstCaseSuboptimality(\robustPolicy,\latent)=0$.
\end{proof}

\begin{corollary}[Nominal Safety Policy Can Be Suboptimal]
\label{cor:theory-nominal-suboptimality}
Following Theorem~\ref{thm:theory-robust-zero-suboptimality}, for every $\latent\in\latentSpace$,
\begin{equation}\label{eq:theory-worst-case-suboptimality-bound}
0
\leq\worstCaseSuboptimality(\nominalPolicy,\latent)
\leq 2 \,\min\left(1, \,\frac{\sqrt{2}\,\discount}{1-\discount}
\sqrt{\thresholdKL}\right).
\end{equation}
Also, the worst-case suboptimality is positive whenever
\begin{equation*}
\worstCaseValue^{\nominalPolicy}(\latent)
<\worstCaseValue^{\robustPolicy}(\latent).
\end{equation*}
\end{corollary}

\begin{proof}[Proof of Corollary~\ref{cor:theory-nominal-suboptimality}]
By Theorem~\ref{thm:theory-robust-zero-suboptimality}, $\nominalPolicy$ is suboptimal:
\begin{equation*}
\worstCaseSuboptimality(\nominalPolicy,\latent)
=\worstCaseValue^{\robustPolicy}(\latent)
-\worstCaseValue^{\nominalPolicy}(\latent)
\geq0.
\end{equation*}
Consider any two latent dynamics models $f,g$, and the fixed-point property and triangle inequality give
\begin{align*}
\|\valfunc_f^\policy-\valfunc_g^\policy\|_\infty\leq\|\bellmanOperator_f^\policy\valfunc_f^\policy
-\bellmanOperator_f^\policy\valfunc_g^\policy\|_\infty+\|\bellmanOperator_f^\policy\valfunc_g^\policy
-\bellmanOperator_g^\policy\valfunc_g^\policy\|_\infty.
\end{align*}
The first term is at most
$\discount\|\valfunc_f^\policy-\valfunc_g^\policy\|_\infty$. Since the minimum is
non-expansive, the second term is at most
\begin{equation*}
\discount\max_{\latent}\left|
\E_{f}[\valfunc_g^\policy(\latent')]
-\E_{g}[\valfunc_g^\policy(\latent')]
\right|,
\end{equation*}
where both expectations condition on $(\latent,\policy(\latent))$. Thus,
\begin{equation}\label{eq:theory-simulation-bound}
\|\valfunc_f^\policy-\valfunc_g^\policy\|_\infty
\leq\frac{\discount}{1-\discount}
\max_{\latent}\left|
\E_{f}[\valfunc_g^\policy(\latent')]
-\E_{g}[\valfunc_g^\policy(\latent')]
\right|.
\end{equation}
For every $\latentDisturbanceDynamics\in\ambiguity(\dynz)$, the KL condition in
\eqref{eq:ambiguity-set} and Pinsker's inequality imply, for any
$\valfunc:\latentSpace\to[-1,1]$,
\begin{equation*}
\left|\E_{\latentDisturbanceDynamics}[\valfunc]
-\E_{\dynz}[\valfunc]\right|
\leq\sqrt{2\thresholdKL}.
\end{equation*}
Substituting this into \eqref{eq:theory-simulation-bound} gives
\begin{equation}\label{eq:theory-sensitivity-bound}
\|\valfunc_{\latentDisturbanceDynamics}^\policy-\nominalValue^\policy\|_\infty
\leq\sensitivity = \frac{\sqrt{2}\,\discount}{1-\discount}
\sqrt{\thresholdKL}.
\end{equation}
Therefore, for every $\policy$,
\begin{equation*}
\nominalValue^\policy(\latent)-\sensitivity
\leq\worstCaseValue^\policy(\latent)
\leq\nominalValue^\policy(\latent)+\sensitivity.
\end{equation*}
Nominal optimality gives
$\nominalValue^{\robustPolicy}(\latent)
\leq\nominalValue^{\nominalPolicy}(\latent)$. Hence,
\begin{align*}
\worstCaseSuboptimality(\nominalPolicy,\latent)
&=\worstCaseValue^{\robustPolicy}(\latent)
-\worstCaseValue^{\nominalPolicy}(\latent)\\
&\leq
\nominalValue^{\robustPolicy}(\latent)+\sensitivity
-\nominalValue^{\nominalPolicy}(\latent)+\sensitivity\\
&\leq2\sensitivity.
\qedhere
\end{align*}
\end{proof}

Theorem~\ref{thm:theory-robust-zero-suboptimality} states that the exact robust policy achieves zero worst-case suboptimality, while Corollary~\ref{cor:theory-nominal-suboptimality} shows that the nominal policy has no such zero-suboptimality guarantee and can be suboptimal if the worst case is realized.

\begin{remark} The learned disturbance and policy only approximate this game, so the exact guarantees hold to the extent that ambiguity-set feasibility and optimization are achieved.\end{remark}

\fi
\section{Theoretical Analysis: Conformal Prediction}\label{sec:appendix-conformal}

In this section, we briefly review conformal prediction for the calibration procedure in Sec.~\ref{sec:conformal-calib}. Conformal prediction constructs prediction sets with finite-sample, distribution-free coverage under exchangeability~\cite{shafer2008tutorial,angelopoulos2023conformal}, which can be used as data-calibrated uncertainty sets for robust optimization~\cite{johnstone2021conformal}.

\para{Conformal Prediction}
Conformal prediction (CP) provides distribution-free coverage guarantees by calibrating a nonconformity threshold on held-out data~\cite{shafer2008tutorial,angelopoulos2023conformal}. In Sec.~\ref{sec:conformal-calib}, let $s_t^i$ denote the nonconformity score at time $t$ for calibration sample $i$, and write $s^i$ when $t$ is fixed. Given a user-specified miscoverage level $\alpha$, CP sets the threshold $\epsilon$ to the appropriate $(1-\alpha)$ quantile of the calibration scores. Under exchangeability between the calibration and test samples, this yields
\begin{equation}\label{eq:appendix-cp-quantile}
\begin{aligned}
\epsilon&:=s^{(k)}, \quad
k=\left\lceil (n+1)(1-\alpha)\right\rceil,\
\mathbb{P}\!\left(s^{\mathrm{test}}\leq\epsilon\right)\geq1-\alpha,
\end{aligned}
\end{equation}
where $s^{(k)}$ is the $k$th smallest of the $n$ calibration scores. The probability is marginal over both the calibration data and the test sample. We apply this procedure to calibrate the KL-ball radius with $(\epsilon,\alpha)=(\thresholdKL,\calibAmbiguity)$ and the OOD threshold with $(\epsilon,\alpha)=(\thresholdOOD,\calibOOD)$.

\para{Coverage of the Dynamics-Aware Similarity}
Our goal is to calibrate the KL-ball radius $\thresholdKL$ so that the uncertainty set contains the test-time latent distribution encoded from observations with high probability. In Sec.~\ref{sec:conformal-calib}, we measure the discrepancy between this encoded distribution and the learned latent dynamics using the KL nonconformity score:
\begin{equation}\label{eq:appendix-cp-kl-score}
s_{\mathrm{KL}}^i:=\KL\!\left(
\encoder(\obs^i_{\leq t})
\,\middle\|\,
\dynz(\latent^i_{t-1},\action^i_{t-1})
\right).
\end{equation}
Here, $t$ is fixed as above. We apply \eqref{eq:appendix-cp-quantile} with miscoverage level $\calibAmbiguity$ to obtain $\thresholdKL=s_{\mathrm{KL}}^{(k)}$, where $k=\lceil(n+1)(1-\calibAmbiguity)\rceil$. The following theorem gives the coverage guarantee in \eqref{eq:kl-ball-cp}.

\begin{theorem}[Marginal KL Coverage]\label{thm:appendix-cp-marginal}
Suppose that the calibration and test KL scores in \eqref{eq:appendix-cp-kl-score} are exchangeable and $k\leq n$. Then, the calibrated KL ball contains the test-time encoded latent distribution with probability at least $1-\calibAmbiguity$:
\begin{equation}\label{eq:appendix-cp-marginal}
\mathbb{P}(s_{\mathrm{KL}}^{\mathrm{test}}\leq\thresholdKL)
\geq1-\calibAmbiguity.
\end{equation}
\end{theorem}

\begin{proof}
The encoded test distribution belongs to the calibrated KL ball whenever its score satisfies $s_{\mathrm{KL}}^{\mathrm{test}}\leq\thresholdKL$. To bound this probability, consider the rank $J$ of the test score among the $n$ calibration scores and the test score, with ties broken uniformly at random. Under exchangeability, the test score is equally likely to occupy any of the $n+1$ ranks. If $J\leq k$, at most $k-1$ calibration scores precede the test score, so it cannot exceed the $k$th smallest calibration score $\thresholdKL$. Therefore,
\begin{equation*}
\mathbb{P}(s_{\mathrm{KL}}^{\mathrm{test}}\leq\thresholdKL)
\geq\mathbb{P}(J\leq k)
=\frac{k}{n+1}\geq1-\calibAmbiguity.
\qedhere
\end{equation*}
\end{proof}

\para{Dataset-Conditional KL Coverage}
In practice, we calibrate $\thresholdKL$ once using a held-out dataset and keep it fixed during deployment. We therefore consider the probability that the resulting KL ball contains the encoded distribution of a new test sample, conditional on the calibration dataset:
\begin{equation}\label{eq:appendix-cp-conditional}
p_{\mathrm{KL}}(\dataset_{\calib})
:=\mathbb{P}(s_{\mathrm{KL}}^{\mathrm{test}}\leq\thresholdKL\mid\dataset_{\calib}).
\end{equation}
Because different calibration datasets produce different radii, this coverage varies with $\dataset_{\calib}$. The following theorem provides a lower bound on the coverage attained by the calibrated radius with high probability over the calibration data~\cite{angelopoulos2023conformal}.

\begin{theorem}[Dataset-Conditional KL Coverage]\label{thm:appendix-cp-conditional}
Suppose that the calibration and test samples are i.i.d., and use $\thresholdKL=s_{\mathrm{KL}}^{(k)}$ with $k=\lceil(n+1)(1-\calibAmbiguity)\rceil$. The conditional coverage varies over calibration datasets according to
\begin{equation}\label{eq:appendix-cp-beta}
p_{\mathrm{KL}}(\dataset_{\calib})\sim\operatorname{Beta}(k,n+1-k).
\end{equation}
For a user-specified confidence level $1-\eta$, let $b_{\eta}$ be the $\eta$-quantile of this Beta distribution. Then, with probability at least $1-\eta$ over the calibration data, the calibrated KL ball has test coverage at least $b_{\eta}$:
\begin{equation}\label{eq:appendix-cp-confidence}
\mathbb{P}_{\dataset_{\calib}}\!\left(
p_{\mathrm{KL}}(\dataset_{\calib})\geq b_{\eta}\right)\geq1-\eta.
\end{equation}
\end{theorem}

\begin{proof}
Let $F$ denote the KL score cumulative distribution function. Since the test sample is independent of the calibration data,
\begin{equation*}
p_{\mathrm{KL}}(\dataset_{\calib})=F(\thresholdKL)=F(s_{\mathrm{KL}}^{(k)}).
\end{equation*}
For continuous $F$, the transformed calibration scores $U_i:=F(s_{\mathrm{KL}}^i)$ are i.i.d. uniform on $[0,1]$, and the coverage equals their $k$th order statistic $U_{(k)}$. The event $U_{(k)}\leq u$ occurs when at least $k$ of the $n$ scores lie below $u$. Thus,
\begin{equation*}
\mathbb{P}(U_{(k)}\leq u)
=\sum_{j=k}^{n}\binom{n}{j}u^j(1-u)^{n-j},
\end{equation*}
which is the cdf of $\operatorname{Beta}(k,n+1-k)$. Since $b_{\eta}$ is its $\eta$-quantile, $\mathbb{P}_{\dataset_{\calib}}(p_{\mathrm{KL}}(\dataset_{\calib})\geq b_{\eta})=1-\eta$, establishing the bound. With ties, the generalized-inverse representation $s_{\mathrm{KL}}^i=F^{-1}(U_i)$ gives $F(s_{\mathrm{KL}}^{(k)})\geq U_{(k)}$, so the lower bound remains valid.
\end{proof}

\para{Class-Conditional In-Distribution Detection}
Our goal is to calibrate the OOD threshold $\thresholdOOD$ so that in-distribution latent states are accepted with high probability. Since OOD states are not directly available, we follow Sec.~\ref{sec:conformal-calib} and use class-conditional conformal prediction to calibrate the threshold using only held-out ID states~\cite{ding2023class,seo2025uncertainty}. Given $n_{\mathrm{ID}}$ calibration states with scores $s_{\mathrm{OOD}}^i:=\densityEstimator(\latent_t^i)$, we apply \eqref{eq:appendix-cp-quantile} with miscoverage level $\calibOOD$:
\begin{equation*}
k_{\mathrm{ID}}:=\left\lceil(n_{\mathrm{ID}}+1)(1-\calibOOD)\right\rceil,
\qquad\thresholdOOD:=s_{\mathrm{OOD}}^{(k_{\mathrm{ID}})},
\end{equation*}
where we assume $k_{\mathrm{ID}}\leq n_{\mathrm{ID}}$. We classify a latent state as OOD when its score exceeds $\thresholdOOD$:
\begin{equation}\label{eq:appendix-cp-id-detector}
\widehat Y(\latent):=
\begin{cases}
\mathrm{ID}, & \densityEstimator(\latent)\leq\thresholdOOD,\\
\mathrm{OOD}, & \densityEstimator(\latent)>\thresholdOOD.
\end{cases}
\end{equation}
To quantify how reliably the detector retains ID states, we define its ID recall as the probability that an actual ID state is classified as ID:
\begin{equation}\label{eq:appendix-cp-id-recall-definition}
\begin{aligned}
\operatorname{Recall}_{\mathrm{ID}}
&:=\mathbb{P}(\widehat Y(\latent_{\mathrm{test}})=\mathrm{ID}
\mid Y_{\mathrm{test}}=\mathrm{ID})\\
&=1-\mathbb{P}(\widehat Y(\latent_{\mathrm{test}})=\mathrm{OOD}
\mid Y_{\mathrm{test}}=\mathrm{ID}).
\end{aligned}
\end{equation}
Here, $Y_{\mathrm{test}}$ denotes the true class, and the probability is taken over the ID calibration data and an ID test state. Thus, guaranteeing recall of at least $1-\calibOOD$ limits the probability of incorrectly rejecting an ID state to $\calibOOD$.

\begin{theorem}[In-Distribution Recall]\label{thm:appendix-cp-id-recall}
Suppose that the OOD score function is fixed before calibration and, conditional on $Y_{\mathrm{test}}=\mathrm{ID}$, the $n_{\mathrm{ID}}$ ID calibration states and test state are exchangeable. The calibrated detector in \eqref{eq:appendix-cp-id-detector} satisfies
\begin{equation}\label{eq:appendix-cp-id-recall}
\operatorname{Recall}_{\mathrm{ID}}
=\mathbb{P}\!\left(
\densityEstimator(\latent_{\mathrm{test}})\leq\thresholdOOD
\mid Y_{\mathrm{test}}=\mathrm{ID}\right)\geq1-\calibOOD.
\end{equation}
This establishes the ID recall guarantee in \eqref{eq:ood-cp}.
\end{theorem}

\begin{proof}
Conditional on the test state being ID, the calibration and test scores are exchangeable. Let $J_{\mathrm{ID}}$ be the rank of $s_{\mathrm{OOD}}^{\mathrm{test}}:=\densityEstimator(\latent_{\mathrm{test}})$ among these $n_{\mathrm{ID}}+1$ scores. Exchangeability makes this rank uniform. If $J_{\mathrm{ID}}\leq k_{\mathrm{ID}}$, the test score cannot exceed the $k_{\mathrm{ID}}$-th calibration score $\thresholdOOD$, so the detector accepts the test state as ID:
\begin{equation}
\operatorname{Recall}_{\mathrm{ID}}
\geq\mathbb{P}(J_{\mathrm{ID}}\leq k_{\mathrm{ID}}
\mid Y_{\mathrm{test}}=\mathrm{ID})=\frac{k_{\mathrm{ID}}}{n_{\mathrm{ID}}+1}\geq1-\calibOOD.
\end{equation}
The final inequality follows from the calibration rank.
\end{proof}

Thus, calibration controls how often ID states are incorrectly excluded from the uncertainty set, averaged over calibration and test data. The ability to reject OOD states depends on the learned score and is not guaranteed by ID calibration alone.

\para{Causal Reconstruction from Trajectory Calibration}
Individual transitions within a trajectory are temporally dependent, so the state-level exchangeability assumption need not hold. In Sec.~\ref{sec:conformal-calib}, we therefore calibrate the maximum nonconformity score over each trajectory. Given a calibration trajectory $\tau_i=\{(\obs_t^i,\action_t^i,\obs_{t+1}^i)\}_{t=1}^{T-1}$, we compute
\begin{equation}\label{eq:appendix-cp-trajectory-calibration}
S(\tau_i):=\max_{1\leq t<T}s_t^i, \quad
\epsilon^{\mathrm{traj}}\!:=\!\operatorname{Quantile}_{1-\alpha}\!\left(
\{S(\tau_i)\}_{i=1}^{N_{\mathrm{calib}}}\right).
\end{equation}
We then use the calibrated threshold at each transition. Since the maximum score is below the threshold exactly when every state-level score is below it, trajectory-level calibration provides simultaneous coverage across time. The following theorem formalizes this causal reconstruction, following~\cite{ren2023robots}.

\begin{theorem}[Causal Reconstruction of Trajectory Coverage]\label{thm:appendix-cp-causal}
Suppose that the calibration trajectories and test trajectory are exchangeable, the score functions are fixed, and each $s_t$ depends only on observations and actions through the scored transition. Let $\epsilon^{\mathrm{traj}}$ be calibrated by \eqref{eq:appendix-cp-trajectory-calibration}, with $k=\lceil(N_{\mathrm{calib}}+1)(1-\alpha)\rceil\leq N_{\mathrm{calib}}$. Applying the same threshold at every transition gives
\begin{equation}\label{eq:appendix-cp-causal-coverage}
\mathbb{P}\!\left(
s_t^{\mathrm{test}}\leq\epsilon^{\mathrm{traj}},\;
\forall\,1\leq t<T\right)\geq1-\alpha.
\end{equation}
\end{theorem}

\begin{proof}
Let $\mathcal{C}^{\mathrm{traj}}$ denote the trajectories accepted by the $\epsilon^{\mathrm{traj}}$ and $\mathcal{C}^{\mathrm{causal}}$ those accepted at every transition:
\begin{equation*}
\begin{aligned}
\mathcal{C}^{\mathrm{traj}}
&:=\{\tau:S(\tau)\leq\epsilon^{\mathrm{traj}}\},\\
\mathcal{C}^{\mathrm{causal}}
&:=\{\tau:s_t(\tau)\leq\epsilon^{\mathrm{traj}},\ \forall\,1\leq t<T\}.
\end{aligned}
\end{equation*}
For any trajectory $\tau$, the definition of its maximum score gives
\begin{equation*}
\begin{aligned}
\tau\in\mathcal{C}^{\mathrm{traj}}
&\iff\max_{1\leq t<T}s_t(\tau)\leq\epsilon^{\mathrm{traj}}\\
&\iff s_t(\tau)\leq\epsilon^{\mathrm{traj}},\quad\forall\,1\leq t<T\\
&\iff\tau\in\mathcal{C}^{\mathrm{causal}}.
\end{aligned}
\end{equation*}
The two sets therefore describe the same acceptance event. Since the trajectories are exchangeable and the score function is fixed, their maximum scores are also exchangeable:
\begin{equation*}
\mathbb{P}(\tau_{\mathrm{test}}\in\mathcal{C}^{\mathrm{traj}})
\geq\frac{k}{N_{\mathrm{calib}}+1}\geq1-\alpha.
\end{equation*}
Substituting $\mathcal{C}^{\mathrm{causal}}=\mathcal{C}^{\mathrm{traj}}$ proves \eqref{eq:appendix-cp-causal-coverage}.
\end{proof}

Thus, trajectory-level calibration guarantees coverage at all transitions over the calibrated horizon with a single miscoverage level $\alpha$, while allowing temporal dependence within each trajectory. For i.i.d. trajectories, the dataset-conditional bound in Theorem~\ref{thm:appendix-cp-conditional} also applies to the maximum scores with $n=N_{\mathrm{calib}}$. These guarantees rely on exchangeability between calibration and deployment trajectories.

\ifanonymous
\else
    \section{Background: HJ Reachability Analysis}\label{sec:appendix-HJ}

A safety filter is a policy-agnostic control mechanism that safeguards any task policy by evaluating its proposed actions and intervening when necessary to prevent failure under admissible disturbances~\cite{hsu2023safety}. We consider the discrete-time system
$\state_{t+1}=\dynamics(\state_t,\action_t,\disturbance_t)$, where
$\state_t\in\stateSpace$, $\action_t\in\actionSpace$, and
$\disturbance_t\in\disturbanceSpace$ denotes an unknown bounded disturbance.
Safety is specified by a margin function $\failuremargin$, where negative values indicate failure. Following~\cite{oh2026synthesis}, we define the failure set as
$\failure=\{\state:\failuremargin(\state)<0\}$, so states with zero margin lie on the safety boundary rather than in the failure set.

\para{Maximal Robust Safe Set}
Our goal is to identify the states from which the robot \textit{can} avoid failure for all time \textit{despite} admissible disturbances. Let $\policy:\stateSpace\to\actionSpace$ and
$\disturbancePolicy:\stateSpace\to\disturbanceSpace$ be feedback policies,
and let $\state_{\state}^{\policy,\disturbancePolicy}(t)$ denote the trajectory
starting from $\state$ under these policies. We define the maximal robust safe set as
\begin{equation*}
\Omega^\star:=\big\{\state\in\stateSpace:\;\exists\policy,\ \forall\disturbancePolicy,\
\forall t\in\mathbb{N}_0,\state_{\state}^{\policy,\disturbancePolicy}(t)\notin\failure\big\}.
\end{equation*}
For every state in $\Omega^\star$, there exists a control policy that keeps the system outside $\failure$ under every admissible disturbance.

\para{Dynamic Programming Isaacs Equation}
HJ reachability characterizes this set through a safety value function~\cite{bansal2017hamilton}. Since a trajectory enters failure whenever its margin becomes negative, we evaluate the smallest margin attained over time. The controller maximizes this margin against the worst-case disturbance, giving the zero-sum safety game:
\begin{equation}\label{eq:appendix-hj-value}
\valfunc(\state):=\sup_{\policy}\inf_{\disturbancePolicy}
\inf_{t\in\mathbb{N}_0}
\failuremargin\!\left(\state_{\state}^{\policy,\disturbancePolicy}(t)\right).
\end{equation}
The value $\valfunc(\state)$ therefore represents the minimum future safety margin that an optimal controller can maintain despite disturbances. A positive value indicates that the controller can keep the system safely outside the failure set, while a negative value indicates unavoidable failure under an adverse disturbance~\cite{hsu2023isaacs, hsu2023safety, oh2026synthesis}. 
To compute this value, we apply the discrete-time dynamic programming Isaacs equation:
\begin{equation}\label{eq:appendix-hj-isaacs}
\valfunc(\state)=\min\!\left\{
\failuremargin(\state),\,
\max_{\action\in\actionSpace}\min_{\disturbance\in\disturbanceSpace}
\valfunc\!\left(\dynamics(\state,\action,\disturbance)\right)
\right\}.
\end{equation}
The infinite-horizon value defines the maximal robust safe set:
\begin{equation*}
\Omega^\star=\{\state:\valfunc(\state)\geq0\},
\end{equation*}
A corresponding robust safety policy is
\begin{equation*}
\robustPolicy(\state)\in
\argmax_{\action\in\actionSpace}
\min_{\disturbance\in\disturbanceSpace}
\valfunc\!\left(\dynamics(\state,\action,\disturbance)\right).
\end{equation*}
This value function provides a criterion for safety filtering: starting within $\Omega^\star$, the filter allows $\policyTask(\state)$ if its worst-case next-state value is nonnegative; otherwise, it intervenes with $\robustPolicy(\state)$ to keep the system within the robust safe set.

\para{Discounted Safety Update}
The undiscounted backup in \eqref{eq:appendix-hj-isaacs} is not generally a contraction. Consequently, a discount factor is introduced $\discount\in[0,1)$ for a contraction mapping ~\cite{fisac2019bridging}:
\begin{equation}\label{eq:appendix-hj-discounted}
\valfunc(\state)\!\leftarrow\!(1-\discount)\failuremargin(\state)+\discount\min\!\left\{
\failuremargin(\state),\,
\max_{\action\in\actionSpace}\min_{\disturbance\in\disturbanceSpace}
\valfunc\!\left(\dynamics(\state,\action,\disturbance)\right)
\right\}.
\end{equation}

\para{Game-Theoretic Adversarial RL} While \eqref{eq:appendix-hj-discounted} defines a contraction mapping for the robust safety value function, solving it directly becomes intractable in high-dimensional spaces. To approximate its solution, adversarial actor--critic methods learn a safety critic $Q_\omega(\state,\action,\disturbance)$ together with a control policy and a disturbance policy~\cite{hsu2023isaacs, nguyen2024gameplay}. The critic evaluates the next-state safety value,
\(
Q(\state,\action,\disturbance):=
\valfunc\left(\dynamics(\state,\action,\disturbance)\right).
\) The control actor maximizes the critic value, while the disturbance actor minimizes it. Repeated adversarial gameplay jointly improves the safety-value approximation and the competing policies. However, such game-theoretic learning can oscillate as the two policies continually adapt to one another. To stabilize training, ISAACS updates the disturbance and critic more frequently than the controller and maintains a leaderboard of past opponents, reducing overfitting to a single adversary~\cite{hsu2023isaacs}. Similarly, separating the learning rates allows the disturbance to more closely track a best response while the controller evolves more slowly; together with suitable critic updates, this supports local convergence under the assumptions in~\cite{wang2024magics}.

\para{Operational Design Domain}
These methods assume access to a system dynamics model $\dynamics(\state,\action,\disturbance)$, such as a simulator, together with a physically meaningful disturbance space $\disturbanceSpace$ specified by the designer to capture uncertainty within the operational design domain (ODD), such as bounded external forces or model errors~\cite{hsu2023isaacs}. Thus, while the disturbance \emph{policy} is learned, the admissible uncertainty itself is prescribed. In \ours, the game instead operates over learned latent dynamics: the world model provides imagined transitions, while the disturbance perturbs the latent predictive distribution within an uncertainty set calibrated from data. This extends adversarial safety learning to systems for which physical dynamics and disturbances are difficult to specify explicitly.

\fi
\section{Implementation Details}\label{sec:appendix-experiments}

\subsection{Adversarial Reinforcement Learning}

Algorithm~\ref{alg:impl-safety-learning} summarizes the game-theoretic adversarial reinforcement learning procedure used to compute the robust latent safety filter in Sec.~\ref{sec:robust-safety-filter}. Following Soft Actor-Critic~\cite{haarnoja2018soft}, we implement $\QfunctionRobust$ with twin safety critics and use their minimum in both the control and disturbance updates to mitigate value overestimation; the target critic $\targetQfunctionRobust$ is updated as an exponential moving average of the critic parameters. The replay buffer $\buffer$ stores each imagined transition together with the parameters of the nominal latent predictive distribution, so that the current disturbance policy can efficiently resample perturbed next latent states at every update without re-querying the world model. Rollouts are collected in parallel across independent WM imaginations, and transitions from all imaginations are aggregated into $\buffer$ and sampled in minibatches for replay updates.

Each of the $N_{\mathrm{iter}}$ training iterations starts from a newly encoded trajectory and imagines up to $H_{\mathrm{imag}}$ steps. We truncate an imagination once the next latent state satisfies $\densityEstimator(\latent_{t+1})>\thresholdOOD$, preventing rollouts from extending into regions unsupported by the training data. Each imagination step is followed by $G$ replay updates, where $j$ counts replay updates across iterations, $K_\pi$ and $K_d$ are the policy and disturbance update intervals, $\eta_Q,\eta_\pi,\eta_d$ are the corresponding learning rates, and $\rho$ is the target-update coefficient. The world model and the uncertainty models remain fixed throughout safety-filter training.

\begin{algorithm}[!ht]
\caption{Adversarial RL for Robust Latent Safety Filter}
\label{alg:impl-safety-learning}
\small
\begin{algorithmic}[1]
\REQUIRE $\dataset_{\mathrm{train}},\encoder,\dynz,\ellz,\densityEstimator,\thresholdKL,\thresholdOOD$; $N_{\mathrm{iter}},H_{\mathrm{imag}},G,K_\pi,K_d$
\ENSURE $\QfunctionRobust,\ \fallbackRobust,\ \latentDisturbance$
\STATE Initialize $\theta_Q,\theta_\pi,\latentDisturbanceParam$ (zero residual output); $\bar\theta_Q\leftarrow\theta_Q$, $\buffer\leftarrow\emptyset$, $j\leftarrow0$
\FOR{$n=1,\ldots,N_{\mathrm{iter}}$}
    \STATE $\tau\sim\dataset_{\mathrm{train}},\quad\latent_0\sim\encoder(\cdot\mid\obs_{\leq t_0})$, with $\obs_{\leq t_0}$ from $\tau$
    \FOR{$t=0,\ldots,H_{\mathrm{imag}}-1$}
        \STATE $\action_t\sim\fallbackRobust(\cdot\mid\latent_t)$
        \STATE $\latent_{t+1}\sim\latentDisturbanceDynamicsParameterized(\cdot\mid\latent_t,\action_t)$
        \STATE $\buffer\leftarrow\buffer\cup\{(\latent_t,\action_t,\ellz(\latent_t),\dynz(\cdot\mid\latent_t,\action_t))\}$
        \FOR{$g=1,\ldots,G$}
            \STATE $\mathcal{M}\sim\buffer,\quad j\leftarrow j+1$
            \STATE $\latent'\sim\latentDisturbanceDynamicsParameterized(\cdot\mid\latent,\action),\quad \action'\sim\fallbackRobust(\cdot\mid\latent'),\quad(\latent,\action)\in\mathcal{M}$
            \STATE $\theta_Q\leftarrow\theta_Q-\eta_Q\nabla_{\theta_Q}\mathcal{L}_{\QfunctionRobust}$ \hfill\eqref{eq:discounted-safety-value}
            \IF{$j\bmod K_\pi=0$}
                \STATE $\hat\action\sim\fallbackRobust(\cdot\mid\latent)$
                \STATE $\theta_\pi\leftarrow\theta_\pi-\eta_\pi\nabla_{\theta_\pi}\mathcal{L}_{\fallbackRobust}$ \hfill\eqref{eq:safety-policy-loss}
            \ENDIF
            \IF{$j\bmod K_d=0$}
                \STATE $\latent'\sim\latentDisturbanceDynamicsParameterized(\cdot\mid\latent,\action)$ \hfill (reparameterized)
                \STATE $\action'\leftarrow\operatorname{stopgrad}_{\theta_\pi}\!\big(\fallbackRobust(\latent')\big)$
                \STATE $\latentDisturbanceParam\leftarrow\latentDisturbanceParam-\eta_d\nabla_{\latentDisturbanceParam}\mathcal{L}_{\latentDisturbance}$ \hfill\eqref{eq:disturbance-loss}
            \ENDIF
            \STATE $\bar\theta_Q\leftarrow(1-\rho)\bar\theta_Q+\rho\theta_Q$
        \ENDFOR
        \IF{$\densityEstimator(\latent_{t+1})>\thresholdOOD$}
            \STATE \textbf{break} \hfill $\triangleright$ truncate OOD imagination
        \ENDIF
    \ENDFOR
\ENDFOR
\end{algorithmic}
\end{algorithm}

\subsection{Latent Safety Filter Training}

Because the world model is trained offline, actions proposed during safety filter training may be poorly represented in the training data for the current latent state. Such out-of-distribution (OOD) state--action pairs can induce hallucinated transitions, allowing the safety policy to exploit model errors and overestimate safety. Following UNISafe~\cite{seo2025uncertainty}, we incorporate epistemic uncertainty into the safety margin to discourage these unreliable imaginations. Augmenting the task failure margin with this uncertainty criterion encourages the learned safety policy to select actions whose predicted outcomes are both safe and supported by the offline data.

\subsection{Disturbance Parameterization}

In Sec.~\ref{sec:latent-robust-optim}, we parameterize the latent disturbance $\latentDisturbance(\latent,\action)$ as an additive residual to the nominal distribution parameters, enabling optimization over plausible next-state distributions through \eqref{eq:disturbance-residual}--\eqref{eq:rssm-pessimistic-dynamics}. We implement this network as an MLP that takes features of the nominal next-state prediction: the deterministic features $h$ and prior parameters $(\mu,\sigma)$ for Gaussian dynamics, or $h$ alone for categorical dynamics. Thus, the current latent state and action condition the disturbance through the nominal prediction. The residual changes only the stochastic prior, preserving the deterministic features.

\begin{table}[!ht]
\caption{Residual parameterization of $\latentDisturbance$ in \eqref{eq:disturbance-residual}.}
\label{tab:impl-disturbance-architecture}
\centering\footnotesize
\setlength{\tabcolsep}{4pt}
\renewcommand{\arraystretch}{1.15}
\begin{tabular}{@{}lcc@{}}
\toprule
\textbf{Component} & \textbf{Gaussian} & \textbf{Categorical}\\
\midrule
Input & $[h,\mu,\sigma]$ & $h$\\
Input dimension & $512+32+32$ & $512$\\
Hidden layers & \multicolumn{2}{c}{$4$ layers, $512$ units each}\\
Normalization / activation & \multicolumn{2}{c}{LayerNorm / ReLU}\\
Output residual & $(\Delta\mu,\Delta\sigma)$ & $\Delta\parameter$\\
Output dimension & $32+32$ & $32\times32$\\
Output transform & $(\tanh,\,0.5\tanh)$ & $15\tanh$\\
Prior update & \shortstack{$\mu+\Delta\mu$\\$\sigma+\Delta\sigma$} & $\parameter+\Delta\parameter$\\
Output initialization & \multicolumn{2}{c}{Zero weights and bias}\\
\bottomrule
\end{tabular}
\vspace{-0.05in}
\end{table}

Table~\ref{tab:impl-disturbance-architecture} summarizes the parameterization. For Gaussian dynamics, we perturb the standard deviation $\sigma$, and the covariance in \eqref{eq:rssm-pessimistic-dynamics} is obtained by squaring these standard deviations. For categorical dynamics, we add the residual to the prior logits $\parameter$. Zero initialization of the output layer makes the initial residual zero, and the KL and OOD penalties constrain the learned perturbation during optimization.

\subsection{3D Dubins' Car with Disturbances}\label{sec:appendix-dubins}

\para{Optimal Solution of Naughty Dubins' Car} In this section, we derive the optimal robust control and worst-case disturbance for the \textit{naughty} Dubins' car in \eqref{eq:naughty-dynamics}:
\begin{equation*}
\dot p_x=v\cos\theta,\qquad
\dot p_y=v\sin\theta,\qquad
\dot\theta=\delta\action,
\end{equation*}
Let $\valfuncGt(\state)$ be the optimal robust safety value for these physical dynamics. Define its gradient components as
\begin{equation*}
g_x:=\frac{\partial\valfuncGt}{\partial p_x},\qquad
g_y:=\frac{\partial\valfuncGt}{\partial p_y},\qquad
g_\theta:=\frac{\partial\valfuncGt}{\partial\theta}.
\end{equation*}
By the chain rule, the Hamiltonian gives its instantaneous change along the dynamics:
\begin{equation*}
H(\state,\action,\delta):=\nabla_{\state}\valfuncGt(\state)^\top\dot\state=v(g_x\cos\theta+g_y\sin\theta)+g_\theta\delta\action.
\end{equation*}
The action maximizes this, while the disturbance minimizes it after observing the action. The optimization reduces to
\begin{equation*}
\max_{|\action|\leq1.25}\min_{\delta\in\{-1,1\}}
g_\theta\delta\action
=\max_{|\action|\leq1.25}-|g_\theta\action|=0.
\end{equation*}
Thus, a Hamiltonian-optimal robust action is $\mathbf{\action^\star=0}$, and a worst-case response to any proposed action is
\begin{equation*}
\delta^\star=
\begin{cases}
-1,&g_\theta\action>0,\\
+1,&g_\theta\action\leq0.
\end{cases}
\end{equation*}
When $g_\theta\neq0$, every nonzero turn lets the disturbance decrease safety through the heading contribution. \textit{Driving straight}, as in Fig.~\ref{fig:naughty-dubins}, eliminates this contribution.

\para{Latent Safety Filter Setup} World model hyperparameters are listed in Table~\ref{tab:impl-dubins-wm} and safety filter training settings are listed in Table~\ref{tab:impl-dubins}. We add the predicted residuals to the mean and standard deviation of the Gaussian prior, with the standard deviation bounded below by $0.1$. We train a flow-matching density model on frozen latents and use its logpZO score~\cite{xu2025can} to penalize implausible perturbations. For safety learning, we train a Gaussian safety actor and twin critics using the discounted safety update in \eqref{eq:discounted-safety-value}. The actor and critic networks each use four hidden layers of width $512$, and the critic outputs are bounded by $1.5\tanh(\cdot)$. We update the critics, safety policy, and disturbance every $1$, $4$, and $2$ replay updates, respectively.

\begin{table}[!ht]
\caption{3D Dubins' Car Latent World Model Settings.}\label{tab:impl-dubins-wm}
\centering\footnotesize
\renewcommand{\arraystretch}{1.12}
\begin{tabular}{@{}>{\raggedright\arraybackslash}p{.45\columnwidth}>{\centering\arraybackslash}p{\dimexpr\columnwidth-.45\columnwidth-2\tabcolsep\relax}@{}}
\toprule
\textbf{Hyperparameter} & \textbf{Value}\\
\midrule
Image Dimension & $128\times128\times3$\\
Action Dimension & $1$ (continuous)\\
Stochastic Latent & Gaussian\\
Deterministic Dimension & $512$\\
Stochastic Dimension & $32$\\
Activation Function & SiLU\\
Encoder CNN Depth & $32$\\
Failure Classifier Layers & $2$\\
Batch Size & $16$\\
Batch Length & $32$\\
Optimizer & Adam\\
Learning Rate & $10^{-4}$\\
Training Iterations & $50{,}000$\\
\bottomrule
\end{tabular}
\vspace{-0.2in}
\end{table}

\begin{table}[!ht]
\caption{3D Dubins' Car Latent Safety Filter Settings.}\label{tab:impl-dubins}
\centering\footnotesize
\renewcommand{\arraystretch}{1.12}
\begin{tabular}{@{}>{\raggedright\arraybackslash}p{.45\columnwidth}>{\centering\arraybackslash}p{\dimexpr\columnwidth-.45\columnwidth-2\tabcolsep\relax}@{}}
\toprule
\textbf{Parameter} & \textbf{Value}\\
\midrule
Actor Learning Rate & $10^{-3}$\\
Critic Learning Rate & $10^{-4}$\\
Disturbance Learning Rate & $10^{-4}$\\
Discount Factor & $0.85\rightarrow0.9999$\\
Training Iterations & $100{,}000$\\
Replay Buffer Size & $100{,}000$\\
Batch Size & $512$\\
Parallel Imagination States & $512$\\
Imagination Steps & $50$\\
\bottomrule
\end{tabular}
\vspace{-0.2in}
\end{table}

\subsection{Simulation: Block Pouring}\label{sec:impl-sim}

\para{System Disturbance}
We randomize the environment physics at reset, with all randomized parameters sampled uniformly from the ranges listed in Table~\ref{tab:impl-sim-randomization}. The poses of all three cubes are perturbed within the specified bounds, while the physical properties of the green cube are randomized independently. These parameters remain fixed throughout each episode. For paired evaluation, we pre-sample $20$ physics configurations with a fixed seed and reuse them across source trajectories and filtering conditions, ensuring identical physics for corresponding rollouts.

\begin{table}[!ht]
\caption{Physical Randomization of Block Pouring in IsaacLab.}
\label{tab:impl-sim-randomization}
\centering\footnotesize
\renewcommand{\arraystretch}{1.12}
\begin{tabular}{lc}
\toprule
\textbf{Parameter} & \textbf{Randomization Range} \\
\midrule
Cube Position $x$ & $[-0.02,\,0.02]$ m \\
Cube Position $y$ & $[-0.02,\,0.02]$ m \\
Mass & $[0.8,\,1.2]$ kg \\
Static Friction & $[0.2,\,0.8]$ \\
Dynamic Friction & $[0.1,\,0.5]$ \\
Restitution & $[0.05,\,0.3]$ \\
\bottomrule
\end{tabular}
\end{table}

\para{Latent Safety Filter Setup} World model hyperparameters are listed in Table~\ref{tab:impl-sim-wm} and safety filter training settings are listed in Table~\ref{tab:impl-sim}. Following UNISafe~\cite{seo2025uncertainty}, a five-member probabilistic ensemble provides the log-JRD uncertainty score used in the safety margin to prevent OOD action exploitation of WM imaginations~\cite{seo2025uncertainty}. For safety filter synthesis, the actor and critic networks each use four hidden layers of width $512$, and update the critics, safety policy, and disturbance every $1$, $8$, and $2$ replay updates, respectively.

\begin{table}[!ht]
\caption{Block Pouring World Model Training Settings.}\label{tab:impl-sim-wm}
\centering\small
\renewcommand{\arraystretch}{1.12}
\begin{tabular}{@{}>{\raggedright\arraybackslash}p{.45\columnwidth}>{\centering\arraybackslash}p{\dimexpr\columnwidth-.45\columnwidth-2\tabcolsep\relax}@{}}
\toprule
\textbf{Hyperparameter} & \textbf{Value}\\
\midrule
Image Dimension & $2\times128\times128\times3$\\
Proprioception Dimension & $7$ (eef position/quaternion)\\
Action Dimension & $7$ (continuous)\\
Stochastic Latent & Categorical\\
Deterministic Dimension & $512$\\
Stochastic Dimension & $32\times32$\\
Activation Function & SiLU\\
Encoder CNN Depth & $32$\\
Encoder MLP Layers & $5$\\
Failure Classifier Layers & $3$\\
Batch Size & $16$\\
Batch Length & $64$\\
Optimizer & Adam\\
Learning Rate & $10^{-4}$\\
Training Iterations & $1{,}000{,}000$\\
\bottomrule
\end{tabular}
\vspace{-0.2in}
\end{table}

\begin{table}[!ht]
\caption{Simulated manipulation safety-learning settings.}\label{tab:impl-sim}
\centering\small
\renewcommand{\arraystretch}{1.12}
\begin{tabular}{@{}>{\raggedright\arraybackslash}p{.45\columnwidth}>{\centering\arraybackslash}p{\dimexpr\columnwidth-.45\columnwidth-2\tabcolsep\relax}@{}}
\toprule
\textbf{Parameter} & \textbf{Value}\\
\midrule
Actor Learning Rate & $10^{-3}$\\
Critic Learning Rate & $10^{-3}$\\
Disturbance Learning Rate & $10^{-3}$\\
Discount Factor & $0.95$\\
Training Iterations & $400{,}000$\\
Replay Buffer Size & $1{,}000{,}000$\\
Batch Size & $512$\\
Parallel Imagination States & $1{,}024$\\
Imagination Steps & $50$\\
\bottomrule
\end{tabular}
\vspace{-0.1in}
\end{table}

\para{Task Policy} We train the DreamerV3 as a task policy~\cite{hafner2023dreamerv3} independently of the safety filter. The actor optimizes imagined $\lambda$-returns with $\lambda=0.95$ and continuation-weighted discount, backpropagating through the dynamics with an entropy coefficient of $3\times10^{-4}$. We normalize advantages using moving $5$th and $95$th return percentiles (update rate $10^{-2}$, minimum scale $1$), and the critic uses a twohot log-likelihood objective with a slow-target regularizer. We also use Diffusion Policy~\cite{chi2023diffusionpolicy} by behavior cloning as a task policy. A vision-conditioned MLP or conditional U-Net predicts $8$ future actions of $7$D using $100$ denoising steps. Both are trained for $1{,}000$ epochs at learning rate $10^{-4}$; the MLP uses ViT image encoders and batch size $256$, while the U-Net uses ResNet-18 encoders and batch size $64$. During deployment, the task policy executes four actions before replanning.

\subsection{Real-World: Sunny-Side-Up Egg Serving}\label{sec:impl-egg}

\begin{table}[!ht]
\caption{Transformer-Based Latent World Model}
\label{tab:impl-world-model-comparison}
\centering\small
\setlength{\tabcolsep}{4pt}
\renewcommand{\arraystretch}{1.15}
\begin{tabular}{@{}>{\raggedright\arraybackslash}p{.27\columnwidth}>{\raggedright\arraybackslash}p{\dimexpr.32\columnwidth-2\tabcolsep\relax}>{\raggedright\arraybackslash}p{\dimexpr.41\columnwidth-2\tabcolsep\relax}@{}}
\toprule
 & \textbf{RSSM} & \textbf{Transformer RSSM}\\
\midrule
Visual encoder & Trainable CNN & Frozen DINOv3\\
Temporal model & GRU & Causal transformer\\
History & Recurrent & $8$-step attention window\\
Recon. Target & Pixels & Pixels and DINO features\\
\bottomrule
\end{tabular}
\end{table}

\begin{table}[ht]
\caption{Real World: World Model Hyperparameters.}\label{tab:impl-egg-wm}
\centering\footnotesize
\renewcommand{\arraystretch}{1.12}
\begin{tabular}{@{}>{\raggedright\arraybackslash}p{.45\columnwidth}>{\raggedright\arraybackslash}p{\dimexpr\columnwidth-.45\columnwidth-2\tabcolsep\relax}@{}}
\toprule
\textbf{Hyperparameter} & \textbf{Value}\\
\midrule
Image Dimension & $2\times192\times256\times3$\\
Proprioception Dimension & $8$\\
Action Dimension & $7$ (continuous)\\
Stochastic Latent & Gaussian\\
Deterministic Dimension & $512$\\
Stochastic Dimension & $32$\\
Activation Function & SiLU\\
Encoder MLP Layers & $2$ (proprioception)\\
Failure Classifier Layers & $3$ per failure head\\
Batch Size & $20$\\
Batch Length & $64$\\
Optimizer & Adam\\
Learning Rate & $10^{-4}$\\
Training Iterations & $200{,}000$\\
\bottomrule
\end{tabular}
\end{table}

\begin{table}[ht]
\caption{Real World Latent Safety Filter Setups.}
\label{tab:impl-egg-rl}
\centering\small
\renewcommand{\arraystretch}{1.12}
\begin{tabular}{@{}>{\raggedright\arraybackslash}p{.45\columnwidth}>{\centering\arraybackslash}p{\dimexpr.55\columnwidth-2\tabcolsep\relax}@{}}
\toprule
\textbf{Parameter} & \textbf{Value}\\
\midrule
Actor Learning Rate & $5\times10^{-4}$\\
Critic Learning Rate & $10^{-3}$\\
Disturbance Learning Rate & $10^{-3}$\\
Target Smoothing & $0.005$\\
Discount Factor & $0.85\rightarrow0.9999$\\
Training Iterations & $400{,}000$\\
Replay Buffer Size & $500{,}000$\\
Batch Size & $512$\\
Parallel Imagination States & $256$\\
Imagination Steps & $30$\\
\bottomrule
\end{tabular}
\end{table}

\para{Latent World Model}
Each observation contains RGB images from two tabletop cameras, seven robot joint positions, and the gripper width, while actions specify an end-effector delta pose and gripper command. Failures are manually labeled for each frame based on falling and flipping conditions. For latent world mode, we adapt the RSSM~\cite{hafner2019learning} with pretrained visual features and transformer-based latent dynamics, with the main architectural differences summarized in Table~\ref{tab:impl-world-model-comparison}. A frozen DINOv3 ViT-S/16+~\cite{simeoni2025dinov3} extracts $192$ patch tokens of dimension $384$ from each image; learned view embeddings distinguish the two cameras, and a two-layer transformer fuses their tokens. Eight learned queries pool the fused tokens into $3{,}072$ visual features, while a separate MLP maps the eight proprioceptive inputs to $512$ features. Their concatenation forms the $3{,}584$-dimensional observation embedding. 

For latent dynamics, a causal transformer predicts $512$ deterministic features from an eight-step history of stochastic states and actions, providing the same temporal context during training and imagination. Actions are embedded and injected into each transformer block through AdaLN-zero conditioning. The prior predicts a $32$-dimensional diagonal Gaussian from the deterministic features, and concatenating a stochastic sample with the deterministic features yields the $544$-dimensional latent state $\latent$ used by the safety filter. We jointly train pixel and frozen-feature reconstruction together with continuation, falling, and flipping prediction. Table~\ref{tab:impl-egg-wm} lists the architecture and training hyperparameters.

\para{Latent Safety Filter Setup}
Table~\ref{tab:impl-egg-rl} summarizes the safety-filter training settings. Following Algorithm~\ref{alg:impl-safety-learning}, we update the critics, safety policy, and disturbance every $1$, $4$, and $2$ replay updates, respectively.



\end{document}